%% file: main.tex
\documentclass{article}

\usepackage{microtype}
\usepackage{graphicx}
\usepackage{subcaption}
\usepackage{booktabs} 

\usepackage{hyperref}

 \usepackage[accepted]{icml2026}

\usepackage{amsmath}
\usepackage{amssymb}
\usepackage{mathtools}
\usepackage{amsthm}

\usepackage[capitalize,noabbrev]{cleveref}

\theoremstyle{plain}
\newtheorem{theorem}{Theorem}[section]

\newtheorem{lemma}[theorem]{Lemma}

\theoremstyle{definition}
\newtheorem{definition}[theorem]{Definition}

\theoremstyle{remark}
\newtheorem{remark}[theorem]{Remark}

\usepackage{amsmath,amssymb}
\usepackage{bbold}
\usepackage{bm}
\usepackage{enumitem}
\usepackage{multirow}
\usepackage{multicol}

\usepackage[T1]{fontenc}

\usepackage{etoc}       

\makeatletter
\providecommand\@currentHref{}
\def\addcontentsline#1#2#3{%
  \addtocontents{#1}{\protect\contentsline{#2}{#3}{\thepage}{\@currentHref}}%
}
\makeatother

\usepackage{xcolor}
\usepackage{upquote}      
\usepackage{inconsolata}  
\usepackage{listings}
\usepackage{adjustbox}    

\lstdefinestyle{py}{
  language=Python,
  basicstyle=\ttfamily\footnotesize,      
  keywordstyle=\bfseries\color{blue!60!black},
  commentstyle=\itshape\color{teal!60!black},
  stringstyle=\color{orange!60!black},
  showstringspaces=false,
  numbers=left,
  numbersep=6pt,
  frame=single,
  rulecolor=\color{black!40},
  framesep=3pt,
  tabsize=2,
  breaklines=true,
  breakatwhitespace=false,
  columns=fullflexible,
  upquote=true
}

\lstnewenvironment{pycode}[1][]%
  {\lstset{style=py,#1}}%
  {}

\definecolor{purp}{RGB}{88,24,124}

\newcommand{\secref}[1]{Sec.~\ref{#1}}
\newcommand{\figref}[1]{Fig.~\ref{#1}}

\newcommand{\hpow}[2]{#1^{\circ #2}}                   

\newcommand{\by}{\mathbf{y}}
\newcommand{\bt}{\mathbf{t}}
\newcommand{\bff}{\mathbf{f}}
\newcommand{\bx}{\mathbf{x}}
\newcommand{\bxref}{\mathbf{x}_{\rm ref}}
\newcommand{\bz}{\mathbf{z}}
\newcommand{\bZ}{\mathbf{Z}}
\newcommand{\bX}{\mathbf{X}}

\newcommand{\DD}{\mathcal{D}}
\newcommand{\XX}{\mathcal{X}}
\newcommand{\OO}{\mathcal{O}}
\newcommand{\NN}{\mathcal{N}}
\newcommand{\RR}{\mathbb{R}}

\newcommand{\bbeta}{\bm{\beta}}
\newcommand{\bomega}{\bm{\omega}}
\newcommand{\bomegat}{\tilde{\bm{\omega}}}
\newcommand{\omegat}{\tilde{{\omega}}}
\newcommand{\bh}{\mathbf{h}}
\newcommand{\bht}{\tilde{\mathbf{h}}}
\newcommand{\bW}{\mathbf{W}}
\newcommand{\TT}{\mathcal{T}}
\newcommand{\bC}{\mathbf{C}}
\newcommand{\bL}{\mathbf{L}}
\newcommand{\bS}{\mathbf{S}}
\newcommand{\bU}{\mathbf{U}}
\newcommand{\bD}{\mathbf{D}}

\newcommand{\SSS}{\mathcal{S}}

\newcommand{\klin}{k_{\rm lin}}
\newcommand{\knl}{k_{\rm nl}}
\newcommand{\kLlin}{k_{\rm lin}^{\text{\tiny LOCK}}}
\newcommand{\kLlinp}{\tilde{k}_{\rm lin}^{\text{\tiny LOCK}}}
\newcommand{\kLnl}{k_{\rm nl}^{\text{\tiny LOCK}}}
\newcommand{\kL}{k^\text{\tiny LOCK}}

\newcommand{\sprod}{\mathop{\textstyle\prod}}

\newcommand{\eqnref}[1]{Eqn.~\ref{#1}}

\icmltitlerunning{Flexible Kernels for Protein Property Prediction}

\begin{document}

\twocolumn[
\icmltitle{Flexible Kernels for Protein Property Prediction}


\icmlsetsymbol{equal}{*}

\begin{icmlauthorlist}
\icmlauthor{Martin Jankowiak}{genb}  
\icmlauthor{Yerdos Ordabayev}{genb}
\icmlauthor{Rudraksh Tuwani}{genb}
\icmlauthor{Henry N.~Ward}{genb}
\icmlauthor{Hunter Nisonoff}{genb}
\icmlauthor{James M.~McFarland}{genb}
\icmlauthor{Gevorg Grigoryan}{genb}
\end{icmlauthorlist}

\icmlaffiliation{genb}{Generate Biomedicines, Somerville, MA, USA}

\icmlcorrespondingauthor{Martin Jankowiak}{jankowiak@gmail.com}
\icmlcorrespondingauthor{Gevorg Grigoryan}{gevorg@generatebiomedicines.com}

\icmlkeywords{Gaussian processes, sequence kernels, protein property prediction}

\vskip 0.3in
]



\printAffiliationsAndNotice{}  

\begin{abstract}
Despite its importance to applications in protein design, predicting protein properties like binding affinity and 
thermostability from sparse experimental data remains a significant challenge.
Accordingly, we introduce a class of sequence kernels that exploit evolutionary substitution matrices as well as local linearity and
demonstrate that the resulting Gaussian processes provide data-efficient models of protein property landscapes, frequently outperforming
alternatives that rely on foundation model embeddings.
Furthermore---by learning what are in effect structure-aware substitution matrices---we show that our kernels can readily incorporate
structural information from foundation models.
We demonstrate that these structure-conditioned kernels are well suited to multi-task learning across multiple protein property landscapes
and can decisively outperform local supervised learning methods. 
\end{abstract}

\section{Introduction}
\label{sec:intro}

\input{intro.tex}

\section{Background}
\label{sec:bg}

\input{bg.tex}

\section{Flexible Protein Sequence Kernels}
\label{sec:method}

\input{method.tex}

\section{Related Work}
\label{sec:related}

\input{related.tex}

\section{Experiments}
\label{sec:exp}

\input{exp.tex}

\section{Discussion}
\label{sec:disc}

\input{discussion.tex}

\section*{Acknowledgements}

We thank the anonymous reviewers for their constructive feedback and suggestions.

\section*{Impact Statement}

This work aims to improve data-efficient prediction of protein properties, 
with potential benefits for therapeutic discovery and industrial enzyme engineering. 
Like other methods that accelerate protein design, it could potentially be misused to facilitate harmful applications. 
Responsible use, appropriate oversight, and adherence to relevant regulations are therefore important.

\bibliography{ref}
\bibliographystyle{icml2026}


\newpage
\appendix
\onecolumn

\section{Appendix}
\label{sec:app}
\etocsettocstyle{\section*{Table of Contents}}{}
\etocsetnexttocdepth{subsection}
\localtableofcontents

\input{app.tex}

\end{document}

%% file: intro.tex
The ability to accurately predict protein properties in lieu of measuring them is an essential tool in protein design.
While much recent attention has focused on all-atom generative models, 
the utility of high-fidelity discriminators based on structure predictors like AlphaFold2 \cite{Jumper2021AlphaFold} is no less evident \cite{pacesa2025one},
and we contend that the entire generative-to-discriminative continuum is essential to the toolkit of the well-equipped protein designer.
With that motivation in mind, we set out to develop accurate models of protein property landscapes. 
We are particularly interested in methods that are: i) data-efficient; ii) endowed with high-quality uncertainty estimates; 
and iii) well-suited to multi-task learning across multiple landscapes. 
Below we argue theoretically and demonstrate empirically that Gaussian processes (GPs) with suitably constructed kernels are a good fit for all three criteria. 
Our contributions include:
\begin{enumerate}[itemsep=0pt, topsep=0pt]
    \item comprehensive benchmarking of $30+$ protein property predictors on $21$ datasets in three different regimes
    \item a sequence-only GP that frequently outperforms more complex (structure-conditioned) predictors that rely on 
        foundation models with millions of parameters
    \item a simple but powerful recipe for training multi-task GPs that leverage foundation model embeddings 
        to form zero-shot structure-conditioned kernels
\end{enumerate}

%% file: bg.tex
\subsection{Problem Description}
\label{sec:problem}

An \emph{empirical protein property landscape} consists of $N$ protein sequences $\{ \bx_n \}$, where each protein is associated with a scalar property $t_n$.
Each sequence is assumed to be of a fixed length\footnote{It is straightforward to handle variable length proteins, provided
that sequences are aligned and we introduce gap tokens.} $L$ over a fixed alphabet of size $A$.
Throughout we assume a one-hot sequence encoding in which $\bx$ is represented
as a binary matrix of shape $L \times A$ with entries $x_{\ell a}$.
Additionally, each landscape is optionally associated with
a reference sequence $\bx_{\rm ref}$ and 
a corresponding reference structure $\SSS$.
Our aim is to learn a predictive model that can predict $t$ for any 
sequence $\bx$.\footnote{In \secref{sec:multi} we consider multi-task models trained on multiple related protein property landscapes.}

\subsection{Gaussian Processes}
\label{sec:gpbg}

GPs offer powerful non-parametric function priors that often perform well in data scarce regimes.
A GP on the input space $\XX$ is specified
 by a covariance function or kernel $k:\XX \times \XX \to \mathbb{R}$ \cite{rasmussen2003gaussian}.
In order to be a valid kernel the $N\times N$ matrix $k(\bX, \bX)$ must be a 
valid covariance matrix---i.e.~positive semidefinite---for all $\bX = \{ \bx_n \}_{i=n}^N$ with $\bx_n \in \XX$.
For scalar regression\footnote{Apart from the experiments in \secref{sec:binary} we focus on regression.} $f: \XX \to \mathbb{R}$ the joint density of a GP with zero prior mean takes the form
\begin{equation}
\label{eqn:jointyf}
p(\bt, \bff | \bX) = \NN(\bt|\bff, \sigma_n^2 \mathbb{1}_N) \NN(\bff | \bm{0}, K_{\bX\bX})
\end{equation}
where $\bt$ are the real-valued targets, $\bff$ are the latent function values,
$\sigma_n^2$ is the variance of the Normal likelihood $\NN(\bt|\cdot)$,
and $K_{\bX\bX}$ is the $N \times N$ kernel matrix.
The marginal likelihood of the observed data can be computed in closed form
\begin{align}
\label{eqn:mll}
p(\bt|\bX) = \int \! d \bff \; p(\bt, \bff | \bX) = \NN(\bt | \bm{0}, K_{\bX\bX} + \sigma_n^2 \mathbb{1}_N)
\end{align}
and can be maximized with gradient methods to fit kernel hyperparameters.
The posterior predictive distribution of the GP at a query point $\bx^* \in \XX$ is the
Normal distribution $\NN(\mu_\bff(\bx^*), \sigma_\bff(\bx^*)^2)$
where $\mu_\bff(\cdot)$ and $\sigma_\bff(\cdot)^2$ are given by
\begin{align}
\label{eqn:gppredmean}
\mu_\bff(\bx^*) &= {k_{* \bX}}^{\rm T}  {(K_{\bX\bX} + \sigma_n^2 \mathbb{1}_N)}^{-1}\bt    \\
\label{eqn:gppredvar}
\sigma_\bff(\bx^*)^2 &=  k_{**} - {k_{* \bX}}^{\rm T}  {(K_{\bX\bX} + \sigma_n^2 \mathbb{1}_N)}^{-1}    {k_{* \bX}}
\end{align}
and where e.g.~$k_{**} \equiv k(\bx^*,\bx^*)$.
Computing \eqnref{eqn:mll}-\ref{eqn:gppredvar} is cubic in the number of data points $N$. 
Unless noted otherwise, we use the posterior mean \eqnref{eqn:gppredmean} for prediction.

\subsection{Elementary Sequence Kernels}
\label{sec:simpleseq}

Two of the simplest kernels that can be defined on one-hot sequence space are the linear and RBF kernel.
The (isotropic) linear kernel is defined as 
\begin{equation}
    \label{eqn:isolin}
    \klin(\bx, \by) = \bx^{\rm T} \by = \sum_{\ell=1}^L \bx_\ell^{\rm T} \by_\ell = 
    \sum_{\ell=1}^L \sum_{a=1}^A x_{\ell a} y_{\ell a}
\end{equation}
and can be written in terms of the Hamming distance $d_{\rm H}$ as  
\begin{equation}
    \label{eqn:isolin2}
    \klin(\bx, \by) = L - d_{\rm H}(\bx, \by) \in \{0, 1, ..., L\}
\end{equation}
The kernel in \eqnref{eqn:isolin} is called linear, since it corresponds to Bayesian linear 
regression with $f(\bx) = \bbeta \cdot \bx$, where an isotropic Gaussian prior has been placed on the coefficients
$\bbeta \in \RR^{L \times A}$. 
In other words it assumes no prior structure among the coefficients $\beta_{\ell a}$.
This means that, for example, observing a mutation $8A\to 8V$ does not inform the effects of other mutations like $8A\to 8I$.

In contrast to \eqnref{eqn:isolin}, which has an additive structure, the RBF kernel has a multiplicative structure:
\begin{equation}
    \label{eqn:rbfkernel}
    k_{\rm rbf}(\bx, \by) = \prod_{\ell=1}^L \exp \left(-\tfrac{1}{2} ||\bx_\ell - \by_\ell||^2 / \tau_\ell^2\right)
\end{equation}
where $\tau_\ell >0$ is a lengthscale. 
Unlike $\klin$, the RBF kernel decays exponentially as $d_{\rm H}(\bx, \by)$ increases.
Since $||\bx_\ell - \by_\ell||^2 \in \{0, 2\}$ for one-hot sequences, \eqnref{eqn:rbfkernel} is more naturally written as 
\begin{equation}
    \label{eqn:rbfkernel2}
    k_{\rm rbf}(\bx, \by) = \prod_{\ell=1}^L \exp (\tfrac{-1}{\tau_\ell^2})^{\mathbb{1}(\bx_\ell \ne \by_\ell)} 
    = \prod_{\ell=1}^L \rho_\ell^{\mathbb{1}(\bx_\ell \ne \by_\ell)}
\end{equation}
where $\rho_\ell\equiv \exp (-1 / \tau_\ell^2) < 1$ and $\mathbb{1}(\bx_\ell \ne \by_\ell)\in \{0,1\}$ encodes whether 
sequences $\bx$ and $\by$ have the same amino acid at position $\ell$. That is to compute $k_{\rm rbf}(\bx, \by)$ you
take all the positions where $\bx$ and $\by$ differ and multiply together the corresponding $\rho_\ell$, with small $\rho_\ell$
corresponding to low similarity between $\bx_\ell$ and $\by_\ell$. 
Since the RBF kernel does not distinguish between different amino acids at a given position, 
observing a mutation $8A\to 8V$ does not inform the effects of other mutations like $8A\to 8I$.

To connect with familiar concepts, we have introduced the RBF kernel via its canonical form in \eqnref{eqn:rbfkernel}, 
which is equally applicable to euclidean inputs $\bx$. 
In the context of one-hot encoded sequences $\bx$, however, this framing is a bit misleading, since it obscures that for such 
discrete $\bx$ 
the RBF kernel is best viewed in terms of a product over covariance matrices:
\begin{align}
    \label{eqn:rbfkernel3}
    k_{\rm rbf}(\bx, \by) = \prod_{\ell=1}^L S_{\ell a_\ell(\bx) a_\ell(\by)} 
     = \prod_{\ell=1}^L \bx_\ell^{\rm T} \bS_{\ell} \by_\ell
\end{align}
Here each $\bS_\ell$ is a covariance matrix of size $A \times A$, and
$a_\ell(\bx)$ is the amino acid index of sequence $\bx$ at position $\ell$.
Viewed through this lens \eqnref{eqn:rbfkernel2} is a special case of \eqnref{eqn:rbfkernel3} where $\bS_\ell$ is a correlation matrix
with ones along the diagonal and all off-diagonal entries given by $\rho_\ell$.
For generic correlation matrices $\bS_{\ell}$, \eqnref{eqn:rbfkernel3} can be understood as an \emph{anisotropic} RBF kernel;
we will make use of this straightforward generalization in \secref{sec:kernel}.

%% file: method.tex
Before we introduce our kernel in \secref{sec:kernel}-\ref{sec:kernelprop}, we lay important foundations in \secref{sec:sub}-\ref{sec:loclin}. 
We describe the priors we place on kernel hyperparameters in \secref{sec:hyper}. 
Then in \secref{sec:clock} we show how to generalize our construction to \emph{structure-conditioned} kernels.

\subsection{Substitution Matrices}
\label{sec:sub}

\begin{figure}[t]
  \centering
    \includegraphics[width=0.475 \columnwidth]{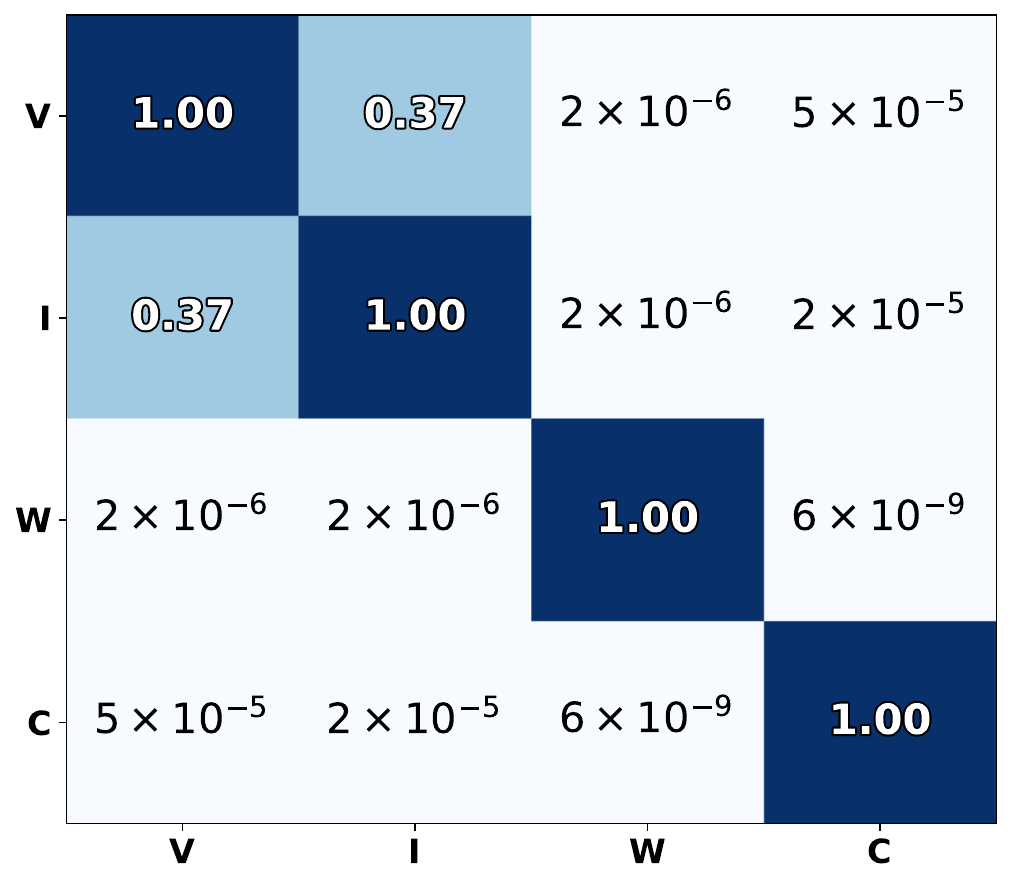}  
    \includegraphics[width=0.503 \columnwidth]{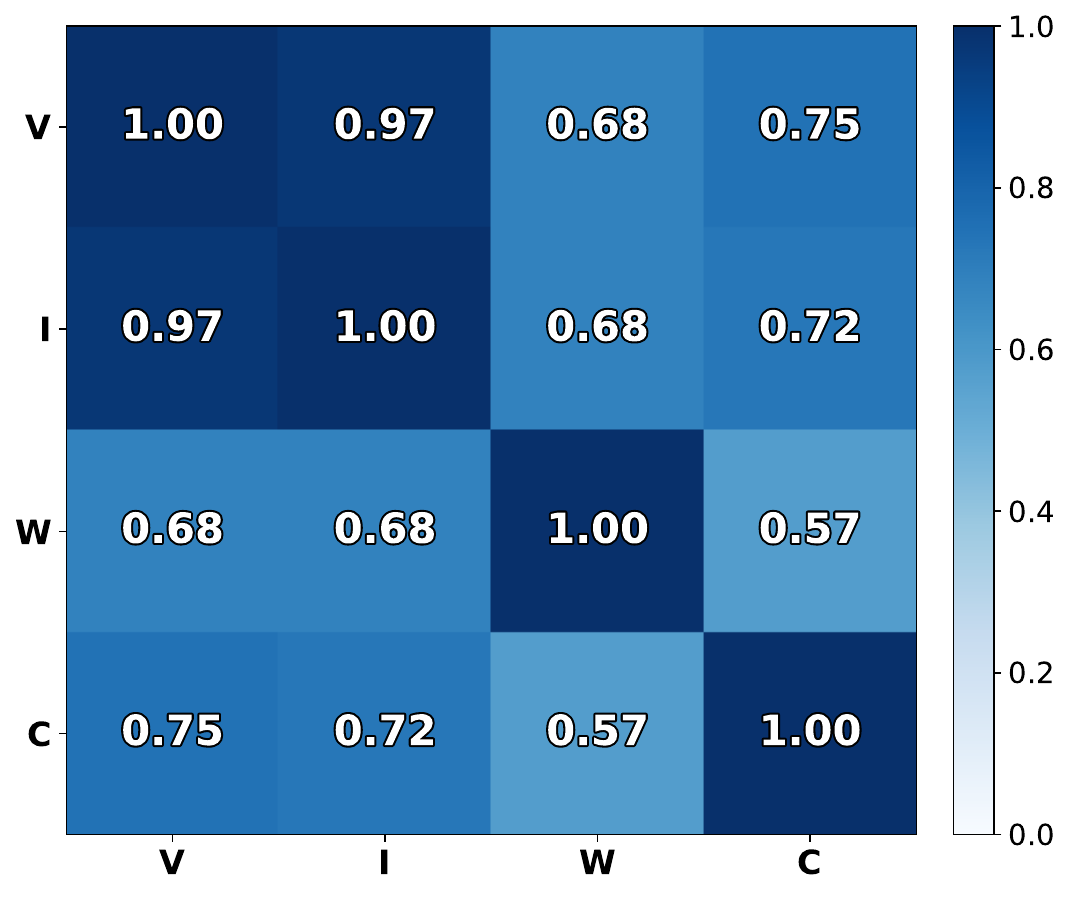}    
  \caption{The BLOSUM50 substitution matrix as a correlation matrix. 
    We zoom-in on four amino acids, two of which are biophysically similar (valine and isoleucine) and two
    of which are largely dissimilar to all other amino acids (tryptophan and cysteine). 
    The matrix on the right has been raised to the power $0.03$.
    For details on the normalization scheme we use refer to \secref{app:lock}. 
    For the full matrix see \figref{fig:app:blosum-corr-full}.
    }
  \label{fig:blosum-corr}
\end{figure}

We begin with a short discussion of substitution matrices,\footnote{In the bioinformatics literature substitution matrices are usually
reported as log-odds scores $M_{ij}=\log\!\big(\tfrac{q_{ij}}{p_i p_j}\big)$
\citep{altschul1991amino}. We work with their elementwise exponentiated form
$S_{ij}\equiv e^{M_{ij}}$ (sometimes called an odds-ratio matrix), and throughout we refer to $S$ as the substitution matrix.}
a classical bioinformatics 
tool that quantifies how likely one amino acid is to replace another during evolution.
Substitution matrices are derived empirically by counting substitutions in aligned homologous sequences and converting those frequencies into log-odds scores.
In practice we find it convenient to work with the correlation matrix form of substitution matrices, i.e.~for substitution
matrix $\bS$ we define the corresponding correlation matrix as
\begin{equation}
\label{eqn:corrdefn}
C_{aa^\prime} \equiv S_{aa^\prime} / \sqrt{S_{aa} S_{a^\prime a^\prime}} 
\end{equation}
In this way, an RBF-like kernel as in \eqnref{eqn:rbfkernel3} is equipped with a natural normalization in which
$k(\bx, \bx) =1$ for all $\bx$. 
Moreover, the diagonal entries of BLOSUM substitution matrices---which we use throughout---encode the \emph{absolute} level of conservation of each 
amino acid in evolutionary contexts, whereas we are interested in encoding \emph{relative} similarity \cite{henikoff1992amino}.
See \figref{fig:blosum-corr} for an illustration of a BLOSUM50-derived correlation matrix.

We make the interesting observation that while many substitution matrices (e.g.~most PAM matrices \citep{dayhoff1972model})
are positive semi-definite (PSD), some substitution matrices---in particular most of the BLOSUM family---have the
further property that they are \emph{infinitely divisible}.
That is for these matrices elementwise exponentiation by a positive power (i.e.~the Hadamard power $\hpow{S}{\alpha}$ 
given by $(\hpow{S}{\alpha})_{ij}=(S_{ij})^{\alpha}$) preserves positive semi-definiteness.\footnote{See 
Table~\ref{tab:substitution-matrices} for an inventory of substitution matrices and \secref{app:infdiv} 
for background on infinite divisibility.} 
The infinite divisibility of BLOSUM matrices makes them uniquely suited as ingredients in protein sequence kernels, since including
learnable exponents as part of the hyperparameters of a GP kernel offers a powerful and parsimonious mechanism
for modulating the covariance structure of the kernel. 
Powers greater than unity attenuate off-diagonal similarites $C_{a a^\prime}$, while powers less than unity amplify them.
Thus by learning landscape-specific exponents the biophysical similarity encoded by these substitution matrices can be adapted to the landscape at hand.
We return to these exponents in \secref{sec:hyper} when we discuss hyperparameter priors.

\subsection{Local Linearity}
\label{sec:loclin}

It has long been recognized that in many cases protein properties are approximately
additive w.r.t.~individual mutations \cite{wells1990additivity} and that pairwise epistasis effects
are often modest, at least once global epistasis is taken into account \cite{otwinowski2018inferring}.
We propose to incorporate this observation in our modeling approach by explicitly building local linearity into our kernel.
This is straightforward to do, since any kernel of the form $k^\prime \klin$---where $k^\prime$ is any kernel
and $\klin$ is a linear kernel---is naturally viewed as describing a \emph{locally linear} class of functions. 

In more detail, let $f(\bx) = \bbeta \cdot \bx$ be a linear model defined on one-hot sequences $\bx$ with coefficients $\bbeta$. 
If we promote $\bbeta$ to $\bx$-dependent coefficients $\bbeta(\bx)$ and place a 
GP prior on $\bbeta(\bx)$ controlled by $k^\prime$, 
then the resulting model is equivalent to a GP with product kernel $k^\prime \klin$.\footnote{This general observation
has been leveraged by \citet{yoshikawa2023gpx} in the context of explainable AI.} 
Indeed if $k^\prime$ is chosen so that $\bbeta(\bx)$ varies slowly across sequence space, $f(\bx) \approx \bbeta(\bx_0) \cdot \bx$ in
the vicinity of $\bx_0$, i.e.~$f(\bx)$ is locally linear.
For additional details and the elementary proof see \secref{app:loclin}.

\subsection{LOCK: Locally Linear Correlation Kernels}
\label{sec:kernel}

We combine the ingredients from \secref{sec:sub}-\ref{sec:loclin} to define our kernel:
the Locally Linear Correlation Kernel ({\bf LOCK}).
We first define the two base kernels from which the LOCK kernel is constructed. 
The first is the linear kernel,
\begin{equation}
    \label{eqn:linkernel}
    \kLlin(\bx, \by) = \Sigma_{\ell=1}^L \bx_\ell^{\rm T} \bC_\ell^{\alpha_\ell} \by_\ell
\end{equation}
where $\bC_\ell$ are correlation matrices as in \eqnref{eqn:corrdefn} and $\alpha_\ell>0$ are 
(optional) learnable exponents.\footnote{Note that in \eqnref{eqn:linkernel} we allow for the possibility that the 
correlation matrices vary from position to position.}
The kernel in \eqnref{eqn:linkernel} corresponds to Bayesian linear regression with 
coefficients $\bbeta  \in \RR^{L \times A}$ governed by a block diagonal covariance matrix with blocks $\bC_\ell^{\alpha_\ell}$.
The second can be seen as a generalization of the RBF kernel in \eqnref{eqn:rbfkernel}:
\begin{equation}
    \label{eqn:nlkernel}
    \kLnl(\bx, \by) = \sprod_{\ell=1}^L \bx_\ell^{\rm T} \bC_\ell^{\alpha_\ell} \by_\ell
\end{equation}
Evidently \eqnref{eqn:nlkernel} is the multiplicative analog of \eqnref{eqn:linkernel}.\footnote{Importantly, since the entries of $\bC_\ell$ are positive for infinitely divisible matrices, \eqnref{eqn:nlkernel}
can be computed in log space for improved numerical stability: 
    $\exp\left( \Sigma_{\ell=1}^L \alpha_\ell \bx_\ell^{\rm T} \log^{\circ}(\bC_\ell) \by_\ell \right)$
}

With \eqnref{eqn:linkernel}-\ref{eqn:nlkernel} in hand, we can define the \emph{locally linear correlation kernel} 
$\kL(\bx, \by)$ as
\begin{align}
    \label{eqn:lockdefn}
 \sigma_1^2 \kLnl(\bx, \by) \kLlin(\bx, \by)  + \sigma_2^2 \kLlinp(\bx, \by)
\end{align}
where we have introduced learnable kernel scales $\sigma_1, \sigma_2>0$. 
Due to the inclusion of the product kernel $\kLnl \kLlin$, the kernel $\kL$ is \emph{locally linear} 
as described in \secref{sec:loclin}. 
For more flexibility we also include the second linear kernel $\kLlinp$ (see \secref{sec:kernelprop} \& \secref{app:loclin} for additional motivation).
Besides depending on correlation matrices $\bC_\ell$, we emphasize that each of the three base kernels 
in \eqnref{eqn:lockdefn} optionally depend on additional learnable exponents, a modeling
choice that we discuss in more detail in \secref{sec:hyper}.
 \fbox{\parbox{0.97\linewidth}{Unless noted otherwise, we use the same {\bf BLOSUM50}-derived correlation matrix at each position.}}

\subsection{Kernel Properties}
\label{sec:kernelprop}

Above we have emphasized two properties of the LOCK kernel $\kL$: i) local linearity; 
and ii) how it incorporates prior information from substitution matrices.
Here we touch on a few additional properties of the kernel.\footnote{For a discussion of epistasis see \secref{app:epistasis}.}

\paragraph{Basis Functions} \!\!\!\! It is helpful to view $\kL$ as defining a function space built 
from $\mkern-2mu$`rotated'$\mkern-2mu$ amino acid basis functions
\begin{equation}
    f_{\ell a}(\bx) = (\bL_\ell^{\rm T} \bx_\ell)_a \qquad {\rm with} \qquad \bL_\ell \equiv {\rm Cholesky}(\bC_\ell)
    \nonumber
\end{equation}
so that $\bL_\ell \bL_\ell^{\rm T} = \bC_\ell$.
For example if we refer to the BLOSUM50 correlation matrix in \figref{fig:blosum-corr} we can infer that
one basis function puts significant weight on both valine and isoleucine, since these are highly correlated.
Thus the linear kernel $\kLlin$ can be seen as Bayesian linear regression w.r.t.~the
$LA$ basis functions $f_{\ell a}$ with an \emph{isotropic} prior covariance in the rotated basis. 
By contrast the non-linear kernel $\kLnl$ utilizes a \emph{tensor-product basis} made up of $A^L$ composite basis functions,
each of which is a $L$-wise product of $f_{\ell a}$.
This basis function perspective makes it manifest how 
observing a mutation like $8A\to 8V$ informs the effects of other mutations like $8A\to 8I$.

\paragraph{Kernel Decay} When using a model trained on fitness data in model-based protein design, 
an important consideration is how that model performs in interpolatory (close to the training data) and 
extrapolatory (away from the training data) regimes.
Consider a GP with a linear kernel like in \eqnref{eqn:linkernel}. 
It assumes perfect additivity, even for sequences that differ from the training data by tens of mutations. 
Put differently, since the kernel decays slowly (in particular linearly) as the Hamming distance to the training data increases, 
the linear kernel extrapolates aggressively. 
By contrast a multiplicative kernel like in \eqnref{eqn:nlkernel} decays exponentially as the Hamming distance to the 
training data increases. 
Such a kernel extrapolates more conservatively; in particular it is \emph{mean-reverting}, i.e.~the GP predictor
asymptotes towards the prior mean far away from the training data. 

Neither of these behaviors is particularly attractive in the context of model-based design,
and by construction the LOCK kernel in \eqnref{eqn:lockdefn} avoids both these extremes. Near the training
data ($\kLnl \gg 0$) it provides locally linear predictions that respond to the nuances of the local sequence data. 
Instead of reverting to the prior mean far away from the training data ($\kLnl \approx 0$), it reverts 
to a linear predictor controlled by $\kLlinp$.
This is arguably the best of both worlds: nuanced non-linear predictions in the vicinity of the training data 
that smoothly yield to a simple and robust linear predictor further away---all informed by the biophysical knowledge
encoded in substitution matrices.

\subsection{Hyperparameter Priors}
\label{sec:hyper}

With two kernel scales $\sigma_1$ and $\sigma_2$ and the noise scale $\sigma_n$ (see \eqnref{eqn:mll}) 
the LOCK-GP defined by \eqnref{eqn:lockdefn} has at least three hyperparameters. 
As for the exponents for each of the three base kernels, we can consider various levels of flexibility:
\begin{enumerate}[itemsep=3pt, topsep=0pt, parsep=0pt, partopsep=0pt]
\item no exponents (i.e.~all exponents are unity)
\item a global exponent $\alpha$ that is the same for all positions $\ell$
\item local exponents $\alpha_\ell$ that vary with $\ell$
\end{enumerate}
Crucially, these hyperparameters allow the LOCK kernel to adapt to the landscape at hand and are readily
fit by maximizing the GP marginal likelihood in \eqnref{eqn:mll} using gradient methods. 
Particularly in the case where the exponents $\alpha_\ell$ are allowed to vary with $\ell$,
it is important to regularize the hyperparameters, both for numerical stability as well as to mitigate
against possible overfitting. 
Unless noted otherwise, we utilize the following regularized hyperparameters: 
\begin{enumerate}[itemsep=3pt, topsep=0pt, parsep=0pt, partopsep=0pt]
\item $\kLnl$ is equipped with local exponents $\alpha_\ell$
\item $\kLlin$ and $\kLlinp$ each receive a global exponent 
\item $\sigma_1^2$, $\sigma_2^2$, and $\sigma_n^2$ are governed by weak Gamma priors 
\item global exponents are governed by weak priors
\item local exponents are governed by relatively strong priors 
\end{enumerate}
We find that these choices work well across a variety of scenarios; see \secref{app:lock} for the precise hyperpriors used. 
In experiments in \secref{sec:ablation} we ablate some of these choices.

\subsection{CLOCK: Structure-conditioned Kernels}
\label{sec:clock}

The LOCK kernel is a sequence kernel that depends on global protein data solely through pre-canned substitution matrices. 
In this section we show that $\kL$ yields a natural generalization to \emph{conditional} kernels that we dub {\bf CLOCK}.
The result is a powerful and elegant mechanism for incorporating protein foundation models into LOCK-GPs.

The core idea is that we map positional structure embeddings $\bh_{1:L}(\SSS)$ obtained from a pre-trained foundation model  
to positional correlation matrices $\bC_{1:L}$, each of size $A\times A$.\footnote{We could instead use protein language model embeddings that
do not rely on structure, but we do not explore that possibility here.}
These correlation matrices, which are specialized to the local structural context as encoded by each $\bh_\ell$, 
are then used to construct (zero-shot) LOCK kernels as in \secref{sec:kernel}.\footnote{The user can choose how many kernel hyperparameters to
adapt to the local landscape. In most cases learning the kernel and noise scales (at a minimum) is likely prudent.} 
Each correlation matrix is parameterized as 
\begin{equation}
    \label{eqn:corrrep}
    C_{\ell a a^\prime} = \exp\left(- || \bz_{\ell a} - \bz_{\ell a^\prime}||^2 \right)
\end{equation}
where $\bz_{\ell a}\in\RR^{A-1}$ is obtained from $\bh_\ell$ via a linear map:
\begin{equation}
\label{eqn:Wdefn}
    \bz_{\ell} = \bW \bh_\ell \qquad {\rm where} \qquad \bz_\ell \in \RR^{A \times (A-1)}
\end{equation}
Thus $\bC_\ell$ can be thought of as an RBF kernel defined on an abstract amino acid embedding space with embeddings $\bz_{\ell a}$.
Importantly, this parameterization is \emph{generic} in that each $\bC_\ell$ is infinitely divisible and all infinitely divisible correlation matrices are of the form in \eqnref{eqn:corrrep}, see \secref{app:infdiv}.

To learn $\bW$ we use a simple and robust training procedure based on the so-called concentrated form of the standard
GP training objective in which the overall kernel scale is `profiled' out (see \secref{app:clock} for details). 
Since $\bW$ has $\sim 49$k parameters for $A=20$ and 128-dimensional embeddings $\bh_\ell$,
this setup is most appropriate for settings with sufficient data. 
In particular it is an attractive approach for \emph{multi-task learning} across multiple related protein property landscapes, 
as we demonstrate empirically in \secref{sec:multi}.

%% file: related.tex
Quantitative modeling of protein sequence–function relationships has a long history, with early studies employing partial least squares 
regression (PLS) and Gaussian processes to model properties such as enzyme activity and thermostability \cite{Fox2003ProSAR,romero2013navigating}.
Much recent work has focused on zero-shot methods that utilize protein language models (PLMs), 
typically via masked-token pseudo-likelihoods or autoregressive log-likelihood ratios aggregated over mutated sites \cite{meier2021esm1v,rives2021pnas}. 
Alongside MSA-based generative models such as DeepSequence and EVE \citep{riesselman2018deepsequence,frazer2021eve}, 
these likelihood-based scores achieve strong zero-shot performance on ProteinGym \cite{notin2023proteingym}.

A popular approach to supervised learning on protein property landscapes uses PLMs like ESM-2 \cite{lin2023evolutionary} as feature extractors and fits a lightweight
supervised head \cite{Alley2019UniRep, Elnaggar2022ProtTransTPAMI}.
This approach can be extended to supervised fine-tuning, where the weights of the PLM are adapted to the supervised task \cite{rao2019evaluating, Brandes2022ProteinBERT},
and includes methods that fine-tune likelihood-based scores with ranking- or regression-based losses \cite{zhao2024contrastive,hawkins2024likelihood}.
Indeed several works in the literature benchmark a variety of such recipes, 
including for example those that employ parameter-efficient adapters \cite{schmirler2024fine,sledzieski2024democratizing,zhou2024enhancing,bikias2025plmfit}.
While recognizing the broad utility of this line of work, some authors have observed that improved performance on the PLM pretraining task
can fail to translate to the supervised task \cite{pmlr-v235-li24a,vieira2025medium}.
Likewise, despite the demonstrated value of this family of methods, a number of disadvantages must also be acknowledged, including
the large number of training hyperparameters, the high computational cost, and the risk of overfitting.

As of Jan.~2026, the top performer on ProteinGym's supervised DMS substitutions benchmark is the Kermut GP \cite{groth2024kermut},
which utilizes ESM-2 embeddings, an ESM-2-based zero-shot prior mean, 3d structure coordinates, 
and ProteinMPNN inverse folding logits to construct a kernel that sums over pairs of 
mutations.\footnote{See \secref{app:kermut} for a more detailed discussion of Kermut.}
As shown in \citet{groth2024kermut}, Kermut convincingly outperforms strong baselines like ProteinNPT \cite{notin2023proteinnpt}.
Researchers have also explored other kernels, including e.g.~fingerprint kernels on BLOSUM-derived embeddings \cite{gessner2024active-mlsb} 
and string kernels on mutation codes \cite{benjamins2024bayesian}. 
For a theoretical discussion of biological sequence kernels please refer to \citet{amin2025biological}.
For a recent analysis of BLOSUM similarity scores in the context of antibody binding see \citet{ucar2025blosum}.

%% file: exp.tex
We provide an open-source implementation of \texttt{LOCK-GP} at \texttt{https://github.com/generatebio/lock\_gp}.

\subsection{Local Landscape Data}
\label{sec:data}

We curate $21$ datasets with reference structures for benchmarking that span a range of property types (thermostability, 
binding affinity, fluorescence, capsid viability, etc.), including $9$ from ProteinGym \cite{notin2023proteingym}.
To enable comprehensive benchmarking we choose datasets with $\ge\! 1800$ data points, 
$\ge\! 10$ variable positions, and an abundance of higher-order mutants. See Tables~\ref{tab:datasets}-\ref{tab:datasets_numeric} for additional information on each dataset.\footnote{For all datasets the maximum Hamming distance to the reference
sequence is $\ge\! 6$; for $19/21$ datasets it is $\ge\! 10$. Structures for $10/21$ datasets are experimental; the rest are predicted.} 
We evaluate model performance in three different regimes: 
i) i.i.d.~cross-validation; 
ii) Hamming-distance-based extrapolation; 
and iii) an `unseen mutations' regime, in which all test sequences have at least one mutation not present in the corresponding training dataset. 
For details see \secref{app:data}.

\begin{figure*}[t]
\centering
    \includegraphics[width=\textwidth]{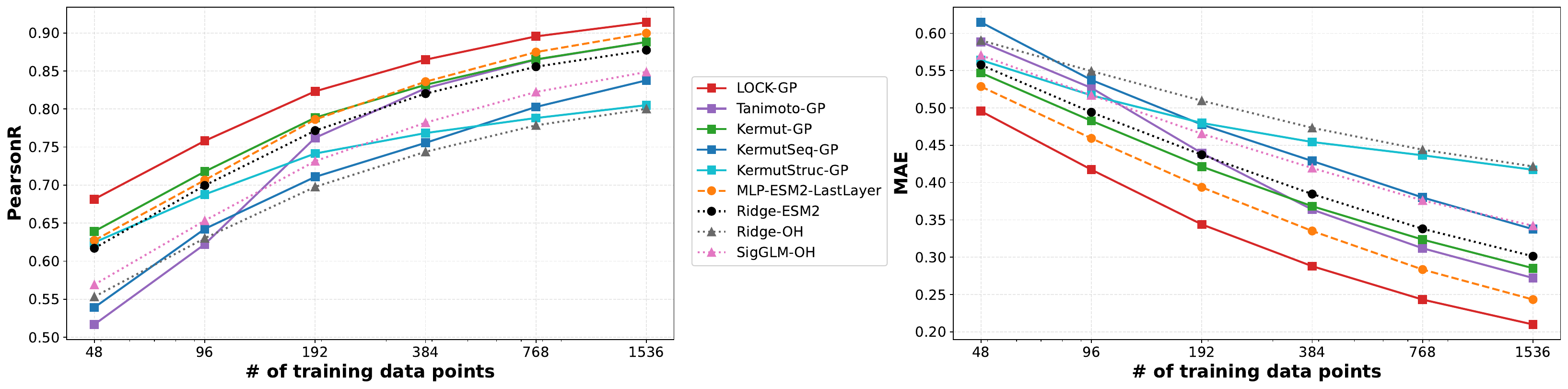}
    \caption{Predictive performance as a function of the number of training data points in the cross-validation setting.
    We depict Pearson R (left) and mean absolute error (right); metrics are averaged across $21$ datasets. 
    See \secref{sec:local} for discussion and \figref{fig:cvspearmanrmse} for Spearman R and RMSE.
    }
\label{fig:cvpearsonmae}
\end{figure*}

\begin{table*}[h]
    \caption{Model performance metrics for three different evaluation regimes, 
    with the number of training data points ranging from $48$ to $1536$.
    Metrics are averaged across $21$ datasets. For each column the best performing metric value is marked in bold. See also Table \ref{tab:ranks}.}
\label{tab:metric_summary}
\input{metric_summary_table.tex}
\end{table*}

\begin{figure}[h]
\centering
    \includegraphics[width=\columnwidth]{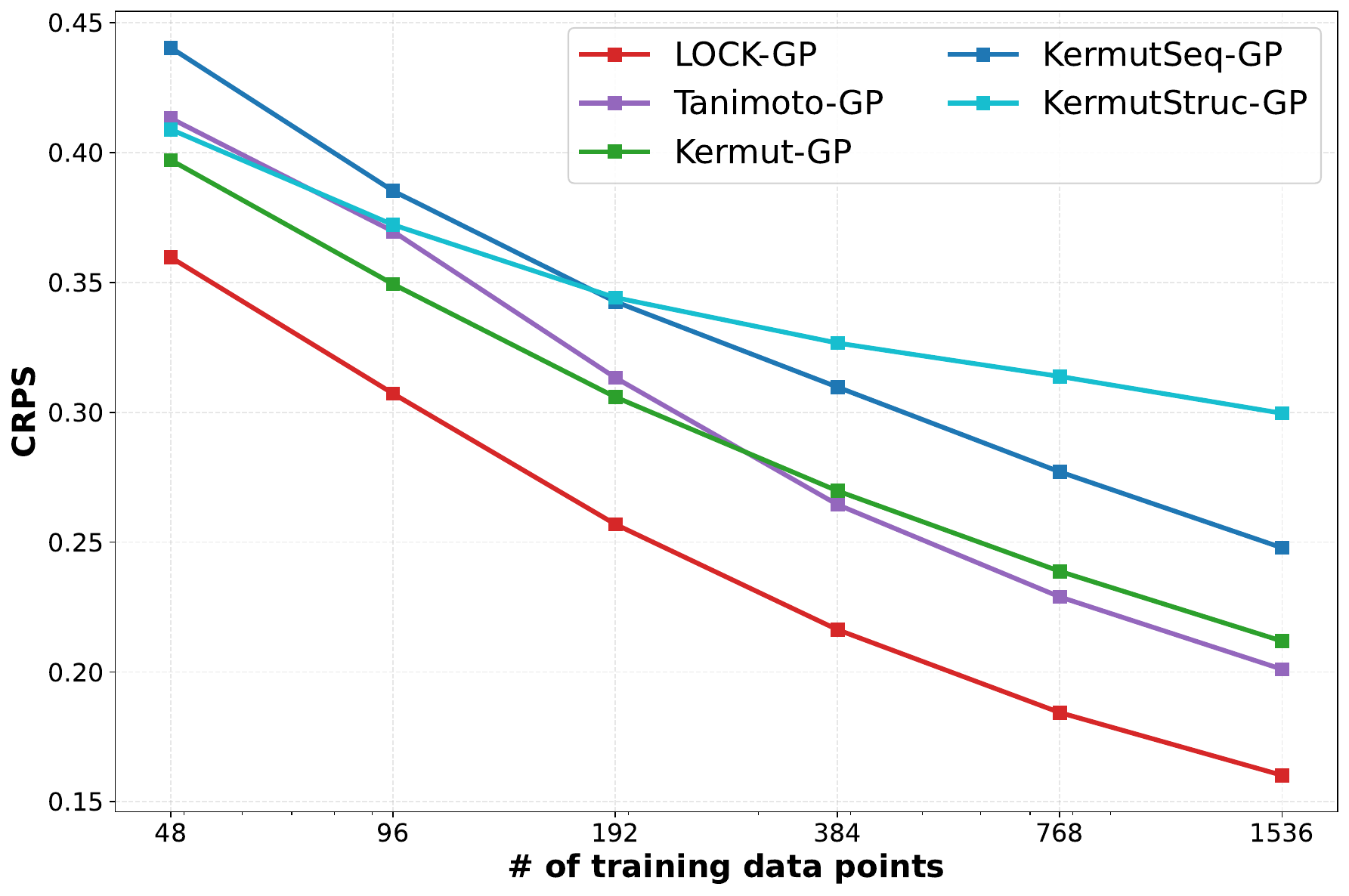}
    \caption{Predictive performance of GP models as a function of the number of training data points in the cross-validation setting.
    We depict the continuous ranked probability score (CRPS; \citet{gneiting2007strictly}),
    an MAE-like proper scoring rule that evaluates both accuracy and calibration of predictive distributions.
    Metrics are averaged across $21$ datasets.}
\label{fig:cvcrps}
\end{figure}

\input{model_table.tex}

\begin{figure}[t]
\centering
\includegraphics[width=0.942\columnwidth]{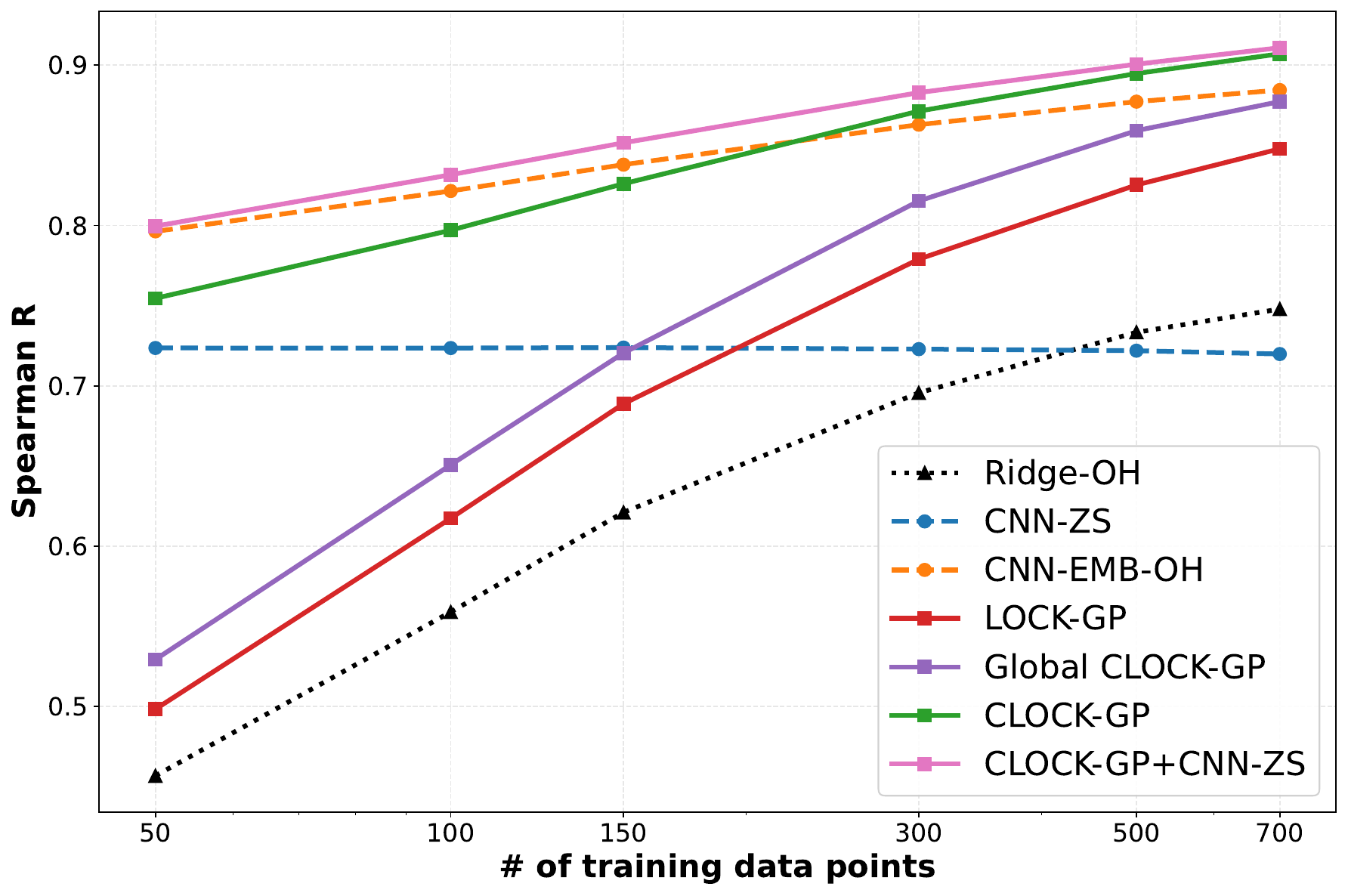}
\includegraphics[width=\columnwidth]{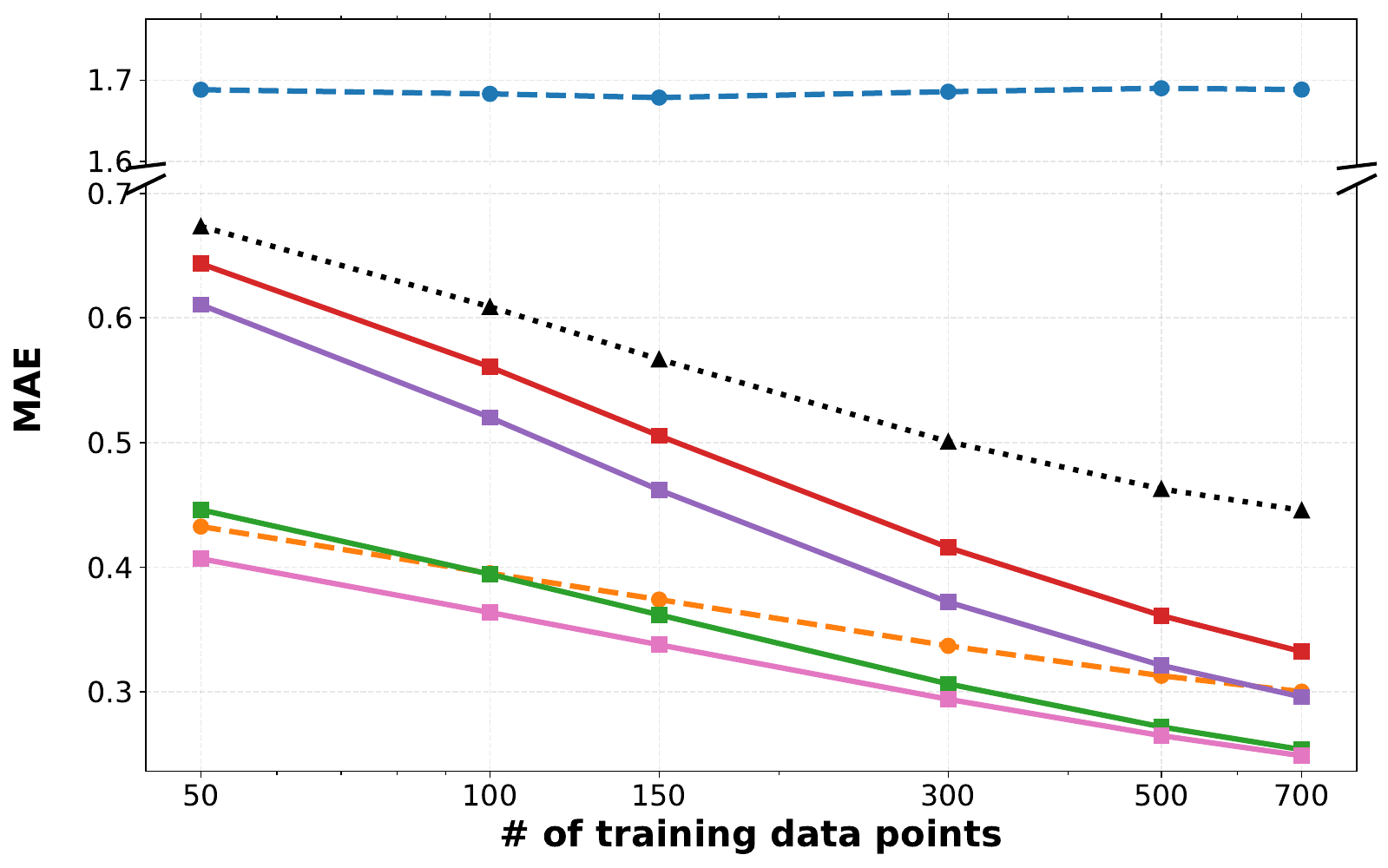}
    \caption{Predictive performance as a function of the number of (landscape-local) training data points on the multi-landscape thermostability 
    data described in \secref{sec:multi}. We depict both Spearman R (top) and mean absolute error (bottom); 
    metrics are averaged across 100 train/test splits in 50 held-out landscapes.
    The structure-conditioned CLOCK-GP augmented with an additional CNN-based kernel performs best across the board.
    Note the axis break for the MAE figure: the absolute scale of the zero-shot predictor is very poorly calibrated.
    See \figref{fig:app:tsuboperf} for additional metrics.}
\label{fig:tsuboperf}
\end{figure}

\subsection{Ablation Study}
\label{sec:ablation}

We conduct a systematic ablation study using the data from \secref{sec:data}.
We highlight a few conclusions: 
  i) dependence on the particular BLOSUM substitution matrix used is weak;
  ii) the non-linear kernel $\kLnl$ in \eqnref{eqn:nlkernel} consistently outperforms the RBF kernel in \eqnref{eqn:rbfkernel}, underscoring the value of incorporating substitution matrices;
  iii) the exponent priors in \secref{sec:hyper} are important for regularization;
  iv) using `local' residuewise exponents $\alpha_\ell$ improves fit noticeably;
  and v) assuming local linearity improves uncertainty quantification.
Taken together our analysis shows that the thoughtful inclusion of BLOSUM matrices is the single most 
critical ingredient in the good performance of LOCK.
For detailed results and additional analysis see \secref{app:ablation}.

\subsection{Local Learning}
\label{sec:local}

We use the data from \secref{sec:data} to benchmark \texttt{LOCK-GP} against $8$ different baselines; these include $4$ GPs with the following kernels: 
\texttt{Tanimoto-GP}) a Tanimoto kernel that uses BLOSUM62-derived embeddings \cite{gessner2024active-mlsb};
\texttt{KermutSeq-GP}) an isotropic RBF kernel that uses mean-pooled ESM-2 features;
\texttt{KermutStruc-GP}) a structure-conditioned kernel that exploits ProteinMPNN inverse folding logits \cite{dauparas2022robust};
and \texttt{Kermut-GP}) an additive combination of the previous two kernels \cite{groth2024kermut}.\footnote{In addition all three Kermut
variants use an ESM-2-derived zero-shot prior mean.}
We also include $4$ non-GP baselines:
\texttt{Ridge-OH}) ridge regression on one-hot features;
\texttt{Ridge-ESM2}) ridge regression on ESM-2 features;
\texttt{MLP-ESM2-LastLayer}) a single-layer neural network with ESM-2 features in which the last layer of the PLM is fine-tuned; and
\texttt{SigGLM-OH}) generalized linear regression on one-hot features with a sigmoidal link function \cite{Ding_2024}.
See Table~\ref{tab:models} for a high-level model taxonomy and \secref{app:sec:localdetails} for details on each baseline model.

We find that \texttt{LOCK-GP} exhibits the strongest cross-regime performance,
with \texttt{Kermut-GP} performing comparably on the unseen mutations benchmark.
See \figref{fig:cvpearsonmae} and Table \ref{tab:metric_summary} for the results 
and Table \ref{tab:metric_summary_long} for twelve additional baselines, including zero-shot methods and ConFit \cite{zhao2024contrastive}. 
In the cross-validation setting for $N=192$ training data points, for example, \texttt{LOCK-GP} has the best Pearson R 
(resp., MAE) for $16/21$ (resp., $17/21$) datasets; see Tables \ref{tab:landscape_pearsonr_cv}-\ref{tab:landscape_mae_extrapolation}
for these and other landscape-level results.
This strong performance is striking, since \texttt{LOCK-GP} makes do with the $210$ parameters of a 
BLOSUM50-derived correlation matrix,\footnote{Note that \texttt{Tanimoto-GP} also leverages BLOSUM matrices, but via an eigendecomposition in log-odds space, which is evidently less well-suited to typical protein landscapes.} while \texttt{Kermut-GP} relies on ESM-2 (650M parameters), ProteinMPNN (1.7M parameters) 
and the reference structure.
The low MAEs of \texttt{LOCK-GP} are matched by low CRPS---see \figref{fig:cvcrps}---reflecting
the high quality of \texttt{LOCK-GP} uncertainty estimates.
We find that \texttt{Ridge-ESM2} and \texttt{MLP-ESM2-LastLayer} perform rather poorly on the unseen mutations benchmark---indeed \texttt{Ridge-ESM2} 
is outperformed by the much simpler \texttt{Tanimoto-GP}---highlighting the potential fragility of high-dimensional embeddings in OOD settings.\footnote{We
find similar behavior for \texttt{ConFit}; see Table~\ref{tab:metric_summary_long}.}
Similarly in the extrapolation regime with $N=128$,
\texttt{LOCK-GP} (resp., \texttt{Tanimoto-GP}) outperforms \texttt{Kermut-GP} on $16/21$ (resp., $11/21$)
datasets w.r.t.~Pearson R, raising questions about the added value of using two foundation models.
Moreover, in the extrapolation regime with $N=512$ \texttt{LOCK-GP} exhibits a mean MAE that is $22\%$ lower than the next best model (\texttt{Tanimoto-GP}).
We also note that all three Kermut variants exhibit poor log likelihoods on some datasets, likely due to poorly regularized hyperparameters; 
see Table~\ref{tab:metric_summary_gp} and \secref{app:kermut} for additional discussion.
By contrast in the extrapolation regime with $N=128$ \texttt{LOCK-GP} exhibits the best negative log likelihood (NLL) on $12/21$ datasets,
whereas the corresponding numbers for \texttt{Tanimoto-GP} and \texttt{Kermut-GP} are $4/21$ and $3/21$, respectively; see Table \ref{tab:landscape_nll_combined}.
Finally we note that \texttt{LOCK-GP} can be much faster ($\sim \! 2-140$x) at training and inference time than methods like \texttt{MLP-ESM2-LastLayer} and \texttt{Kermut-GP} that rely on foundation models; see \secref{app:speed} for details.

\subsection{Multi-task Learning}
\label{sec:multi}

We demonstrate the utility of the CLOCK-GP introduced in \secref{sec:clock} by evaluating a variety of multi-task 
models on $371$ thermostability landscapes from \cite{Tsuboyama_2023}. 
Each landscape has $\sim \! 1000$ sequences, most of which are double mutants; the typical sequence length is $L \sim 50$.
Each landscape has an AlphaFold2 predicted structure \cite{Jumper2021AlphaFold}. 
We use 280 landscapes for training, while holding out $41+50$ landscapes for validation and testing.

As our main baseline we train a multi-scale CNN with $2.53$M parameters
that utilizes $128$-dimensional positional sequence-and-structure embeddings provided by Chroma \cite{ingraham2023illuminating}.
We use the CNN both as a zero-shot predictor (\texttt{CNN-ZS}) as well as a supervised learner (\texttt{CNN-EMB-OH}) in 
which we combine features from the CNN with one-hot sequences and feed the concatenated features into a ridge regression head. 
For \texttt{CLOCK-GP} we use $128$-dimensional positional structure embeddings from Chroma.
In addition to \texttt{CLOCK-GP} we consider two GP baselines: i) \texttt{LOCK-GP} with a BLOSUM50 substitution
matrix; and ii) \texttt{Global CLOCK-GP} in which we fit a global correlation matrix that is deployed 
at all positions.\footnote{That is we remove positional indices from \eqnref{eqn:corrrep}
and learn a single $\bz\in\RR^{A, A-1}$ 
that does \emph{not} depend on $\bh_{1:L}$.}
Finally we include ridge regression on one-hot sequences (\texttt{Ridge-OH}) as well as \texttt{CLOCK-GP} augmented
with an additional RBF kernel that leverages the CNN zero-shot predictor (\texttt{CLOCK-GP+CNN-ZS}).
For additional details see \secref{app:exp}.

The value of a structure-conditioned function prior is clearly demonstrated by comparing \texttt{CLOCK-GP} to \texttt{Ridge-OH}: the
former model trained on $N=50$ (landscape-local) data points outperforms the latter trained on $N=700$ data points; see \figref{fig:tsuboperf}.
Likewise the Spearman R of \texttt{CLOCK-GP} at $N=50$ is $0.755$, while it is only $0.498$ for \texttt{LOCK-GP}.
Setting aside the hybrid \texttt{CLOCK-GP+CNN-ZS},
the best performing model in the scarce data regime is \texttt{CNN-EMB-OH}, while \texttt{CLOCK-GP} is the strongest performer
once the number of training data points increases ($N \gtrsim 200$). 
Not surprisingly, we see the best performance from \texttt{CLOCK-GP+CNN-ZS}, which combines the flexibility of neural networks with the carefully constructed function prior of \texttt{CLOCK-GP}. 
Notably at $N=700$ \texttt{Global CLOCK-GP}, which does not use structure embeddings, matches the performance of \texttt{CNN-EMB-OH}, which does.
As expected \texttt{LOCK-GP}
is quite effective, although it is outperformed
by \texttt{Global CLOCK-GP} which learns a global correlation matrix adapted to thermostability.
See Table~\ref{tab:tsuboperf} for additional quantitative results.

\begin{figure}[t]
\centering
\includegraphics[width=\columnwidth]{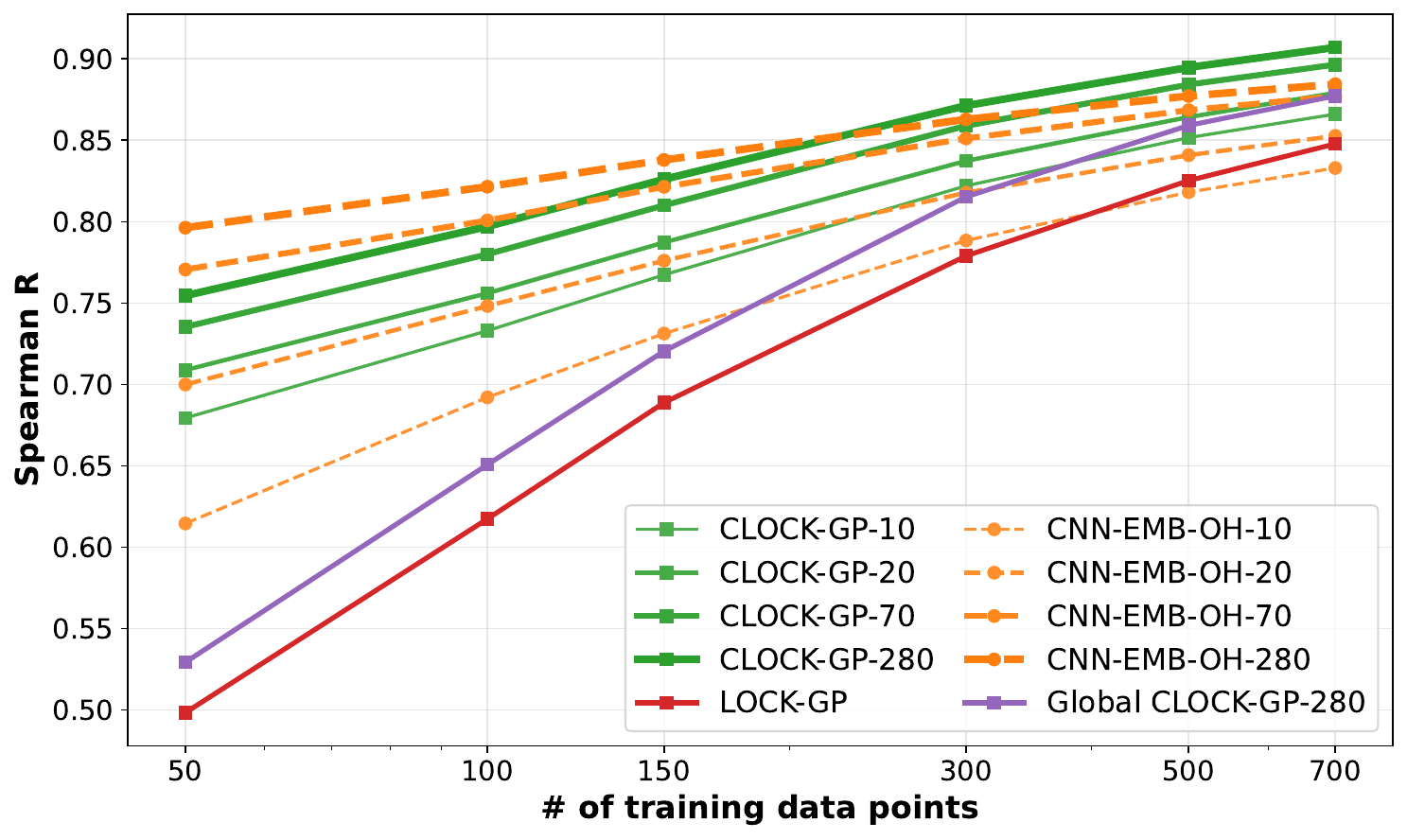}
    \caption{We depict how multi-task model performance changes as we vary the number of training landscapes from $10$ to $280$.
    Spearman R is averaged across $100$ train/test splits in $50$ held-out landscapes. 
    \texttt{CLOCK-GP} performs best in the low landscape regime: e.g.~\texttt{CLOCK-GP-10}, which is trained on $10$ landscapes, roughly matches the 
    performance of \texttt{CNN-EMB-OH-20}. 
    See \secref{sec:multi} for details and \figref{fig:app:clock_gp_land}-\ref{fig:app:landscape_pearson_mae} for more figures.
    }
    \label{fig:landscape_spearman}
\end{figure}

\paragraph{Landscape Subsampling}

Above we trained models on $280$ landscapes, whereas
in many applications we may have significantly less data. 
This motivates a simple question: can we make do with (far) fewer landscapes?
To investigate this systematically we train models on $n_{\rm landscapes} \in \{10, 20, 35, 70, 140 \}$.
For the results see \figref{fig:landscape_spearman}. Encouragingly, we find that \texttt{CLOCK-GP} performs
best in the low-landscape regime. This is arguably expected, since GPs can be very data efficient. 
Put differently, GPs can still provide good predictive performance with a suboptimal kernel, and a modest
number of landscapes may be sufficient to learn a kernel that is ``good enough.'' 
Indeed with just $10$ landscapes\footnote{Since $L\sim50$ this corresponds to $\sim500$ local contexts $\bh_\ell$.} \texttt{CLOCK-GP} outperforms \texttt{Global CLOCK-GP} trained on all $280$ landscapes, 
highlighting the value of structure-conditioned correlation matrices (even imperfect ones). 
We highlight that \texttt{LOCK-GP}---a landscape-local model
that uses a pre-canned substitution matrix---outperforms \texttt{CNN-EMB-OH} trained on $10$ landscapes for $N\!\gtrsim\!500$.
By contrast \texttt{CLOCK-GP} trained on $10$ landscapes outperforms \texttt{LOCK-GP} across the board.
Notably all multi-task models improve as the number of training landscapes increases
(although the improvements for \texttt{Global CLOCK-GP} are modest, see \figref{fig:app:global_clock_gp_land}).

\paragraph{Interpretation}

\begin{figure}[t]
\centering
\includegraphics[width=\columnwidth]{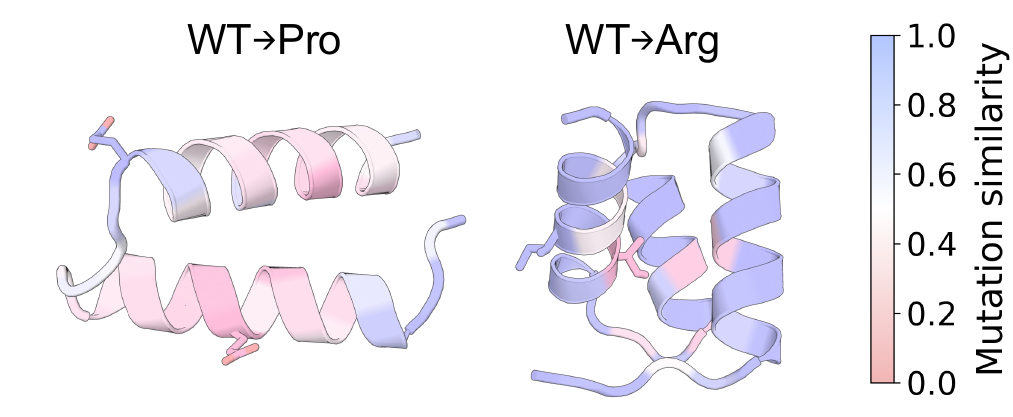}
    \caption{Structure-conditioned amino acid correlations learned by CLOCK on representative structures from \citet{Tsuboyama_2023}.
    Left ({\scriptsize PDB ID: 2JVD}): correlation between the wild-type amino acid and proline at each position.
    Right ({\scriptsize PDB ID: 2JWD}): correlation between the wild-type amino acid and arginine.
    Each structure highlights a pair of sites (namely those with sidechains pictured) with the \emph{same} amino-acid identity 
    but sharply different kernel correlation to the reference amino acid, illustrating how amino acids with identical 
    BLOSUM-derived correlation can have distinct structure-conditioned correlations.
    Blue (resp., red) indicates substitutions judged (dis)similar to the reference residue.}
    \label{fig:kernel_interp}
\end{figure}
As shown in \figref{fig:kernel_interp}, the correlations learned by \texttt{CLOCK-GP} reveal interpretable structural patterns: 
protein {\scriptsize 2JVD} illustrates that CLOCK captures a well-known preference to substitute proline at the N-terminal region of alpha helices \citep{richardson1988amino}, while protein {\scriptsize 2JWD} shows that substitutions to arginine are preferred on the protein surface rather than in the core.
See \secref{app:clock_interpretation} for additional analysis.
For visual representations of the correlation matrices learned by \texttt{(Global) CLOCK-GP} see \figref{fig:app:global-clock-corr}-\ref{fig:app:local-clock-corr-quantiles}, which highlight the special role played by proline in thermostability.
\figref{fig:app:global-clock-corr}-\ref{fig:app:local-clock-corr-quantiles} also reveal that the thermostability data contain sequences with gap tokens, 
confirming the ability of \texttt{(C)LOCK-GP}s to handle landscapes with variable length sequences.

\subsection{Thompson Sampling}
\label{sec:thompson}

We demonstrate the utility of \texttt{LOCK-GP} uncertainty estimates by using Thompson sampling (TS) to 
trade off between exploration and exploitation in an optimization-based protein design problem
with parallel candidate selection.
See \secref{sec:app:ts} and \figref{fig:app:diversity} for the results, which
illustrate how TS offers a natural lever for modulating sequence diversity.
We note that Thompson sampling of GP surrogate models is by no means the only way to leverage uncertainty estimates in the context
of protein design, and we refer the reader to the growing literature on this topic for 
alternative approaches \cite{stanton2022accelerating,hu2023protein,hawkins2023preferential,khan2023toward,rapp2024self}.

\subsection{Binary Classification}
\label{sec:binary}

We demonstrate that \texttt{LOCK-GP} can readily be applied to landscapes with binary labels.
In our benchmark we evaluate most of the model classes considered in \secref{sec:local}.
See \figref{fig:binary} for the results and \secref{sec:app:scalability} and \ref{app:sec:binarydetails} for additional details.

\begin{figure}[t!]
\centering
\includegraphics[width=\columnwidth]{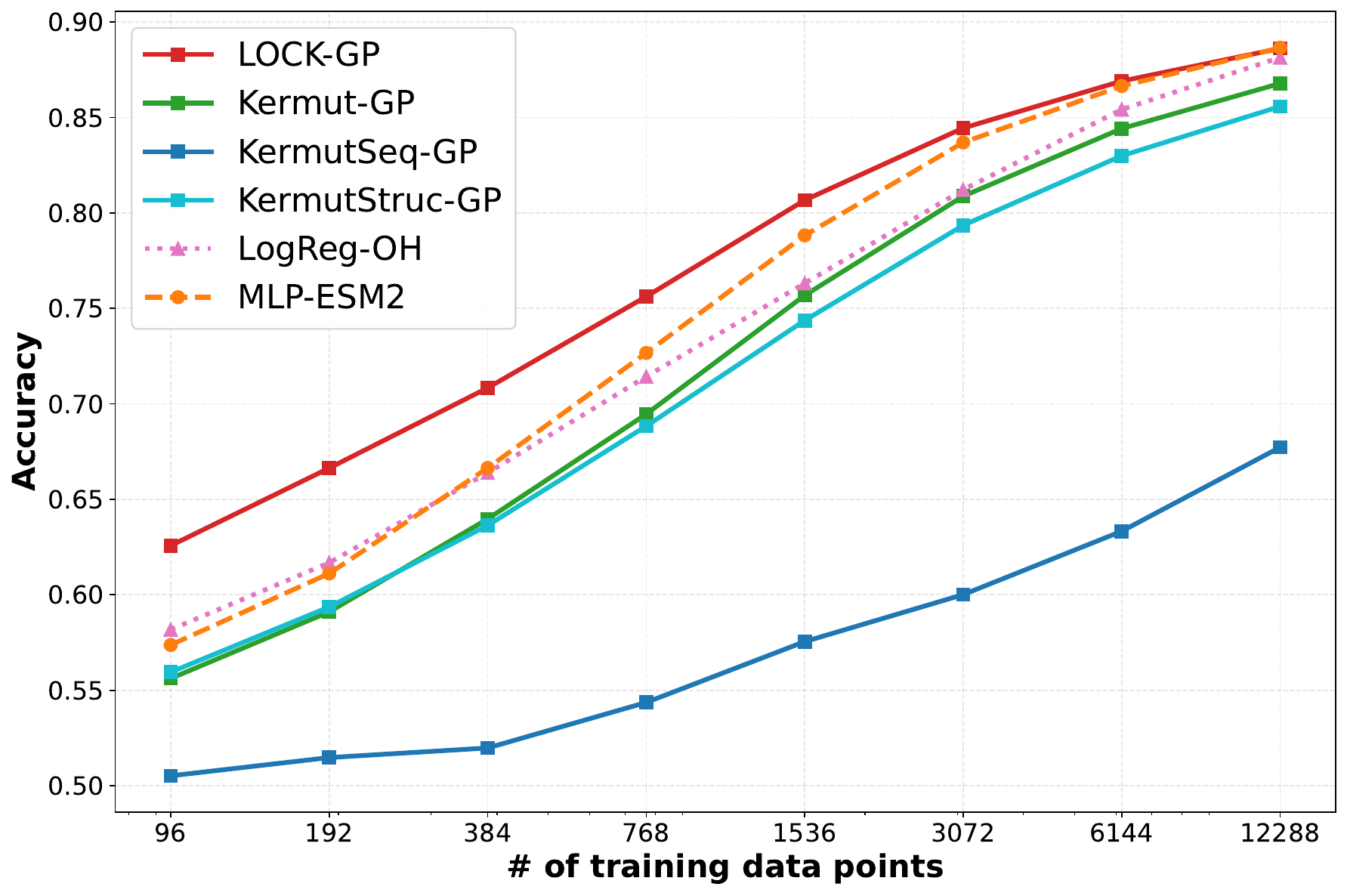}
    \caption{We report binary classification accuracy for models trained 
    on (quantized) fluorescence readouts from \citet{Gonzalez_Somermeyer_2022}
    as a function of the number of training data points. Accuracy is averaged across three replicates.
    \texttt{LogReg-OH} is logistic regression on one-hot features and is the direct analog of \texttt{Ridge-OH} in \secref{sec:local}.
}
    \label{fig:binary}
\end{figure}

%% file: metric_summary_table.tex
\centering
\resizebox{\linewidth}{!}{
\begin{tabular}{lccccccccccccccc}
\toprule
 & \multicolumn{3}{c}{\textbf{Cross-validation} (48 data points)} & \multicolumn{3}{c}{\textbf{Cross-validation} (1536 data points)} & \multicolumn{3}{c}{\textbf{Unseen mutations} (96 data points)} & \multicolumn{3}{c}{\textbf{Extrapolation} (128 data points)} & \multicolumn{3}{c}{\textbf{Extrapolation} (512 data points)} \\
 & Spearman & Pearson & MAE & Spearman & Pearson & MAE & Spearman & Pearson & MAE & Spearman & Pearson & MAE & Spearman & Pearson & MAE \\
\cmidrule(lr){2-4} \cmidrule(lr){5-7} \cmidrule(lr){8-10} \cmidrule(lr){11-13} \cmidrule(lr){14-16}
\texttt{LOCK-GP} & \textbf{0.655} & \textbf{0.682} & \textbf{0.496} & \textbf{0.867} & \textbf{0.914} & \textbf{0.210} & 0.610 & 0.622 & \textbf{0.591} & \textbf{0.669} & \textbf{0.711} & \textbf{0.592} & \textbf{0.759} & \textbf{0.807} & \textbf{0.439} \\
\texttt{Tanimoto-GP} & 0.520 & 0.517 & 0.588 & 0.846 & 0.888 & 0.272 & 0.555 & 0.560 & 0.675 & 0.632 & 0.654 & 0.762 & 0.739 & 0.769 & 0.574 \\
\texttt{Kermut-GP} & 0.638 & 0.639 & 0.547 & 0.850 & 0.888 & 0.285 & \textbf{0.629} & \textbf{0.632} & 0.617 & 0.639 & 0.654 & 0.845 & 0.750 & 0.767 & 0.670 \\
\texttt{KermutSeq-GP} & 0.541 & 0.539 & 0.615 & 0.794 & 0.838 & 0.338 & 0.532 & 0.537 & 0.681 & 0.518 & 0.505 & 0.836 & 0.605 & 0.614 & 0.714 \\
\texttt{KermutStruc-GP} & 0.622 & 0.625 & 0.564 & 0.809 & 0.805 & 0.417 & 0.614 & 0.613 & 0.634 & 0.624 & 0.641 & 0.847 & 0.706 & 0.713 & 0.847 \\
\texttt{MLP-ESM2-LastLayer} & 0.607 & 0.627 & 0.529 & 0.851 & 0.900 & 0.243 & 0.558 & 0.573 & 0.626 & 0.628 & 0.619 & 0.771 & 0.738 & 0.742 & 0.618 \\
\texttt{Ridge-ESM2} & 0.606 & 0.617 & 0.558 & 0.837 & 0.877 & 0.301 & 0.476 & 0.490 & 1.628 & 0.629 & 0.653 & 0.781 & 0.719 & 0.740 & 0.685 \\
\texttt{Ridge-OH} & 0.547 & 0.553 & 0.590 & 0.802 & 0.800 & 0.422 & 0.514 & 0.519 & 0.682 & 0.589 & 0.599 & 0.939 & 0.694 & 0.698 & 0.877 \\
\texttt{SigGLM-OH} & 0.550 & 0.569 & 0.570 & 0.804 & 0.849 & 0.342 & 0.515 & 0.535 & 0.661 & 0.587 & 0.610 & 0.800 & 0.697 & 0.724 & 0.669 \\
\bottomrule
\end{tabular}
}

%% file: model_table.tex
\begin{table}[t]
\centering
\caption{\small Overview of models benchmarked in \secref{sec:local}. 
    See \secref{app:sec:localdetails} for details on each model and Table \ref{tab:metric_summary_long} for
    performance metrics of additional baseline models.}
\label{tab:models}
\resizebox{\linewidth}{!}{%
\begin{tabular}{lll}
\toprule
\small Model Name & \small Model Class & \small Prior Information \\
\midrule
\small \texttt{LOCK-GP}               & \small Gaussian process              & \small BLOSUM \\
\small \texttt{Tanimoto-GP}           & \small Gaussian process              & \small BLOSUM \\
\small \texttt{Kermut-GP}             & \small Gaussian process              & \small ESM-2 + ProteinMPNN \\
\small \texttt{KermutSeq-GP}          & \small Gaussian process              & \small ESM-2 \\
\small \texttt{KermutStruc-GP}        & \small Gaussian process              & \small ESM-2 + ProteinMPNN \\
\small \texttt{MLP-ESM2-LastLayer}    & \small Neural network                & \small ESM-2 \\
\small \texttt{Ridge-ESM2}            & \small Neural network                & \small ESM-2 \\
\small \texttt{Ridge-OH}              & \small Linear                        & --- \\
\small \texttt{SigGLM-OH}             & \small Linear + global non-linearity & --- \\
\bottomrule
\end{tabular}%
}
\end{table}

%% file: discussion.tex
We have seen that Gaussian processes are a powerful tool for modeling protein property landscapes, at least once sufficient care
is taken to build a fit-to-purpose kernel. 
Notably, \texttt{LOCK-GP} frequently outperforms more complex methods that rely on foundation models with millions of parameters.
Moreover, the technology introduced in \secref{sec:method} is of wide application, as we have illustrated 
by, for example, training hybrid GPs that utilize neural zero-shot predictors as well as structure-conditioned kernels that leverage
foundation model embeddings. 
A limitation of our kernels is their restriction to settings in which individual landscapes are either fixed-length or can be aligned to a common length.
Extending our approach to relax this requirement is a natural direction for future work, as it would make (C)LOCK kernels applicable in a wider range of settings.

In \secref{sec:multi} we introduced multi-task GPs that employ zero-shot structure-conditioned kernels. 
While these models do not inherit the extreme parameter efficiency of \texttt{LOCK-GP}---since they leverage the rich features provided by foundation models to enable cross-landscape learning---they can decisively outperform landscape-local models 
with as few as $10$ landscapes.
In some applications, however, even $10$ landscapes may be unavailable.
This raises the intriguing question of whether CLOCK kernels can be pre-trained on a given set of property landscapes 
and then fine-tuned on partially related landscapes---for example, transferring learned thermostability kernels to binding affinity landscapes.
Since $\bW$ in \eqnref{eqn:Wdefn} is a structured tensor of shape $A \times A - 1 \times D_{\rm emb}$---with $\sim 54$k parameters in our experiments---it seems plausible that low-rank updates may be sufficient to adapt $\bW$ to a new set of landscapes.\footnote{Alternatively, one could instead fine-tune the foundation model that generates the embeddings $\bh_{1:L}$.}
We provide a proof-of-concept demonstration of this general approach in \secref{sec:app:aav} 
by pre-training $\bW$ on thermostability data and fine-tuning it on AAV capsid data 
(see \figref{fig:app:aav_spearman_mae}-\ref{fig:app:aav_pearson_rmse}).
This initial positive result is encouraging, 
and we believe this broader line of inquiry is a fruitful direction for future work.

%% file: app.tex
\subsection{Infinitely Divisible Matrices}
\label{app:infdiv}

\input{app_infdiv.tex}

\subsection{Local Linearity}
\label{app:loclin}

\input{app_loclin.tex}

\subsection{Detailed LOCK Kernel Description}
\label{app:lock}

\input{app_lock.tex}

\subsection{Epistasis and LOCK}
\label{app:epistasis}

\input{app_epistasis.tex}

\subsection{GP Ablation Experiment}
\label{app:ablation}

\input{ablation.tex}

\subsection{Injecting Diversity with Thompson Sampling}
\label{sec:app:ts}

\input{ts.tex}

\subsection{CLOCK Kernel Interpretation}
\label{app:clock_interpretation}

\input{clock_interpretation.tex}

\subsection{CLOCK Training Objective}
\label{app:clock}

\input{app_loss.tex}

\subsection{Proof-of-concept CLOCK Fine-tuning Experiment}
\label{sec:app:aav}

\input{aav.tex}

\subsection{Generalizing CLOCK}
\label{sec:app:genclock}

\input{genclock.tex}

\subsection{On the scalability of (C)LOCK GPs}
\label{sec:app:scalability}

\input{app_scalability.tex}

\subsection{Inventory of Substitution Matrices}
\label{app:submat}

\input{submat.tex}

\subsection{Kermut Kernel}
\label{app:kermut}

\input{kermut.tex}

\subsection{Description of Datasets and Evaluation Strategies}
\label{app:data}

\input{data.tex}

\subsection{Speed and Memory Usage Benchmark}
\label{app:speed}

\input{speed.tex}

\subsection{Experimental Details}
\label{app:exp}

\input{app_exp.tex}

\subsection{Additional Figures and Tables for the Local Learning Experiment}
\label{app:add_loc}

\input{app_add_local.tex}

\subsection{Additional Figures and Tables for the Multi-task Learning Experiment}
\label{app:add_multi}

\input{app_add_multi.tex}

\subsection{Code Snippets}
\label{app:code}

See \figref{code:corrmachine} and \figref{code:loss}.

\input{correlation_machine.tex}
\input{loss.tex}

%% file: app_infdiv.tex
\subsubsection{Definitions and Useful Facts}

We begin with some definitions.
\begin{definition}[Positive semidefinite matrix]
A real symmetric matrix $M\in\mathbb{R}^{n\times n}$ is \emph{positive semidefinite} or PSD
if
\[
x^\top M x \;\ge\; 0 \quad\text{for all } x\in\mathbb{R}^n.
\]
We write $M\succeq 0$.
\end{definition}

\begin{definition}[Elementwise exponential]
For $A\in\mathbb{R}^{n\times m}$, the elementwise exponential
$\exp^{\circ}(A)\in\mathbb{R}^{n\times m}$ is defined by
\[
\big(\exp^{\circ}(A)\big)_{ij} = e^{A_{ij}}\qquad(1\le i\le n,\;1\le j\le m).
\]
It is important to emphasize that this is \emph{not} the matrix exponential.
\end{definition}

\begin{definition}[Hadamard (elementwise) power]
Let $A=(A_{ij})\in\mathbb{R}^{n\times m}$ and let $t>0$. The Hadamard (elementwise) power
$A^{\circ t}\in\mathbb{R}^{n\times m}$ is defined by
\[
  \big(A^{\circ t}\big)_{ij} \;=\; (A_{ij})^{\,t}\qquad(1\le i\le n,\;1\le j\le m),
\]
i.e., exponentiation is applied to each entry separately.  
For integer $t$, this is defined for any real $A$; for non-integer real $t$, we require
$A_{ij}\ge 0$ to remain in the real numbers. 
\end{definition}

\begin{definition}[Schur (elementwise) logarithm]
Let $A=(A_{ij})\in\mathbb{R}^{n\times m}$ have strictly positive entries ($A_{ij}>0$).
The Schur (elementwise) logarithm of $A$ is
\[
  \log^{\circ}(A) \;\in\; \mathbb{R}^{n\times m},\qquad
  \big(\log^{\circ}(A)\big)_{ij} \equiv \log(A_{ij}).
\]
It is defined elementwise and should be distinguished from the matrix logarithm.
\end{definition}

\begin{definition}[Infinitely divisible matrix]
A symmetric matrix $K \in \mathbb{R}^{n\times n}$ with strictly positive entries is said to be
\emph{infinitely divisible} if its Hadamard (elementwise) powers 
$K^{\circ t}$ are positive semidefinite for all $t > 0$.
\end{definition}

\begin{definition}[Conditionally positive semidefinite matrix]
A real symmetric matrix $K\in\mathbb{R}^{n\times n}$ is \emph{conditionally positive semidefinite} or CPSD
if
\[
x^\top K x \;\ge\; 0 \quad\text{for all } x\in\mathbb{R}^n
\text{ with } \sum_{i=1}^n x_i=0.
\]
Equivalently, letting $J \equiv I-\frac{1}{n}\mathbf{1}\mathbf{1}^\top$ denote the orthogonal projector
onto the zero-sum subspace, $K$ is CPSD iff $J K J \succeq 0$.
A real symmetric matrix $K$ is \emph{conditionally negative semidefinite} or CNSD if $-K$ is CPSD.
\end{definition}
Finally, we state Schoenberg's theorem, as we will use it below.
\begin{theorem}[\citet{Schoenberg1938}]
\label{thm:schoenberg}
Let $\psi$ be symmetric with $\psi(i,i)=0$. Then the following are equivalent:
\begin{enumerate}
\item $\psi$ is conditionally negative semidefinite;
\item for every $t>0$, the kernel $k_t(i,j)=\exp\!\big(-t\,\psi(i,j)\big)$ is positive semidefinite;
\item for finite index sets, $\psi$ is a squared Euclidean distance.
\end{enumerate}
\end{theorem}

\subsubsection{Characterizing Infinitely Divisible Matrices}

With these definitions in hand, we can give a precise characterization of infinitely divisible matrices.
\begin{lemma}[\citet{berg1984harmonic}]
\label{lem:inf-div}
Let $K$ be a symmetric matrix with strictly positive entries. Then
$K$ is infinitely divisible if and only if $\log^{\circ} K$ is
\emph{conditionally positive semidefinite}, i.e.
\[
x^\top \big(\log^{\circ} K\big)\, x \;\ge\; 0
\quad \text{for all } x \in \mathbb{R}^n
\text{ with } \sum_i x_i = 0.
\]
\end{lemma}

\begin{proof}[Proof sketch]
We give two standard arguments; see \citet{berg1984harmonic} for a detailed proof.

\paragraph{(CPSD $\Rightarrow$ ID via Schoenberg)}
Assume $A\equiv\log^{\circ} K$ is CPSD. Define
\[
d(i,j) \;\equiv\; \frac{A_{ii}+A_{jj}}{2}-A_{ij}\,.
\]
Then $d(i,i)=0$, and for any $x$ with $\sum_i x_i=0$,
\[
\sum_{i,j} x_i x_j\, d(i,j)
= -\, x^\top A x \;\le\; 0,
\]
so $d$ is CNSD. By Schoenberg's theorem, the kernel $B^{(t)}$ with entries
$B^{(t)}_{ij} = \exp\!\big(-t\, d(i,j)\big)$ is PSD for every $t>0$.
Noting the factorization
\[
B^{(t)}_{ij}
= \exp\!\Big(-\tfrac{t}{2}A_{ii}\Big)\; \exp\!\big(t A_{ij}\big)\; \exp\!\Big(-\tfrac{t}{2}A_{jj}\Big)
= \big(D_t\, K^{\circ t}\, D_t\big)_{ij},\quad
D_t\equiv\operatorname{diag}\!\big(e^{-t A_{11}/2},\dots,e^{-t A_{nn}/2}\big),
\]
we have $B^{(t)} = D_t\, K^{\circ t}\, D_t$. Since sandwiching with $D_t^{-1}$ preserves PSD,
$K^{\circ t}\succeq 0$ for every $t>0$, i.e., $K$ is infinitely divisible.

\paragraph{(ID $\Rightarrow$ CPSD)}
Assume $K$ is infinitely divisible. Fix $x\in\mathbb{R}^n$ with $\sum_i x_i=0$ and set
\[
g_x(t) \;\equiv\; x^\top K^{\circ t} x \;=\; \sum_{i,j} x_i x_j\, K_{ij}^{\,t}\qquad (t\ge 0).
\]
For $t>0$, $K^{\circ t}\succeq 0$ so $g_x(t)\ge 0$. As $t\downarrow 0$,
$K^{\circ t}\to \mathbf{1}\mathbf{1}^\top$ elementwise, hence
$g_x(0)\equiv\lim_{t\downarrow 0}g_x(t)=(\sum_i x_i)^2=0$.
Thus $g_x$ has a right minimum at $t=0$, so $g'_x(0^+)\ge 0$.
Differentiating termwise at $t=0$ we obtain
\[
g'_x(0^+) \;=\; \sum_{i,j} x_i x_j \log K_{ij}
\;=\; x^\top \big(\log^{\circ} K\big)\, x \;\ge\; 0,
\]
which is exactly conditional positive semidefiniteness of $\log^{\circ} K$.
\end{proof}

\subsubsection{Characterizing Infinitely Divisible Correlation Matrices}

We can also characterize infinitely divisible \emph{correlation} matrices.

\begin{theorem}[Characterization via squared Euclidean distances]
\label{thm:IDcorr-euclid}
Let $K\in\mathbb{R}^{A\times A}$ be a correlation matrix with strictly positive entries.
The following are equivalent:
\begin{enumerate}
\item $K$ is infinitely divisible.
\item $\log^{\circ} K$ is CPSD (equivalently $D\equiv-\log^{\circ} K$ is CNSD).
\item There exist points $x_1,\dots,x_A$ in $\mathbb{R}^m$ (for some $m\le A-1$) such that
\[
K_{ij} \;=\; \exp\!\big(-\|x_i-x_j\|^2\big)\qquad\text{for all }i,j.
\]
\end{enumerate}
In particular, every infinitely divisible correlation matrix is of the form
$K=\exp^{\circ}(-D)$ where $D$ is a \emph{squared Euclidean distance matrix}.
\end{theorem}

\begin{proof}[Proof]
By Lemma~\ref{lem:inf-div}, \emph{(1)}$\;\Leftrightarrow\;$\emph{(2)}.
By Schoenberg’s Theorem~\ref{thm:schoenberg}, \emph{(2)}$\;\Leftrightarrow\;$\emph{(3)}:
$D$ is CNSD iff $e^{-tD}$ is PSD for all $t>0$, and for finite sets such a CNSD $D$
is precisely a squared Euclidean distance.
\end{proof}

\begin{remark}
    For all $n\times n$ doubly nonnegative matrices $A$ (i.e.~PSD matrices with all nonnegative entries), 
    the Hadamard power $A^{\circ t}$ preserves positive semidefiniteness iff $t\in\mathbb{N}$ or $t\ge n-2$ \cite{FitzGeraldHorn1977}.
    This severe dimension-dependent restriction on $t$ motivates our use of infinitely divisible matrices, which
    enable learning the exponent $t$ with gradient methods for all positive $t$.
\end{remark}

%% file: app_loclin.tex
Suppose $\bx \in \RR^D$ and let $f(\bx) = \bbeta(\bx) \cdot \bx = \sum_{d=1}^D \beta_d(\bx) x_d$. 
We place an independent GP prior on each component $\beta_d(\bx)$ of $\bbeta(\bx)$ with zero mean and with kernel $k_d(\cdot, \cdot)$,
i.e.~$\langle \beta_d(\bx) \rangle=0$ and $\langle \beta_d(\bx) \beta_{d}(\by) \rangle=k_d(\bx, \by)$.
Since each $\beta_d(\bx)$ is Gaussian distributed and $f(\bx)$ is linear in $\bbeta(\bx)$, $f(\bx)$ is evidently Gaussian.
To determine the GP prior that controls $f(\bx)$ we need to compute the corresponding mean and covariance.
First compute the mean 
\begin{equation}
\langle f(\bx) \rangle = \langle \bbeta(\bx) \cdot \bx \rangle  = \langle \bbeta(\bx) \rangle  \cdot \bx  = 0
\end{equation}
Next compute the covariance
\begin{equation}
    \label{eqn:prodkernel}
    \langle f(\bx) f(\by) \rangle = \langle \left (\bbeta(\bx) \cdot \bx \right)  \left(\bbeta(\by) \cdot \by \right)\rangle
    = \sum_{d,d^\prime} x_d y_{d^\prime} \langle \beta_d(\bx) \beta_{d^\prime}(\by) \rangle = 
    \sum_{d} x_d y_d \langle \beta_d(\bx) \beta_{d}(\by) \rangle = \sum_{d} x_d y_d k_d(\bx, \by)
\end{equation}
If we specialize to the case that $k_d(\bx, \by)=\knl(\bx, \by)$ for $d=1,...,D$ then the kernel in \eqnref{eqn:prodkernel} simplifies
to $\knl(\bx, \by)  \klin(\bx, \by)$ where $\klin(\bx, \by) \equiv \bx^{\rm T} \by = \bx \cdot \by$ is the canonical linear kernel. 
So we conclude that $f(\bx) = \bbeta(\bx) \cdot \bx$ is a zero-mean Gaussian process governed by the product kernel $\knl(\bx, \by) \klin(\bx, \by)$. 

Now suppose that $\bbeta(\bx)$ is a vector-valued Gaussian process with zero mean and with kernel
$\langle \beta_d(\bx) \beta_{d^\prime}(\by) \rangle=U_{d d^\prime} \knl(\bx, \by)$ where $\bU$ is a PSD matrix. 
Then \eqnref{eqn:prodkernel} becomes
\begin{equation}
    \label{eqn:prodkernel2}
    \langle f(\bx) f(\by) \rangle = \sum_{d,d^\prime} x_d y_{d^\prime} \langle \beta_d(\bx) \beta_{d^\prime}(\by) \rangle 
    = \sum_{d,d^\prime} x_d y_{d^\prime} U_{d d^\prime} \knl(\bx, \by) = \left( \bx^{\rm T} \bU \by \right)  \knl(\bx, \by)
\end{equation}
In other words the resulting kernel is still a product of a linear kernel and $\knl$, only now the linear kernel is 
$\bx^{\rm T} \bU \by$ instead of the isotropic kernel $\bx^{\rm T} \by$. 
This is precisely the scenario we have in \secref{sec:kernel}, where $\bU$ is a large block diagonal covariance matrix of 
size $(LA) \times (LA)$ where there are $L$ blocks of size $A \times A$ and each block is a correlation matrix $\bC_\ell$, 
see \eqnref{eqn:linkernel}.\footnote{Of course in the LOCK kernel correlation matrices $\bC_\ell$ \emph{also appear} in the 
non-linear kernel \eqnref{eqn:nlkernel}.}

Finally, we consider the case where $\bbeta(\bx)$ is a vector-valued Gaussian process with zero mean and with kernel given by
\begin{equation}
\langle \beta_d(\bx) \beta_{d^\prime}(\by) \rangle=U_{d d^\prime} \left(\kappa + \knl(\bx, \by) \right)
\end{equation}
where $\kappa > 0$ is a fixed constant and we note that a constant covariance function is equivalent to including an intercept term in  
the function class (with its prior covariance set by $\kappa$). 
In this case we find that the effective kernel for $f(\bx)$ is given by 
\begin{equation}
    \label{eqn:kappakernel}
    \left( \bx^{\rm T} \bU \by \right) \left(\kappa + \knl(\bx, \by) \right) = \kappa \left( \bx^{\rm T} \bU \by \right) +  \knl(\bx, \by)  \left( \bx^{\rm T} \bU \by \right)
\end{equation}
In other words we have a kernel that is of the form ``linear kernel plus non-linear kernel times linear kernel.''
This is precisely the form of the LOCK kernel in \eqnref{eqn:lockdefn2}, where the only difference is that in the latter 
we have decoupled the hyperparameters used in each kernel in the sum.\footnote{In particular we introduce separate exponents 
$\alpha^{(1)}$, $\alpha^{(2)}$, and $\alpha_\ell$; see \secref{app:lock}.}
This decoupling does not dramatically expand the class of functions under consideration, but it does make the kernel a bit more flexible. 
Put differently, we note that, depending on the application, it may be undesirable for the covariance function that controls $\bbeta(\bx)$ to revert to zero far
away from the training data.
By including $\kappa$ we instead revert to a learnable constant $\kappa \bU$, and the result is an `extra' linear kernel in \eqnref{eqn:kappakernel}.

Whatever its precise form, if the kernel $\knl(\bx, \by)$ is such that $\bbeta(\bx)$ varies
slowly across the input space $\bx$, then $f(\bx)$ is naturally viewed as locally linear. We note that nothing in
our derivation requires $\knl(\bx, \by)$ to be non-linear, but this is the most natural ansatz in practice.

%% file: app_lock.tex
We restate the definitions of the two base kernels used in LOCK (see \eqnref{eqn:linkernel}-\ref{eqn:nlkernel})
\vspace{-1mm}
\begin{equation}
    \label{eqn:basekernels}
    \kLlin(\bx, \by) = \Sigma_{\ell=1}^L \bx_\ell^{\rm T} \bC_\ell^{\alpha_\ell} \by_\ell \qquad \qquad
    \kLnl(\bx, \by) = \prod_{\ell=1}^L \bx_\ell^{\rm T} \bC_\ell^{\alpha_\ell} \by_\ell
\end{equation}
\vspace{-1mm}
as well as the full LOCK kernel
\begin{align}
    \label{eqn:lockdefn2}
    \kL(\bx, \by) = \sigma_1^2 \kLnl(\bx, \by) \kLlin(\bx, \by)  + \sigma_2^2 \kLlinp(\bx, \by)
\end{align}
The first linear kernel $\kLlin$ is equipped with a global exponent $\alpha^{(1)}$, while the second
linear kernel $\kLlinp$ is equipped with a global exponent $\alpha^{(2)}$. 
The non-linear kernel $\kLnl$ is equipped with positionwise exponents $\alpha_\ell$.
To be explicit the first linear kernel is given by
\vspace{-1mm}
\begin{equation}
    \label{eqn:linkernel1defn}
    \kLlin(\bx, \by) = \Sigma_{\ell=1}^L \bx_\ell^{\rm T} \bC_\ell^{\alpha^{(1)}} \by_\ell
\end{equation}
\vspace{-1mm}
and the second linear kernel is given by
\vspace{-1mm}
\begin{equation}
    \label{eqn:linkernel2defn}
    \kLlinp(\bx, \by) = \Sigma_{\ell=1}^L \bx_\ell^{\rm T} \bC_\ell^{\alpha^{(2)}} \by_\ell
\end{equation}
while the non-linear kernel is as in \eqnref{eqn:basekernels}.
In the regression case we also have the noise scale $\sigma_n$ (see \eqnref{eqn:mll}).
Thus we have $L+5$ hyperparameters: $\{ \sigma_1, \sigma_2, \sigma_n, \alpha^{(1)}, \alpha^{(2)}, \alpha_{1:L} \}$.
While it would of course be possible to introduce local exponents for each linear kernel---or even additional positive multiplicative parameters
that modulate the overall scale of each correlation matrix $\bC_\ell$---we find that this choice strikes a good balance between
flexibility and parsimony.

\subsubsection{Normalization}
\label{app:norm}
In practice we use the following normalization scheme for correlation matrices. 
First we start with a raw substitution matrix, e.g.~BLOSUM50, denoted by $\bS$ (assumed to be exponentiated, i.e.~not in log-odds space). 
We then define the corresponding correlation matrix $C_{aa^\prime} \equiv S_{aa^\prime} / \sqrt{S_{aa} S_{a^\prime a^\prime}}$.
In order to make it straightforward to define priors over exponent hyperparameters we do one final normalization step.
In particular we raise $\bC$ to the unique power that makes the median off-diagonal entry of $\bC$ equal to $\exp(-\tfrac{1}{4})$.
In other words $\bC \rightarrow \bC^{\circ t}$ where $t=-\tfrac{1}{4 \log m}$ and $m$ is the median of the off-diagonal entries of $\bC$.
An example of such a normalized correlation matrix can be seen in the right panel of \figref{fig:blosum-corr}.

\subsubsection{Hyperparameter Priors}
\label{app:hyper}
We place ${\rm Gamma}(2, 2)$ priors on both kernel variances $\sigma_1^2$ and $\sigma_2^2$ and 
place a ${\rm Gamma}(2, 2)$ on the observation noise variance $\sigma_n^2$.\footnote{We also use these priors for \texttt{Tanimoto-GP}; for Kermut we
use the priors in \citet{groth2024kermut}.}
We place a ${\rm LogNormal}(0, 1)$ prior on global exponents $\alpha^{(1)}$ and $\alpha^{(2)}$.
We use an overcomplete parameterization for local exponents, i.e.~$\alpha_\ell = \tilde{\alpha} \tilde{\alpha}_\ell$,
where $ \tilde{\alpha}$ is a scalar and where $\tilde{\alpha}_{1:L} \in \RR^L$.
We then place a ${\rm LogNormal}(0, 1)$ prior on the scalar $\tilde{\alpha}$ and a 
${\rm LogNormal}(0, \tfrac{1}{4})$ prior on each $\tilde{\alpha}_\ell$.
\emph{Marginally}, this choice corresponds to a ${\rm LogNormal}(0, \tfrac{5}{4})$ prior on $\alpha_\ell$. 
We note that choosing an overcomplete parameterization has an impact on the \emph{optimization dynamics}; in particular
by introducing $\tilde{\alpha}$ we expect it to be easier to take larger steps in $\alpha_\ell$ space.
These priors were chosen based on regression experiments conducted with datasets disjoint from those
used in our experiments in \secref{sec:exp}.

%% file: app_epistasis.tex
\input{gb1_spearman_mae_cv_extrapolation_table.tex}

The LOCK base kernels in \eqnref{eqn:linkernel}-\ref{eqn:nlkernel} are defined residuewise and do not introduce 
explicit learnable parameters for specific pairwise or higher-order interactions. 
Nevertheless, the full LOCK kernel in \eqnref{eqn:lockdefn} can induce strongly non-additive behavior, as follows. 
First, the non-linear kernel $\knl$ has a multiplicative structure across positions, which couples residuewise contributions and makes motif-level effects natural to represent. 
Second, because the full kernel is locally linear, the effective coefficient of a mutation can vary smoothly with sequence context. 
In this sense, LOCK is best viewed as modeling smooth, context-dependent epistasis rather than explicit sparse interaction terms.
As described in \secref{sec:kernelprop}, the linear kernel uses rotated residue-level basis functions, 
while the non-linear kernel uses their tensor products. 
Consequently, observing that a particular multi-residue motif is associated with high fitness tends to increase predicted fitness for 
nearby motifs composed of biophysically similar residues. 
This allows LOCK to capture a broad class of higher-order effects that are mediated by shared residue similarities and local smoothness.
Accordingly, LOCK is particularly well-suited to modeling epistasis between classes of biophysically similar substitutions, 
rather than to arbitrary interactions between mutations.
Notably, LOCK does not separately parameterize each pairwise, three-way, or higher-order interaction. 
The induced epistatic structure is constrained: it is determined implicitly by the residuewise correlation matrices $\bC_{1:L}$, the learned exponents, and the multiplicative/local-linear kernel structure. 
This is arguably a feature in the low-data regime, where attempting to learn unconstrained higher-order epistasis directly is generally infeasible.
A useful analogy might be the following. 
Suppose $z$ is a gaussian random variable with mean zero and covariance matrix $\Sigma$. 
Then higher-order correlations like $\langle z_i z_j z_k z_l \rangle$ are entirely 
determined by $\Sigma_{ij} = \langle z_i z_j \rangle$. 
Similarly, in LOCK, the induced higher-order epistatic structure is governed by the lower-order correlation matrices $\bC_{1:L}$.

In Table~\ref{tab:gb1_spearman_mae} we benchmark \texttt{LOCK-GP} against two baselines on the GB1 dataset \cite{wu2016adaptation},
which has four variable positions and is known to exhibit significant pairwise as well as higher-order epistasis in these positions.
We find that all three models struggle to provide good fits, especially in the low-data regime, although
the neural model \texttt{MLP-ESM2-LastLayer} achieves somewhat better performance, at least for $N=1536$ training data points.
These results highlight the challenge of obtaining good predictive models for highly epistatic landscapes in the low-data regime.

%% file: gb1_spearman_mae_cv_extrapolation_table.tex
\begin{table}[t]
\centering
\begin{tabular}{@{}lcccccc@{}}
\toprule
 & \multicolumn{3}{c}{Spearman R} & \multicolumn{3}{c}{MAE} \\
\cmidrule(lr){2-4} \cmidrule(lr){5-7}
 & \texttt{LOCK-GP} & \texttt{Kermut-GP} & \texttt{MLP-ESM2-LastLayer} & \texttt{LOCK-GP} & \texttt{Kermut-GP} & \texttt{MLP-ESM2-LastLayer} \\
\midrule
Cross-validation ($N=48$) & 0.22 & 0.13 & 0.25 & 0.31 & 0.32 & 0.27 \\
Cross-validation ($N=1536$) & 0.42 & 0.45 & 0.53 & 0.26 & 0.36 & 0.14 \\
Extrapolation ($N=128$) & 0.32 & 0.21 & 0.36 & 0.41 & 1.14 & 0.42 \\
Extrapolation ($N=512$) & 0.34 & 0.26 & 0.44 & 0.62 & 1.03 & 0.33 \\
\bottomrule
\end{tabular}
    \caption{We report Spearman R and MAE metrics for the GB1 dataset in both the cross-validation and extrapolation regime,
    with the number of training data points ranging from $N=48$ to $N=1536$. 
    The extrapolation training and test sets are defined using a Hamming distance cutoff of $D=2$; c.f.~\secref{app:data}.}
\label{tab:gb1_spearman_mae}
\end{table}

%% file: ablation.tex
\input{ablation_tables.tex}

To better understand the effect of different choices that enter into the LOCK kernel in \eqnref{eqn:lockdefn}, 
we perform a systematic ablation experiment in which we train $14$ additional models. 
For the results see Tables 
\ref{tab:ablation_MLDR-011740_96}-\ref{tab:ablation_MLDR-011736_side_by_side}. 
To make cross-comparison easier, we also report results for the three Kermut variants and \texttt{Tanimoto-GP}.
In the following we dissect the results in detail.

While we use the BLOSUM50 substitution matrix by default, 
there are many more BLOSUM variants (see Table \ref{tab:substitution-matrices}). 
In the ablation we compare to BLOSUM45, BLOSUM62, and BLOSUM80. 
We see that the differences are generally quite small, which is probably to be expected since the normalization step in \eqnref{eqn:corrdefn} substantially reduces the differences between different BLOSUM variants. 
While we do not show the results here, we find that this weak dependence also holds for individual landscapes.

Next we consider the hyperparameter priors described in \secref{sec:hyper}, 
focusing on the exponents $\alpha_\ell$ that enter $\kLnl$.
Instead of removing the priors entirely, we instead place a much weaker prior on $\alpha_\ell$, 
as removing the prior entirely leads to numerical instabilities.\footnote{In particular we
place a ${\rm LogNormal}(0, 16)$ prior on the scalar $\tilde{\alpha}$ and a
${\rm LogNormal}(0, 16)$ prior on each $\tilde{\alpha}_\ell$; c.f.~\secref{app:hyper}.}
Under the weak prior---see \texttt{LOCK-WeakPrior}---predictive performance is impaired across the board.

Next we consider the single linear kernel that enters into the LOCK kernel in \eqnref{eqn:lockdefn}, i.e.~the
term $\sigma_2^2 \kLlinp(\bx, \by)$. 
We remove this kernel component entirely, resulting in \texttt{LOCK-NoLinear}.
We find that while uncertainty-unaware metrics like SpearmanR and MAE are very similar for \texttt{LOCK} and \texttt{LOCK-NoLinear},
the NLL of the latter is degraded for the most challenging OOD regimes. 
In particular in the unseen mutations regime (resp., extrapolation regime with $128$ data points) \texttt{LOCK} obtains a lower
NLL than \texttt{LOCK-NoLinear} on $15/21$ (resp., $14/21$) landscapes.
For this reason, and for the reasons outlined in \secref{sec:kernelprop}, we prefer to keep the term $\sigma_2^2 \kLlinp(\bx, \by)$,
although we would generally expect to get good predictive performance without it.

With reference to \secref{app:lock}, we see that the only `local'---i.e.~residuewise---parameters in LOCK are the exponents
$\alpha_\ell$ that modulate the non-linear kernel $\kLnl$. 
In \texttt{LOCK-DoublyLocal} we also include residuewise exponents in the single linear kernel $\sigma_2^2 \kLlinp(\bx, \by)$.
From Tables \ref{tab:ablation_MLDR-011740_96}-\ref{tab:ablation_MLDR-011736_side_by_side} we can see that
this additional flexibility can lead to slightly improved performance in some (but not all) cases.
Since the differences are small, and since this change approximately doubles the total number of parameters (from $L+5$ to $2L+4$),
we prefer the more parsimonious choice in LOCK.

Next we consider the base kernels $\kLnl$ and $\kLlin$ in isolation, considering both global and local versions of each (i.e.~with scalar and residuewise exponents, respectively; c.f.~\secref{sec:hyper}). 
First we note that it is generally advantageous to include the additional flexibility of residuewise parameters.
Next we note the superiority of the non-linear kernel $\kLnl$ to the linear kernel $\kLlin$.
Finally we note that while $\kLnl$ in isolation approaches the performance of LOCK for uncertainty-unaware metrics 
like SpearmanR and MAE, it has noticeably worse NLL for the more challenging OOD regimes.
In particular in the unseen mutations regime (resp., extrapolation regime with $128$ data points) \texttt{LOCK} obtains a lower
NLL than \texttt{LOCK-NonLinearOnly-Local} on $18/21$ (resp., $14/21$) landscapes.
For this reason we prefer the locally linear LOCK kernel, although we would generally expect good performance
from $\kLnl$ alone (with the caveat that using $\kLnl$ in isolation results in mean-reverting predictions away from the training data, which may be undesirable depending on the application).
Taken together the results in Tables \ref{tab:ablation_MLDR-011740_96}-\ref{tab:ablation_MLDR-011736_side_by_side}  
demonstrate that the most crucial ingredient in LOCK is arguably fit-to-purpose incorporation of biophysical similarity as encoded by 
(infinitely divisible) BLOSUM substitution matrices.

Next we consider an isotropic linear kernel as in \eqnref{eqn:isolin}. 
We find that \texttt{Linear} is uniformly outperformed by \texttt{LOCK-LinearOnly-Global},
highlighting the value of incorporating substitution matrices into the 
kernel---and as will be confirmed further in the non-linear case in the following paragraph. 

Next we consider an RBF kernel as in \eqnref{eqn:rbfkernel} with learnable lengthscales for each residue. 
We place a ${\rm Gamma}(4, 2)$ prior on each lengthscale $\tau_\ell$, i.e.~the prior probability is maximized for
$\tau_\ell=2$ so that the typical `similarity' of nonidentical amino acids is $\sim \exp(-1/4) \approx 0.78$.
The anisotropic RBF kernel performs relatively poorly across the board, 
emphasizing the value of incorporating substitution matrices into the kernel (we note that these improvements tend to be larger than in the linear case
explored in the previous paragraph).
Indeed on some datasets the NLL is catastrophically bad ($> 10^8$).
Moreover the RBF kernel can be numerically unstable: note that we did not obtain complete results for the RBF kernel
in Table \ref{tab:ablation_MLDR-011736_side_by_side}.

Finally we consider whether the Tanimoto kernel exhibits strong dependence on the BLOSUM substitution matrix used.
We consider BLOSUM62 (the default) as well as BLOSUM50. 
We find little difference in predictive performance.
We also compare to the original Tanimoto kernel (i.e.~without eigenvalue clamping) described in \citet{gessner2024active-mlsb} (see \secref{app:sec:localdetails} for discussion),
which we find to perform slightly worse on average.

\paragraph{Additional Discussion}

We have also considered LOCK variants that utilize matrix exponentiation instead of Hadamard exponentiation to modulate
correlation matrices $\bC_\ell$. While we find that this is viable, it is less numerically stable.
As such we prefer using Hadamard exponentiation; this choice is particularly important for CLOCK, where numerical stability
is crucial for robust learning of kernels parameterized by tens of thousands of parameters. 

We have also considered LOCK variants analogous to \texttt{LOCK-DoublyLocal} described above in which the residuewise parameters
are tied together. We find that this can work well on some datasets, but works poorly on others. 
Consequently we prefer the simpler LOCK.

More broadly, instead of selecting a single LOCK variant, we could do a small kernel search and choose the best kernel architecture for each dataset. 
While we do not pursue this option here, it is likely that by doing so we could squeeze some additional predictive performance from LOCK, at least on some datasets.

%% file: ablation_tables.tex
\begin{table}[htbp]
\centering
\input{table_MLDR-011740_96.tex}
\normalsize
\caption{Performance metrics in the unseen mutations regime with $96$ training data points for the GP ablation experiment.
RMSE is root mean squared error; CRPS is continuous ranked probability score \cite{gneiting2007strictly}; NLL is negative log likelihood;
ECE is expected calibration error.
    Note that we obtained constant predictions for some \texttt{Linear} predictors in this regime, so we do not include results
    for this model.
See \secref{app:ablation} for discussion.
    }
\label{tab:ablation_MLDR-011740_96}
\end{table}

\begin{table}[htbp]
\centering
\input{table_MLDR-011737_side_by_side.tex}
\normalsize
\caption{Performance metrics in the Hamming-distance-based extrapolation regime for the GP ablation experiment, with the number
    of training data points equal to either $128$ or $512$.
RMSE is root mean squared error; CRPS is continuous ranked probability score \cite{gneiting2007strictly}; NLL is negative log likelihood;
ECE is expected calibration error.
See \secref{app:ablation} for discussion.
    }
\label{tab:ablation_MLDR-011737_side_by_side}
\end{table}

\begin{table}[htbp]
\centering
\input{table_MLDR-011736_side_by_side.tex}
\normalsize
\caption{Performance metrics in the cross-validation regime for the GP ablation experiment, 
    with the number of training data points ranging from $48$ to $1536$.
RMSE is root mean squared error; CRPS is continuous ranked probability score \cite{gneiting2007strictly}; NLL is negative log likelihood;
ECE is expected calibration error. Note we did not obtain complete results for the RBF kernel due to numerical issues.
See \secref{app:ablation} for discussion.
    }
\label{tab:ablation_MLDR-011736_side_by_side}
\end{table}

%% file: table_MLDR-011740_96.tex
\resizebox{0.66\linewidth}{!}{
\begin{tabular}{lccccccc}
\toprule
Model & SpearmanR & PearsonR & MAE & RMSE & CRPS & NLL & ECE \\
\midrule
\texttt{LOCK} & 0.610 & 0.622 & 0.591 & 0.823 & 0.456 & 2.560 & 0.113 \\
\midrule
\texttt{LOCK-45} & 0.607 & 0.620 & 0.592 & 0.824 & 0.457 & 2.570 & 0.113 \\
\texttt{LOCK-62} & 0.609 & 0.621 & 0.592 & 0.825 & 0.458 & 2.624 & 0.115 \\
\texttt{LOCK-80} & 0.609 & 0.621 & 0.592 & 0.824 & 0.457 & 2.604 & 0.113 \\
\midrule
\texttt{LOCK-WeakPrior} & 0.574 & 0.575 & 0.591 & 0.884 & 0.491 & 8.933 & 0.161 \\
\texttt{LOCK-NoLinear} & 0.609 & 0.622 & 0.596 & 0.824 & 0.456 & 3.004 & 0.106 \\
\texttt{LOCK-DoublyLocal} & 0.612 & 0.623 & 0.587 & 0.821 & 0.454 & 2.554 & 0.117 \\
\midrule
\texttt{LOCK-LinearOnly-Global} & 0.527 & 0.538 & 0.670 & 0.890 & 0.503 & 2.525 & 0.114 \\
\texttt{LOCK-LinearOnly-Local} & 0.534 & 0.544 & 0.664 & 0.890 & 0.497 & 2.665 & 0.107 \\
\texttt{LOCK-NonLinearOnly-Global} & 0.598 & 0.607 & 0.625 & 0.842 & 0.473 & 2.735 & 0.118 \\
\texttt{LOCK-NonLinearOnly-Local} & 0.608 & 0.622 & 0.599 & 0.824 & 0.459 & 3.304 & 0.108 \\
\midrule
\texttt{Linear} & N/A & N/A & N/A & N/A & N/A & N/A & N/A \\
\texttt{RBF} & 0.565 & 0.564 & 1.710 & 6.509 & 1.544 & $8.3 \times 10^{8}$ & 0.176 \\
\midrule
\texttt{Tanimoto} & 0.555 & 0.560 & 0.675 & 0.881 & 0.499 & 2.586 & 0.125 \\
\texttt{Tanimoto-50} & 0.556 & 0.562 & 0.674 & 0.880 & 0.498 & 2.497 & 0.124 \\
\texttt{Tanimoto-Orig} & 0.547 & 0.552 & 0.678 & 0.883 & 0.501 & 2.550 & 0.127 \\
\midrule
\texttt{Kermut} & 0.629 & 0.632 & 0.617 & 0.827 & 0.465 & 2.773 & 0.097 \\
\texttt{KermutSeq} & 0.532 & 0.537 & 0.681 & 0.894 & 0.507 & 46.657 & 0.109 \\
\texttt{KermutStruc} & 0.614 & 0.613 & 0.634 & 0.846 & 0.478 & 2.982 & 0.104 \\
\bottomrule
\end{tabular}
}

%% file: table_MLDR-011737_side_by_side.tex
\resizebox{\linewidth}{!}{
\begin{tabular}{lcccccccccccc}
\toprule
 & \multicolumn{6}{c}{$128$ training data points} & \multicolumn{6}{c}{$512$ training data points} \\
 & SpearmanR & PearsonR & MAE & CRPS & NLL & ECE & SpearmanR & PearsonR & MAE & CRPS & NLL & ECE \\
\cmidrule(lr){2-7} \cmidrule(lr){8-13}
\midrule
\texttt{LOCK} & 0.669 & 0.711 & 0.592 & 0.442 & 1.954 & 0.120 & 0.759 & 0.807 & 0.440 & 0.333 & 1.262 & 0.097 \\
\midrule
\texttt{LOCK-45} & 0.668 & 0.710 & 0.593 & 0.444 & 1.957 & 0.125 & 0.757 & 0.805 & 0.442 & 0.334 & 1.270 & 0.095 \\
\texttt{LOCK-62} & 0.668 & 0.710 & 0.594 & 0.443 & 1.952 & 0.121 & 0.758 & 0.806 & 0.441 & 0.334 & 1.258 & 0.098 \\
\texttt{LOCK-80} & 0.668 & 0.710 & 0.596 & 0.445 & 1.963 & 0.125 & 0.757 & 0.805 & 0.444 & 0.336 & 1.271 & 0.101 \\
\midrule
\texttt{LOCK-WeakPrior} & 0.649 & 0.687 & 0.568 & 0.466 & 9.912 & 0.161 & 0.740 & 0.784 & 0.428 & 0.335 & 2.228 & 0.110 \\
\texttt{LOCK-NoLinear} & 0.668 & 0.709 & 0.591 & 0.443 & 2.043 & 0.126 & 0.759 & 0.805 & 0.441 & 0.333 & 1.295 & 0.098 \\
\texttt{LOCK-DoublyLocal} & 0.671 & 0.712 & 0.589 & 0.438 & 1.942 & 0.112 & 0.753 & 0.800 & 0.447 & 0.336 & 1.263 & 0.084 \\
\midrule
\texttt{LOCK-LinearOnly-Global} & 0.609 & 0.617 & 0.868 & 0.668 & 3.074 & 0.137 & 0.713 & 0.714 & 0.822 & 0.646 & 3.962 & 0.150 \\
\texttt{LOCK-LinearOnly-Local} & 0.611 & 0.620 & 0.867 & 0.668 & 3.193 & 0.134 & 0.716 & 0.717 & 0.810 & 0.637 & 4.020 & 0.155 \\
\texttt{LOCK-NonLinearOnly-Global} & 0.659 & 0.682 & 0.663 & 0.497 & 2.158 & 0.113 & 0.749 & 0.786 & 0.507 & 0.385 & 1.632 & 0.100 \\
\texttt{LOCK-NonLinearOnly-Local} & 0.667 & 0.709 & 0.595 & 0.447 & 2.259 & 0.131 & 0.759 & 0.806 & 0.440 & 0.333 & 1.376 & 0.098 \\
\midrule
\texttt{Linear} & 0.586 & 0.596 & 0.916 & 0.700 & 3.092 & 0.144 & 0.698 & 0.698 & 0.854 & 0.667 & 3.974 & 0.149 \\
\texttt{RBF} & 0.631 & 0.668 & 0.671 & 0.512 & 3.123 & 0.193 & 0.646 & 0.683 & 7.484 & 7.363 & $1.5 \times 10^{10}$ & 0.189 \\
\midrule
\texttt{Tanimoto} & 0.632 & 0.654 & 0.762 & 0.573 & 2.119 & 0.120 & 0.739 & 0.769 & 0.574 & 0.424 & 1.552 & 0.087 \\
\texttt{Tanimoto-50} & 0.633 & 0.655 & 0.755 & 0.568 & 2.109 & 0.117 & 0.744 & 0.772 & 0.558 & 0.414 & 1.543 & 0.088 \\
\texttt{Tanimoto-Orig} & 0.630 & 0.652 & 0.764 & 0.575 & 2.115 & 0.122 & 0.736 & 0.767 & 0.582 & 0.431 & 1.587 & 0.088 \\
\midrule
\texttt{Kermut} & 0.639 & 0.654 & 0.845 & 0.657 & 3.867 & 0.159 & 0.750 & 0.767 & 0.670 & 0.498 & 2.297 & 0.108 \\
\texttt{KermutSeq} & 0.518 & 0.505 & 0.836 & 0.635 & 159.803 & 0.138 & 0.605 & 0.614 & 0.714 & 0.540 & 2.982 & 0.122 \\
\texttt{KermutStruc} & 0.624 & 0.641 & 0.847 & 0.665 & 4.243 & 0.159 & 0.706 & 0.713 & 0.847 & 0.670 & 4.808 & 0.167 \\
\bottomrule
\end{tabular}
}

%% file: table_MLDR-011736_side_by_side.tex
\resizebox{\linewidth}{!}{
\begin{tabular}{lcccccccccccc}
\toprule
 & \multicolumn{6}{c}{$48$ training data points} & \multicolumn{6}{c}{$1536$ training data points} \\
 & SpearmanR & PearsonR & MAE & CRPS & NLL & ECE & SpearmanR & PearsonR & MAE & CRPS & NLL & ECE \\
\cmidrule(lr){2-7} \cmidrule(lr){8-13}
\midrule
\texttt{LOCK} & 0.653 & 0.678 & 0.498 & 0.361 & 0.971 & 0.056 & 0.867 & 0.914 & 0.210 & 0.160 & 0.112 & 0.108 \\
\midrule
\texttt{LOCK-45} & 0.653 & 0.678 & 0.499 & 0.361 & 0.973 & 0.055 & 0.866 & 0.913 & 0.211 & 0.161 & 0.117 & 0.108 \\
\texttt{LOCK-62} & 0.652 & 0.676 & 0.500 & 0.362 & 0.975 & 0.055 & 0.867 & 0.914 & 0.210 & 0.160 & 0.110 & 0.108 \\
\texttt{LOCK-80} & 0.652 & 0.677 & 0.500 & 0.362 & 0.974 & 0.055 & 0.867 & 0.914 & 0.210 & 0.160 & 0.113 & 0.109 \\
\midrule
\texttt{LOCK-WeakPrior} & 0.619 & 0.634 & 0.478 & 0.394 & 121.508 & 0.135 & 0.860 & 0.908 & 0.205 & 0.160 & 0.368 & 0.113 \\
\texttt{LOCK-NoLinear} & 0.655 & 0.680 & 0.494 & 0.358 & 0.967 & 0.060 & 0.867 & 0.915 & 0.210 & 0.161 & 0.099 & 0.108 \\
\texttt{LOCK-DoublyLocal} & 0.654 & 0.679 & 0.497 & 0.360 & 0.971 & 0.053 & 0.867 & 0.914 & 0.210 & 0.160 & 0.108 & 0.109 \\
\midrule
\texttt{LOCK-LinearOnly-Global} & 0.554 & 0.563 & 0.586 & 0.414 & 1.116 & 0.047 & 0.809 & 0.806 & 0.415 & 0.298 & 0.749 & 0.040 \\
\texttt{LOCK-LinearOnly-Local} & 0.562 & 0.573 & 0.578 & 0.410 & 1.111 & 0.046 & 0.811 & 0.807 & 0.412 & 0.296 & 0.746 & 0.039 \\
\texttt{LOCK-NonLinearOnly-Global} & 0.623 & 0.636 & 0.548 & 0.389 & 1.054 & 0.059 & 0.855 & 0.902 & 0.254 & 0.189 & 0.240 & 0.076 \\
\texttt{LOCK-NonLinearOnly-Local} & 0.651 & 0.675 & 0.496 & 0.359 & 0.974 & 0.061 & 0.867 & 0.915 & 0.210 & 0.161 & 0.107 & 0.107 \\
\midrule
\texttt{Linear} & 0.506 & 0.521 & 0.603 & 0.424 & 1.142 & 0.054 & 0.802 & 0.800 & 0.420 & 0.302 & 0.762 & 0.040 \\
\texttt{RBF} & 0.610 & 0.602 & 0.561 & 0.421 & $3.7 \times 10^{5}$ & 0.102 & N/A & N/A & N/A & N/A & N/A & N/A \\
\midrule
\texttt{Tanimoto} & 0.514 & 0.511 & 0.592 & 0.415 & 1.106 & 0.075 & 0.846 & 0.888 & 0.272 & 0.201 & 0.281 & 0.068 \\
\texttt{Tanimoto-50} & 0.515 & 0.513 & 0.590 & 0.414 & 1.103 & 0.075 & 0.846 & 0.888 & 0.271 & 0.200 & 0.275 & 0.068 \\
\texttt{Tanimoto-Orig} & 0.510 & 0.507 & 0.593 & 0.415 & 1.106 & 0.074 & 0.845 & 0.887 & 0.273 & 0.202 & 0.282 & 0.067 \\
\midrule
\texttt{Kermut} & 0.636 & 0.638 & 0.549 & 0.398 & 1.415 & 0.048 & 0.850 & 0.888 & 0.285 & 0.212 & 0.358 & 0.066 \\
\texttt{KermutSeq} & 0.540 & 0.537 & 0.619 & 0.443 & 45.840 & 0.053 & 0.794 & 0.838 & 0.338 & 0.248 & 0.476 & 0.061 \\
\texttt{KermutStruc} & 0.615 & 0.618 & 0.569 & 0.413 & 1.752 & 0.050 & 0.809 & 0.805 & 0.417 & 0.300 & 0.763 & 0.040 \\
\bottomrule
\end{tabular}
}

%% file: ts.tex
\begin{figure*}[t]
\centering
\includegraphics[width=0.45\textwidth]{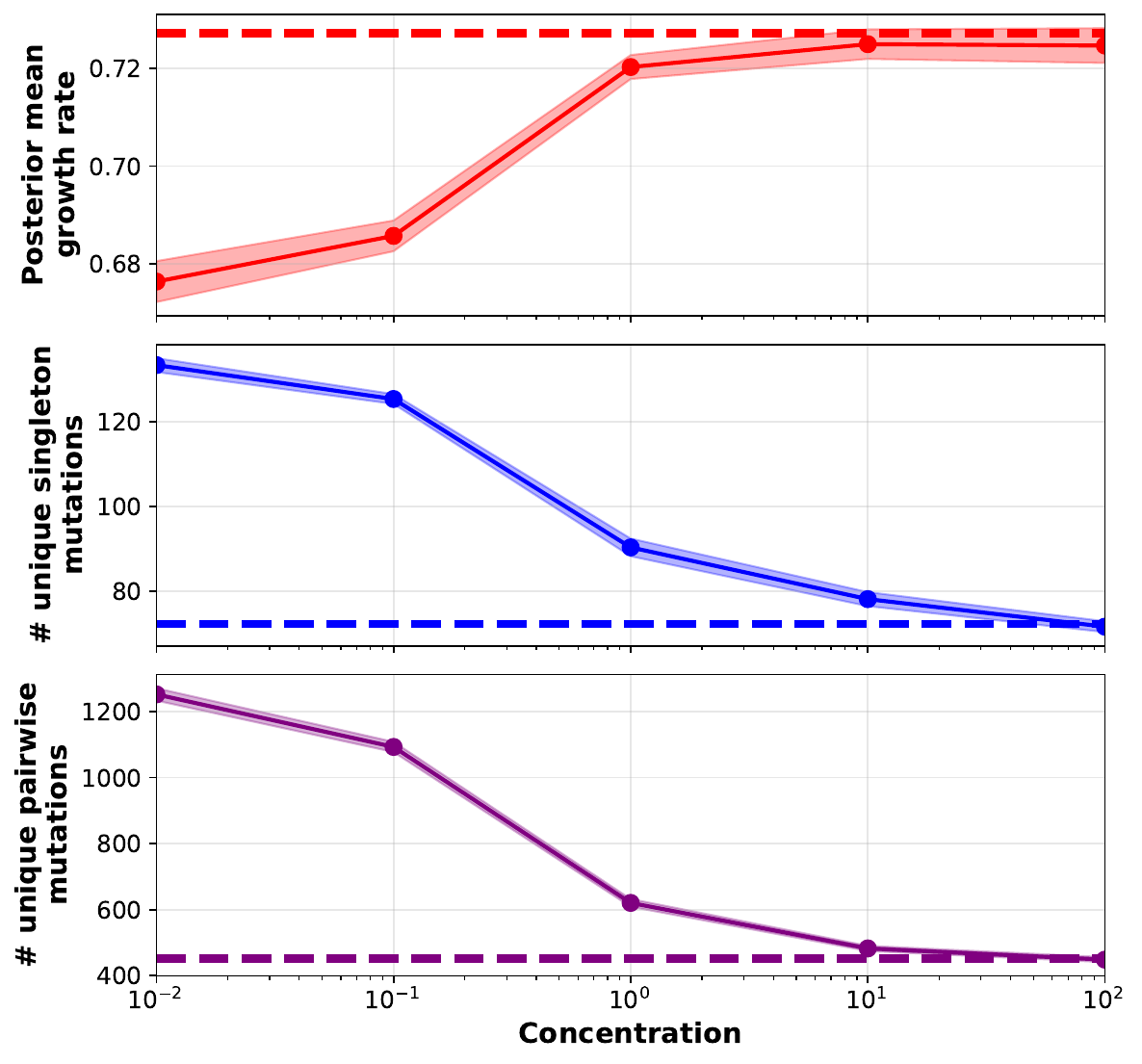}
    \caption{
    We depict results for the Thompson sampling experiment in \secref{sec:app:ts}.
    Solid lines depict properties of the sequences obtained from solving \eqnref{eqn:objectiven} at five distinct concentrations $\alpha$; 
    uncertainty bands reflect standard errors.
    Dashed lines depict properties of the sequences obtained from solving \eqnref{eqn:objective}, which
    corresponds to the $\alpha \rightarrow \infty$ limit.
    We average across $10$ replicate experiments.
    As the concentration $\alpha$ goes to zero and thus the diversity of the Thompson ensemble increases, 
    the sequence diversity---as measured by the total number of unique singleton and 
    pairwise mutations w.r.t.~the wild-type sequence---increases substantially, with e.g.~the
    number of pairwise mutations approximately tripling.
    This increased diversity comes at a relatively small cost in average predicted growth rate.
    }
\label{fig:app:diversity}
\end{figure*}

We consider the ParD3 dataset from \citet{Ding_2024}, which contains sequences of length $L=93$ and
has a variable region of length $L^\prime =10$.
In the following in silico exercise we set out to generate ParD3 antitoxin variants that exhibit improved 
neutralization of the cognate toxin ParE3.
To do so we train a LOCK GP predictor on the experimental growth rates
and use the resulting predictor as a surrogate for neutralization ability in an optimization problem.

Before considering Thompson sampling, we first consider a multi-sequence optimization problem of the form
\begin{equation}
    \label{eqn:objective}
    \max_{\bx_{1:N}} \Phi(\bx_{1:N}) \qquad {\rm with} \qquad \Phi(\bx_{1:N}) \equiv \Psi(\bx_{1:N}) + \sum_{n=1}^N \phi(\bx_n)
\end{equation}
defined on $N=800$ sequences $\bx_{1:N} \equiv \{\bx_1, ..., \bx_N \}$.
We limit the optimization to the $L^\prime$ residues of the variable region (i.e.~all other residues are fixed 
to the wild-type amino acid) and include all $20$ canonical amino acids in the design space. 
Here $\phi$ is a LOCK GP predictor trained on all the data from \citet{Ding_2024}, and higher values of
$\phi(\cdot)$ correspond to higher predicted growth rates.
Furthermore, $\Psi(\bx_{1:N})$ is a diversity-promoting objective that serves the following role:
\begin{enumerate}
\item it heavily penalizes sets of sequences $\bx_{1:N}$ that contain duplicate (i.e.~non-unique) sequences
\item it strongly encourages the sequences $\bx_{1:N}$ to spread out in a balanced manner across a nested set of $8$ Hamming shells centered
    around the wild-type sequence $\bx_{\rm wt}$
\end{enumerate}
In particular $\Psi(\bx_{1:N})$ strongly encourages 
exactly $100$ of the $800$ sequences to reside in the $2$-Hamming shell of $\bx_{\rm wt}$, 
exactly $100$ of the $800$ sequences to reside in the $3$-Hamming shell of $\bx_{\rm wt}$, and so on through
the $9$-Hamming shell of $\bx_{\rm wt}$. 
As such $\Psi(\bx_{1:N})$ encourages mutational diversity in solutions to \eqnref{eqn:objective}.
Nevertheless, since our design space is extremely large,\footnote{The design space for a single sequence consists of 
$20^{10} \approx 10^{13}$ sequences, and we consider a joint design problem defined over $N=800$ sequences.} the diversity effects
of $\Psi(\bx_{1:N})$ are somewhat modest. 
Indeed consider two stylized putative solutions.
In solution A, the $100$ unique sequences placed in the $9$-Hamming shell of $\bx_{\rm wt}$ are very different 
from one another, with typical pairwise Hamming distances of $\sim8-9$. 
In solution B, the $100$ unique sequences placed in the $9$-Hamming shell of $\bx_{\rm wt}$ are quite similar to one
another, with typical pairwise Hamming distances of $\sim2$. 
By assumption both A and B optimally meet the conditions set by $\Psi$ and so the properties of
$\phi(\cdot)$ determine whether solution A or B is preferred. 
In other words, the diversity of the solution to \eqnref{eqn:objective} can vary dramatically as $\phi(\cdot)$ changes.
Consequently, from the point of view of diversity the optimization problem in \eqnref{eqn:objective} is somewhat undetermined.
If ensuring sequence diversity in the solution is important, additional mechanisms for injecting diversity must be considered.

Here we explore how to use (pseudo-)Thompson sampling to inject a controllable amount of diversity into the optimization problem
in  \eqnref{eqn:objective}. To do so we generalize \eqnref{eqn:objective} to leverage
$N=800$ predictors $\{ \phi_n \}$ instead of a single common $\phi$:
\begin{equation}
    \label{eqn:objectiven}
    \max_{\bx_{1:N}} \Phi(\bx_{1:N} | \alpha) \qquad {\rm with} \qquad \Phi(\bx_{1:N} | \alpha) \equiv \Psi(\bx_{1:N}) + \sum_{n=1}^N \phi_n(\bx_n|\alpha)
\end{equation}
Here $\alpha$ is a hyperparameter that controls the diversity of the ensemble $\{ \phi_n \}$.
By construction the diversity of the solution to \eqnref{eqn:objectiven} increases as the diversity of the ensemble increases.
While the diversity in $\{ \phi_n \}$ could come from any number of sources, in the following
each $\phi_n$ is obtained by a (pseudo-)Thompson sampling procedure that we now describe. 

Let $\alpha>0$ be the concentration parameter of a symmetric Dirichlet distribution on the $N-1$ simplex. 
We draw a sample $\bomega \sim {\rm Dir}(\alpha)$ so that each $\omega_k$ has mean $\tfrac{1}{N}$.
We then define $\bomegat \equiv \bomega / {\rm median}(\bomega)$. Thus typically $\omegat_k \approx 1$ and the spread
around $1$ increases as $\alpha \rightarrow 0$. 
We then modify the GP posterior mean prediction given in \eqnref{eqn:gppredmean} as follows
\begin{align}
\label{eqn:gppredmeanmod}
    \mu_\bff(\bx^* | \bomegat) \equiv {k_{* \bX}}^{\rm T}  {(K_{\bX\bX} + \sigma_n^2 {\rm diag}(\bomegat))}^{-1}\bt 
\end{align}
where ${\rm diag}(\bomegat)$ is the $N \times N$ diagonal matrix with entries $\omegat_k$ along the diagonal.
In other words we stochastically modulate the assumed observation noise of each training data point in $(\bX, \bt)$,
while keeping other hyperparameters fixed.
We do this for each predictor $\phi_n$, i.e.~each $\phi_n$ uses the formula in \eqnref{eqn:gppredmeanmod} 
with a distinct sample $\bomegat_n$. 
This procedure results in an ensemble of predictors, each of which modulates the amount of smoothing in different 
amounts of sequence space as controlled by $\bomegat$, but in a way that does not differ dramatically from the 
canonical posterior mean prediction with $\omegat_k \rightarrow 1$.\footnote{An alternative procedure would be
to introduce hyperpriors on the kernel hyperparameters and do MCMC inference over the kernel hyperparameters. One could
then sample hyperparameters from the approximate posterior and define GP posterior mean predictors conditioned on different
hyperparameter samples. While this procedure would be perfectly viable, for simplicity we adopt the noise-modulating
procedure instead. In either case the non-parametric nature of GPs makes exact Thompson sampling computationally
intractable: we simply cannot sample posterior GP function values across the entirety of the sequence space in our 
optimization problem, since it has cardinality $20^{10}$.}

To demonstrate the effect of modulating the concentration $\alpha$, we do an experiment
in which we obtain approximate solutions to \eqnref{eqn:objectiven} using a MCMC-and-annealing-based optimization algorithm
that leverages the gradient-based proposal from \citet{zhang2022langevin}. See \figref{fig:app:diversity} for the results.
As expected, as the concentration goes to zero and the variability in $\bomegat$ increases, the sequence
diversity w.r.t.~the wild-type sequence increases substantially. 
While this comes at a cost in average predicted growth rate $\phi(\cdot)$, this cost is moderate.
For example at $\alpha=0.1$ the average predicted growth rate---as estimated by the canonical GP posterior mean---is reduced
by about $0.04$. This should be compared to the typical standard deviation of predicted growth rates within each of 
the eight Hamming shells, which is about $0.30$. 
In other words the shift of $0.04$ is a small fraction of the variability in the predicted growths rates for the $800$ sequences in $\bx_{1:N}$.

Thus we have demonstrated the viability of using the uncertainty information encoded by LOCK GP to trade-off exploration and exploitation in optimization-based protein design.
The essential point is that Thompson sampling allows us to inject a controllable amount of \emph{model-based} diversity into generation
in a way that reflects data-driven uncertainty about predicted fitness. 
Being able to tunably trade-off exploration and exploitation is a key component of many high-performing Bayesian optimization algorithms, and
so it is attractive to have a predictive model like LOCK GP that can provide high-quality uncertainty estimates---see e.g.~\figref{fig:cvcrps} and Table \ref{tab:metric_summary_gp}---that can be leveraged in design.

%% file: clock_interpretation.tex
\begin{figure*}[ht!]
    \centering
        \includegraphics[width=0.55\textwidth]{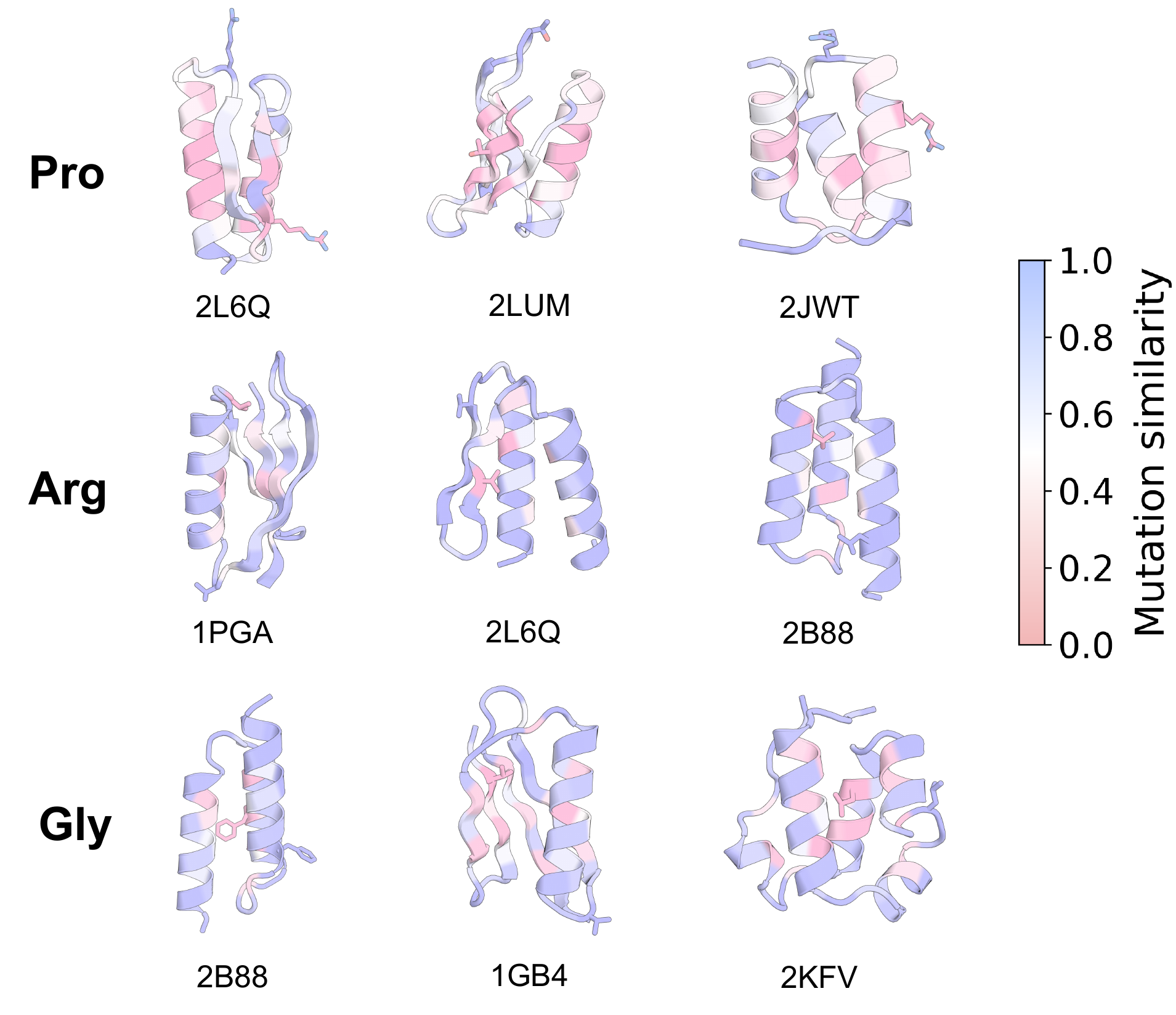}
        \caption{Structure-conditioned amino acid correlations learned by the CLOCK kernel.  
        We show representative structures from \citet{Tsuboyama_2023}, 
        coloring the wild-type residue at each position by its CLOCK correlation
        $C_{\ell a a_{\rm ref}}$ (\eqnref{eqn:corrrep}) to one of three reference amino acids (Pro, Arg, Gly; rows).
        Each structure highlights a pair of sites (namely those with sidechains pictured) with the \emph{same} amino-acid identity but sharply 
        different kernel correlation to the reference amino acid, illustrating how the learned mapping $\bW$ (\eqnref{eqn:Wdefn})
        encodes amino-acid preferences that reflect local structure. 
        Blue indicates substitutions judged similar to the reference residue; red denotes dissimilar, 
        penalized substitutions. 
        The PDB ID associated with each fitness landscape is provided directly under the structure 
        of the corresponding protein. See \secref{app:clock_interpretation} for discussion.
        }
    \label{fig:app:clock_interpretation}
\end{figure*}

See \figref{fig:app:clock_interpretation} for a visual representation of CLOCK correlations.
The top row illustrate that the CLOCK kernel has learned to prefer a substitution
to proline in loops rather than secondary structure elements such as helices and sheets. 
The examples in the second row illustrate that the CLOCK kernel has learned to prefer a substitution to arginine
on the surface of the protein rather than in the interior of the protein or at positions pointing inwards.
The first structure in the final row illustrates that the CLOCK kernel has learned to prefer substitutions
to glycine at the C-terminal region of alpha helices over the interior of alpha helices \citep{richardson1988amino}. 
The second and third structures illustrate the general preference for glycine in loops over beta sheets and alpha helices.

We also note that our structure-conditioned substitution matrices can be seen as an ML-driven descendent of
the `environment-specific substitution tables' (ESSTs) derived in \citet{overington1992environment}---see also \citet{koshi1995context}.
While in our case substitution matrices are regressed on continuous-valued embeddings $\bh_\ell$,
ESSTs are defined for a small number of manually-chosen semantic categories derived from e.g.~secondary structure or solvent accessibility.
The general motivation, however, is similar.

%% file: app_loss.tex
Take the standard GP marginal log likelihood:
\begin{align}
\label{eqn:appmll}
    \log p(\bt|\bX) = \log \NN(\bt |\bm{0}, K_{\bX\bX} + \sigma_n^2 \mathbb{1}_N)
\end{align}
and write 
\begin{align}
K_{\bX\bX} + \sigma_n^2 \mathbb{1}_N = \sigma_f^2 \left(\hat{K}_{\bX\bX} + \hat{\sigma}_n^2 \mathbb{1}_N \right)
\end{align}
i.e.~we factor out the overall kernel scale $\sigma_f^2$ and introduce the inverse signal-to-noise ratio $\hat{\sigma}_n^2 = \sigma_n^2 / \sigma_f^2$.
The MLE estimator for $\sigma_f^2$ can be computed in closed-form:
\begin{align}
    \hat{\sigma}_f^2 = \tfrac{1}{N} \bt^{\rm T}  \left(\hat{K}_{\bX\bX} + \hat{\sigma}_n^2 \mathbb{1}_N \right)^{-1} \bt
\end{align}
If this is plugged into \eqnref{eqn:appmll} and we drop irrelevant constants, we obtain 
\begin{align}
    \log p(\bt|\bX) &= -\tfrac{N}{2} \log \left(2 \pi \hat{\sigma}_f^2 \right) -\tfrac{1}{2}\log \left| \hat{K}_{\bX\bX} + \hat{\sigma}_n^2 \mathbb{1}_N \right| - \tfrac{N}{2} \\
    &\Rightarrow -\tfrac{N}{2} \log \hat{\sigma}_f^2 -\tfrac{1}{2}\log \left| \hat{K}_{\bX\bX} + \hat{\sigma}_n^2 \mathbb{1}_N \right|
\end{align}
in which we have effectively reduced the total number of kernel hyperparameters by one by eliminating $\sigma_f$. 
We use this `concentrated' log likelihood as the training objective for CLOCK. 
This is very convenient in our context, since it obviates the need to learn the overall kernel scale for each landscape. 
Instead we can focus on learning the signal-to-noise ratio. 
Empirically on the data from \citet{Tsuboyama_2023} we find that we can get away with learning a global signal-to-noise ratio that is shared
by all landscapes. 
It's plausible that this ansatz may be less effective for other datasets, but we generically expect this approach to work well in many scenarios,
since we do not need the kernel scale and the noise scale to be perfectly specified in order to extract useful signal from the data. 
In any case learning landscape-specific signal-to-noise ratios is easy, and we find equally good performance when we do so for 
the \citet{Tsuboyama_2023} dataset.

We note that the approach of profiling out $\sigma_f$  has been adopted by many authors in the literature, 
see e.g.~\citet{park2001efficient,moore2016fast,ober2021promises}.
Indeed in \citet{moore2016fast} the authors go a step further and place a prior on $\sigma_f$, in which case the overall kernel scale can 
be integrated out analytically.

%% file: aav.tex
\begin{figure*}[t]
\centering
    \includegraphics[width=0.8\textwidth]{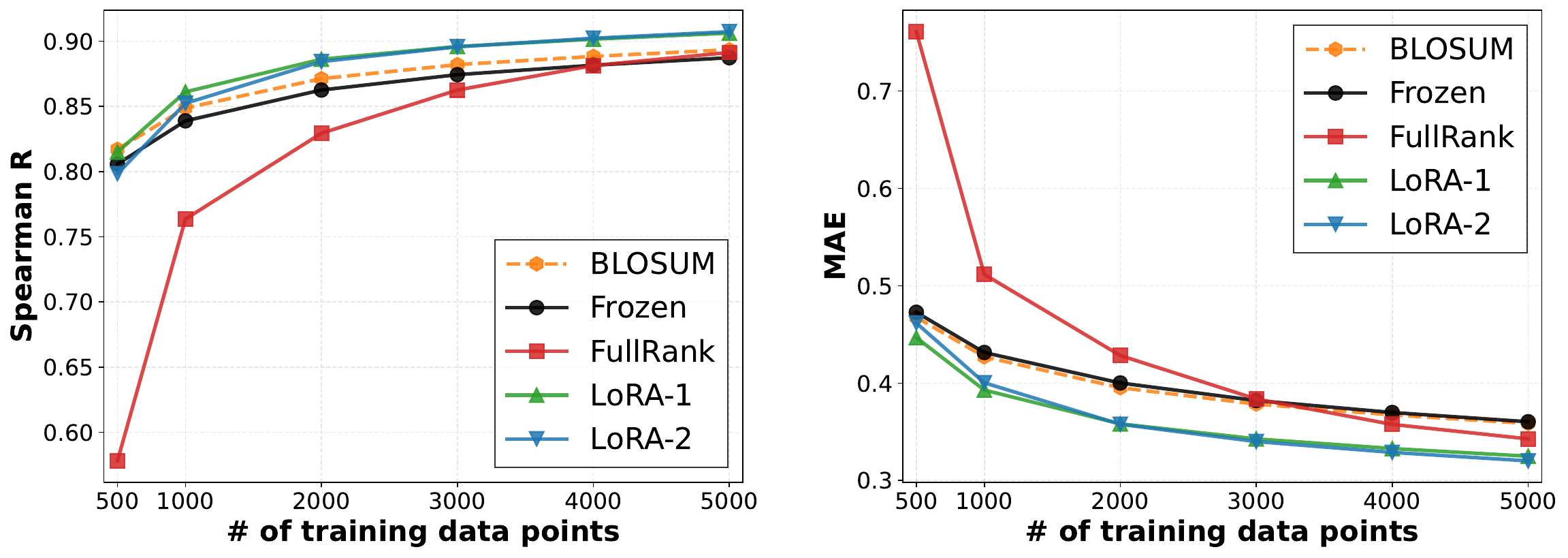}
    \caption{We compare several GPs trained on the AAV dataset from \citet{Sinai_2021}. 
    We plot both Spearman R (left) and MAE (right); metrics are averaged across $10$ i.i.d.~train/test splits.
    All GPs are CLOCK-GPs except for the BLOSUM GP, which uses the same correlation matrix at each position.
    All CLOCK-GPs use a $\bW$ tensor (see \eqnref{eqn:Wdefn}) that is pre-trained on thermostability data (see \secref{sec:multi}).
    All CLOCK-GPs apart from \texttt{Frozen} are then fine-tuned on the AAV data. 
    For $N\gtrsim 1000$ the LoRA-fine-tuned CLOCK-GPs outperform the BLOSUM GP, suggesting performance-enhancing transfer learning from
    the thermostability data.
    }
\label{fig:app:aav_spearman_mae}
\end{figure*}

\begin{figure*}[t]
\centering
    \includegraphics[width=0.8\textwidth]{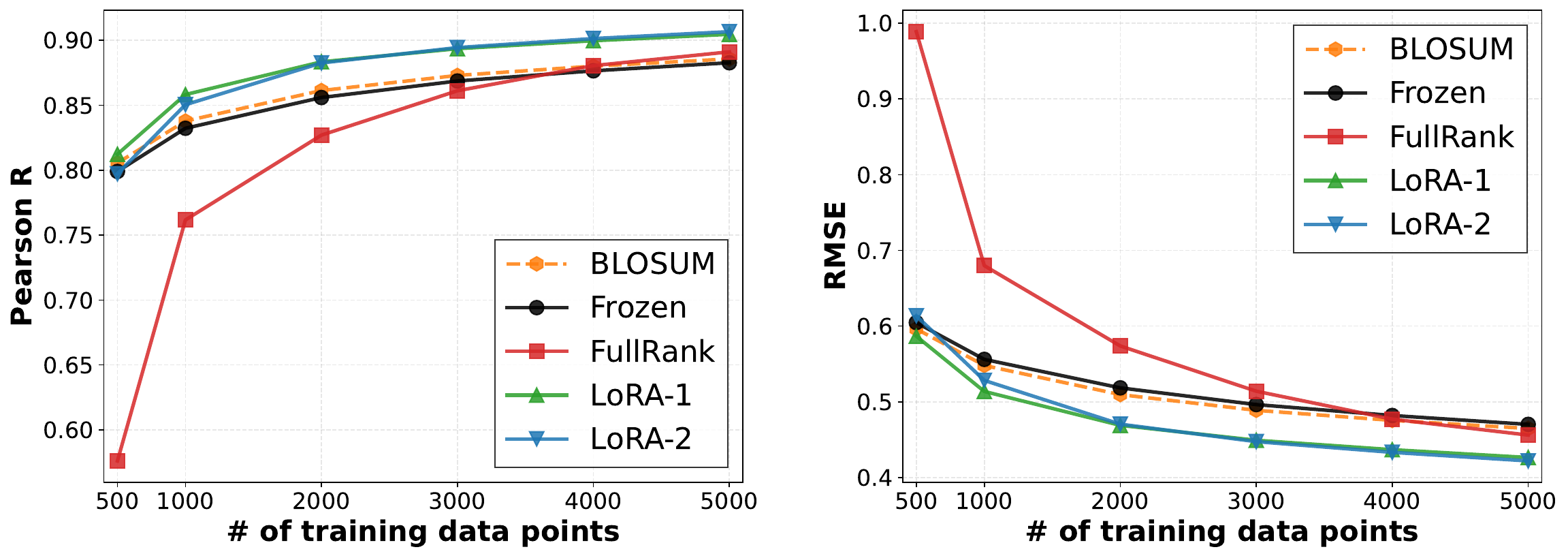}
    \caption{We compare several GPs trained on the AAV dataset from \citet{Sinai_2021}. 
    We plot both Pearson R (left) and RMSE (right); metrics are averaged across $10$ i.i.d.~train/test splits.
    All GPs are CLOCK-GPs except for the BLOSUM GP, which uses the same correlation matrix at each position.
    All CLOCK-GPs use a $\bW$ tensor (see \eqnref{eqn:Wdefn}) that is pre-trained on thermostability data (see \secref{sec:multi}).
    All CLOCK-GPs apart from \texttt{Frozen} are then fine-tuned on the AAV data. 
    For $N\gtrsim 1000$ the LoRA-fine-tuned CLOCK-GPs outperform the BLOSUM GP, suggesting performance-enhancing transfer learning from
    the thermostability data.
    }
    \label{fig:app:aav_pearson_rmse}
\end{figure*}

We investigate the basic viability of fine-tuning the $\bW$ tensor in CLOCK (see \eqnref{eqn:Wdefn}) and adapting it to new landscapes. 
We would expect the most promise for such an approach in one of two scenarios: either i) the new landscape is for a property
that is closely related to the property that $\bW$ is pre-trained on; or ii) we can jointly fine-tune $\bW$ on a considerable 
number of closely related landscapes. 
For simplicity, here we choose to focus on an a priori more difficult scenario in which the property
in the new landscape is not particularly closely related to the property used in pre-training and in which we only have a single 
new landscape. 
In particular we focus on the AAV dataset from \citet{Sinai_2021}, where the property is a measure of capsid viability.
We note that the length of the protein in this dataset is $L=735$, whereas for the thermostability data $L \sim 50$.

To fine-tune $\bW$ we simply optimize the vanilla GP training objective, i.e.~the marginal log likelihood (see \eqnref{eqn:mll}). 
We use the Adam optimizer with an initial learning rate of $0.01$. We train for $1500$ steps and decimate the learning rate after
$500$ and $1000$ steps.
For simplicity in all cases we consider GPs with a non-linear kernel $\kLnl$ (see \eqnref{eqn:nlkernel}) together with a learned global exponent
(in addition to the noise scale $\sigma_n$ and the overall kernel scale). 
We consider $5$ different methods.
In \texttt{BLOSUM} we utilize a BLOSUM50 correlation matrix at all positions.
In \texttt{Frozen} we use $\bW$ obtained from the thermostability pre-training and leave $\bW$ fixed during training. 
In \texttt{FullRank} we initialize with $\bW$ obtained from the thermostability pre-training and train $\bW$ jointly with the kernel hyperparameters. 
In \texttt{LoRA-k} we use $\bW$ obtained from the thermostability pre-training and train a rank-$k$ LoRA adapter \cite{hu2022lora} jointly with the kernel hyperparameters (so that $\bW$ remains fixed).\footnote{While $\bW$ is most naturally viewed as a tensor of shape $A \times A - 1 \times D_{\rm emb}$, here we conceptualize it as a two-dimensional tensor of shape $A(A - 1) \times D_{\rm emb}$ when applying LoRA. Of course other low-rank parameterizations are also possible.}
These parameter-efficient adapters have $548$ (for $k=1$) and $1096$ (for $k=2$) trainable parameters, much reducing the risk of overfitting.
For the results see \figref{fig:app:aav_spearman_mae}-\ref{fig:app:aav_pearson_rmse}.

As we might expect, \texttt{FullRank} performs poorly unless there is sufficient training data from the AAV landscape, 
since it is easy to overfit the $\sim\!54$k parameters in $\bW$.
\texttt{Frozen} performs quite well---and comparably to \texttt{BLOSUM}---which is perhaps surprising, since thermostability is a priori
quite different from capsid viability.
For $N\gtrsim 1000$ the strongest performers are \texttt{LoRA-1} and \texttt{LoRA-2}, 
suggesting performance-enhancing transfer learning from the thermostability data. 
While the performance improvement is relatively small, we emphasize that good performance on \texttt{LoRA-k} requires mapping previously
unseen structural contexts $\bh_{1:L}$ to correlation matrices that are adapted to the AAV landscape.
For example this presumably requires `unlearning' the idiosyncratic role played by proline in the thermostability data, 
see \figref{fig:app:local-clock-corr}.
Thus these empirical results demonstrate the basic viability of fine-tuning $\bW$ to adapt to new landscapes.

%% file: genclock.tex
We highlight some of the general features of the CLOCK kernel introduced in \secref{sec:clock} and describe how it can be readily generalized
to other kernel constructions.
In more detail CLOCK can be understood as a special case of the following more general framework.
For each landscape and corresponding reference structure $\SSS$ we either compute structure embeddings $\bh_{1:L}(\SSS)$\footnote{Alternatively, we compute structure-agnostic embeddings $\bh_{1:L}(\bxref)$. The essential point is that we get a single set of embeddings $\bh_{1:L}$ that is shared across the entire landscape.} or
sequence-and-structure embeddings $\bh_{1:L}(\SSS, \bx)$. 
It is important to emphasize that doing inference with models that rely on structure embeddings will generally be much faster than those that
rely on sequence-and-structure embeddings, since structure embeddings are shared across the entire landscape.
We treat each of these two possibilities in turn.

\paragraph{Structure Embeddings}
We use a sequence-to-sequence model to map $\bh_{1:L}$ to transformed embeddings $\bht_{1:L}$ with some mapping $\TT$.
We then formulate a kernel that aggregates information from residue-level embeddings $\bht_{1:L}$ in a way that naturally supports variable length sequences. 
In particular we can leverage kernel addition or kernel multiplication or some combination of the two:
\begin{itemize}
\item Each embedding $\bht_\ell$ is mapped to a residue-specific subkernel $k_\ell$---i.e.~an $A \times A$ PSD matrix---and the full kernel is given by addition, i.e.~$k = \sum_\ell k_\ell$
\item Each embedding $\bht_\ell$ is mapped to a residue-specific subkernel $k_\ell$---i.e.~an $A \times A$ PSD matrix---and the full kernel is given by multiplication, i.e.~$k = \prod_\ell k_\ell$
\end{itemize}

\paragraph{Sequence-and-structure Embeddings}

We use a sequence-to-sequence model to map $\bh_{1:L}$ to transformed embeddings $\bht_{1:L}$ with some mapping $\TT$.\footnote{In the case of 
sequence-and-structure embeddings we note that $\TT$ can also depend explicitly on $\bx$.}
We treat $\bht_{1:L}$ as features and feed them into any kernel that takes euclidean features as input,\footnote{In more
detail suppose we have features $\bz_{\ell} \in \RR^D$ where $\ell$ ranges across positions and where each positional feature
is of dimension $D$. We can flatten $\{ \bz_{\ell} \}$ to get a vector $\bZ$ of length $LD$ and provide $\bZ$ as an input to a, say, 
Tanimoto kernel \cite{tripp2023tanimoto}.
Alternatively, we can define $L$ Tanimoto kernels $k_\ell$, where each subkernel takes features $\bz_{\ell}$ as input; the complete kernel
is then formed via kernel addition or multiplication.}
e.g.~an RBF or Tanimoto kernel.\footnote{Indeed we have experimented with both of these constructions on the thermostability data from \citet{Tsuboyama_2023}. 
Both kernels resulted in good predictive models, but we found that CLOCK performed better at larger training sizes.}
That is in the case of sequence-and-structure embeddings the essential aggregation step happens at the feature level instead of at the kernel level.

\paragraph{Kernel Learning}

In both cases to learn the mapping $\TT$ as well as any other parameters that feed into the kernel construction, we use gradient-based methods to optimize
the concentrated form of the standard GP training objective in which the overall kernel scale is `profiled' out. 
See \secref{app:clock} and \secref{app:sec:multitaskdetails}.

\paragraph{Discussion}

CLOCK is a special case of this broader framework in which the sequence-to-sequence model $\TT$ is given by applying the $\bW$ matrix
in \eqnref{eqn:Wdefn} residue-by-residue to obtain correlation matrices $\bC_{1:L}$ using \eqnref{eqn:corrrep}.
These correlation matrices are then used to define both linear and non-linear subkernels as in \eqnref{eqn:linkernel} and \eqnref{eqn:nlkernel}.
These subkernels can then be combined as in LOCK, i.e.~as in \eqnref{eqn:lockdefn}, which utilizes both kernel addition and kernel multiplication
to achieve local linearity.\footnote{An interesting feature of CLOCK is that we learn linear kernels at training time, but at inference time we
can deploy both linear and non-linear kernels. This is one of the advantages of unifying around correlation matrices, which can be readily deployed
in various ways. Since linear kernel learning is straightforward, this flexibility contributes to the speed and robustness of CLOCK training.
See \secref{app:sec:multitaskdetails}.}
Alternatively we can discard the linear kernels and use only the non-linear kernels or vice versa; indeed we explicitly benchmark against
these constructions in \secref{app:ablation} in the non-structure-conditioned case. 

This general framework makes it evident that there is considerable room for defining additional CLOCK-like variants. For example,
instead of using a mapping $\TT$ that is applied independently at each residue, we could use a true sequence-to-sequence model like
a RNN or Transformer.

%% file: app_scalability.tex
\begin{figure}[h!]
\centering
\includegraphics[width=0.55\columnwidth]{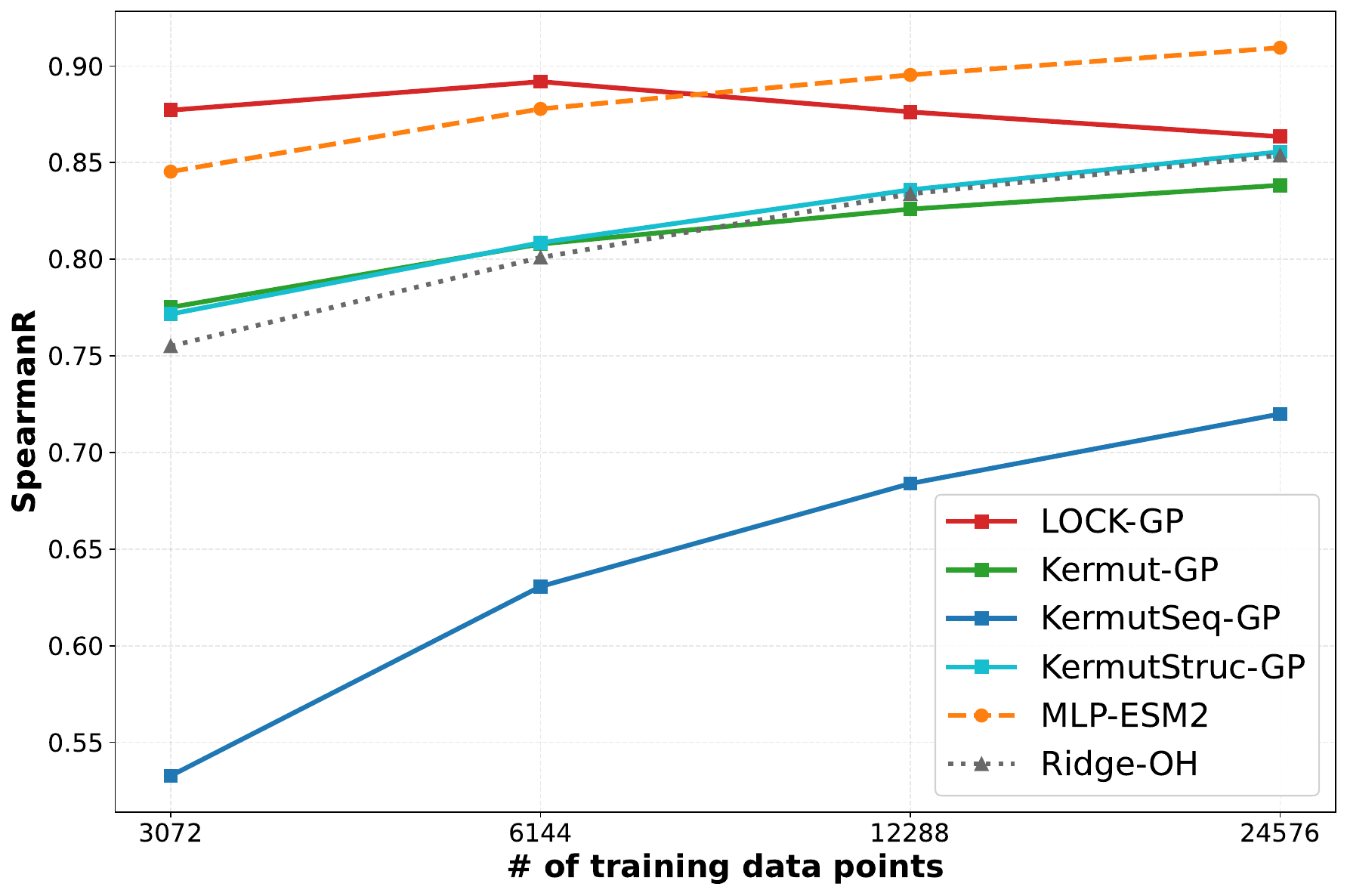}
    \caption{We report Spearman R for models trained on the amacGFP dataset from \citet{Gonzalez_Somermeyer_2022}
    as a function of the number of training data points. Spearman R is averaged across three replicates.
}
    \label{fig:cgscaling}
\end{figure}

Apart from the binary classification experiment in \secref{sec:binary}, 
the GP experiments in \secref{sec:exp} focus
on the regression setting with small-to-moderate sized datasets, a regime in which exact inference is viable.
In this section we briefly touch upon some of the possibilities for scaling LOCK GPs to much larger datasets.

For regression, using the conjugate gradient (CG) method as in \citet{wang2019exact} formally reduces training 
from $\OO(N^3)$ to $\OO(N^2)$ and can result in large speed-ups in practice, making $N \sim 10^5-10^6$ viable on commercial GPUs. 
For the LOCK kernel in particular, this approach can be made significantly more memory 
efficient by utilizing kernel-vector multiplies that are implemented 
in a streaming fashion.\footnote{See \texttt{https://github.com/getkeops/keops}.}

Beyond CG, there are many powerful approximate inference methods that leverage so-called inducing points \cite{hensman2013gaussian}.
Many of these methods are mini-batch-friendly so that they can be pushed to arbitrarily large $N$ in principle.
We provided an example of this kind of approach in the binary classification experiment in \secref{sec:binary} and \figref{fig:binary},
which considered up to $N=12288$ training data points and which uses the inference algorithm from \citet{wenzel2019efficient}. 
For the regression setting we suggest that the method from \citet{cazelles2026efficient} is particularly well-suited for LOCK GPs, 
since it avoids the introduction of inducing points.\footnote{These methods work by utilizing $M$ (with $M \ll N$) inducing 
points that provide a `compressed' view of the $N$ training inputs. As such they can be a poor fit for discrete input spaces. 
This is for at least two reasons: i) the discrete nature of the input space makes it challenging to optimize the inducing
points; and ii) the compression might be excessively lossy (consider e.g.~rare mutations).}
While we do not report any results with this method here, we find that it can work well in practice.

As a demonstration of the scalability of LOCK GPs in the regression setting, 
we report the results of an additional experiment that mirrors
the binary classification experiment in \secref{sec:binary}.
In particular we use CG to train four GPs (and two additional baselines) on the amacGFP dataset from \citet{Gonzalez_Somermeyer_2022} on
up to $24576$ training data points, see \figref{fig:cgscaling}.
This is done using a NVIDIA A10G GPU, which has $24$ GB of memory, i.e.~only 12.5\% of the memory available on a modern NVIDIA B200 GPU.

Finally, we emphasize that these general considerations also translate to CLOCK.
In more detail, we note that in CLOCK (see \secref{app:sec:multitaskdetails} for training details) we randomly subsample within landscapes during training of the matrix $\bW$ (see \eqnref{eqn:Wdefn}) so that each landscape has $512$ training data points. 
As such during training of $\bW$, the cubic cost of exact GP inference is not an issue. 
When deploying CLOCK on a new (potentially large) landscape, the various approaches described above---for example CG---can be used to fit any kernel hyperparameters such as the overall kernel scale to the landscape at hand. 
As such CLOCK can be readily deployed in the large $N$ setting.

%% file: submat.tex
See Table~\ref{tab:substitution-matrices}.

%
%
%
\begin{table}[h]
\centering
\scriptsize
\begin{tabular}{lc@{\qquad}lc@{\qquad}lc}
\toprule
\textbf{Substitution Matrix} & \textbf{ID} & \textbf{Substitution Matrix} & \textbf{ID} & \textbf{Substitution Matrix} & \textbf{ID} \\
\midrule
BLOSUM100                      & $\checkmark$ & BLOSUM85                       & $\checkmark$ & PAM180                         &  \\
BLOSUM30                       &  & BLOSUM90                       & $\checkmark$ & PAM190                         &  \\
BLOSUM35                       &  & BLOSUMN                        & $\checkmark$ & PAM20                          &  \\
BLOSUM40                       & $\checkmark$ & CorBLOSUM49\_5.0               & $\checkmark$ & PAM200                         &  \\
BLOSUM45                       & $\checkmark$ & CorBLOSUM57\_13p               & $\checkmark$ & PAM210                         &  \\
BLOSUM50                       & $\checkmark$ & CorBLOSUM57\_14.3              & $\checkmark$ & PAM250                         &  \\
BLOSUM50\_13p                  & $\checkmark$ & CorBLOSUM61\_5.0               & $\checkmark$ & PAM30                          &  \\
BLOSUM50\_14.3                 & $\checkmark$ & CorBLOSUM66\_13p               & $\checkmark$ & PAM40                          &  \\
BLOSUM50\_5.0                  & $\checkmark$ & CorBLOSUM67\_14.3              & $\checkmark$ & PAM50                          &  \\
BLOSUM55                       & $\checkmark$ & DAYHOFF                        &  & PAM60                          &  \\
BLOSUM60                       & $\checkmark$ & PAM10                          &  & PAM70                          &  \\
BLOSUM62                       & $\checkmark$ & PAM100                         &  & PAM80                          &  \\
BLOSUM62\_13p                  & $\checkmark$ & PAM110                         &  & PAM90                          &  \\
BLOSUM62\_14.3                 & $\checkmark$ & PAM120                         &  & RBLOSUM52\_5.0                 & $\checkmark$ \\
BLOSUM62\_5.0                  & $\checkmark$ & PAM130                         &  & RBLOSUM59\_13p                 & $\checkmark$ \\
BLOSUM65                       & $\checkmark$ & PAM140                         &  & RBLOSUM59\_14.3                & $\checkmark$ \\
BLOSUM70                       & $\checkmark$ & PAM150                         &  & RBLOSUM64\_5.0                 & $\checkmark$ \\
BLOSUM75                       & $\checkmark$ & PAM160                         &  & RBLOSUM69\_13p                 & $\checkmark$ \\
BLOSUM80                       & $\checkmark$ & PAM170                         &  & RBLOSUM69\_14.3                & $\checkmark$ \\
\bottomrule
\end{tabular}
    \caption{We provide an inventory of $57$ substitution matrices available from Biotite \cite{Kunzmann2023}, 
    focusing on the $20 \times 20$ principle submatrices that correspond to the $20$ canonical amino acids. 
    We list only those matrices that are PSD. 
    For each matrix we indicate whether it is infinitely divisible (ID).
    We note that all listed BLOSUM matrices are ID except for BLOSUM30 and BLOSUM35.}
\label{tab:substitution-matrices}
\end{table}

%% file: kermut.tex
The Kermut kernel has two components, a sequence kernel $k_{\rm seq}(\bx, \by)$ and a structure kernel $k_{\rm struct}(\bx, \by)$.
The former is an isotropic RBF kernel defined on mean-pooled ESM-2 embeddings.
Let $M(\bx) \subseteq \{1,\dots,L\}$ denote the mutated positions of sequence $\bx$ w.r.t.~the user-specified reference (e.g.~wild-type)
sequence $\bxref$. 
Then the structure kernel is a set kernel defined over pairs of mutations:
\begin{equation}
    k_{\rm struct}(\bx, \by) = \sum_{i\in M(\bx)}\sum_{j\in M(\by)} k_{\rm struct}^\prime(\bx_i, \by_j)
\end{equation}
where $k_{\rm struct}^\prime(\bx_i, \by_j)$ is a base kernel that acts on pairs of residues. 
In Kermut $k_{\rm struct}^\prime(\bx_i, \by_j)$ is a product kernel made up of three subkernels and leverages ProteinMPNN inverse
folding probabilities and 3d structure coordinates. 
In addition Kermut GP utilizes an ESM-2-derived zero-shot mean function.
For more details see \citet{groth2024kermut}.

Here we would like to draw attention to some of the counterintuitive properties of $k_{\rm struct}(\bx, \by)$:
\begin{itemize}
\item \textbf{Combinatorial explosion:} $k_{\rm struct}$ scales quadratically as the number of mutations increases, which tends to inflate similarities for high-order mutants.
\item \textbf{Variable behavior as we go further away from the reference sequence:} 
Choose a sequence $\by$ that is everywhere mutated away from the reference
sequence, i.e.~the Hamming distance is given by $d_{\rm H}(\bxref, \by) = L$.  
Now let $\{\bx^{(0)}\equiv\bxref, \bx^{(1)}, ..., \bx^{(L)}\equiv\by \}$ be a series
of sequences that interpolate from $\bxref$ to $\by$ with the property that $d_{\rm H}(\bxref, \bx^{(\ell)}) = \ell$.
In other words $\bx^{(\ell)}$ is like $\bxref$ except that it has adopted $\ell$ mutations from $\by$.
Then under the assumption that $k_{\rm struct}^\prime(\bx_i^{(\ell)}, \bx_j^{(\ell^\prime)})$ is always strictly greater than zero---which 
is the case for Kermut---we find that
\begin{equation}
k_{\rm struct}(\bx^{(0)}, \bx^{(1)}) < k_{\rm struct}(\bx^{(1)}, \bx^{(2)}) < \dots < k_{\rm struct}(\bx^{(L-1)}, \bx^{(L)})
\end{equation}
since the sum that defines $k_{\rm struct}(\bx^{(\ell)}, \bx^{(\ell+1)})$ contains all the terms that enter into 
$k_{\rm struct}(\bx^{(\ell-1)}, \bx^{(\ell)})$ plus new (positive) terms. 
This means, for example, that $k_{\rm struct}(\bx^{(L-1)}, \bx^{(L)})$ is generally going to be \emph{much} larger
than $k_{\rm struct}(\bx^{(0)}, \bx^{(1)})$, even though both similarities are between pairs of sequences that differ
by a single mutation. 
This highly non-intuitive behavior would seem to suggest that, while Kermut may perform well for datasets where most sequences are in the vicinity
of the reference sequence, it may exhibit unusual behavior for datasets with lots of higher-order mutants.\footnote{As such this behavior is essentially a special case of the `combinatorial explosion' mentioned above.}
Indeed this may be a contributing factor to the relatively poor performance of Kermut in the extrapolation setting, see e.g.~Table \ref{tab:metric_summary}.
\item \textbf{Wild-type degeneracy:} $k_{\rm struct}(\bxref, \by) = 0$ for all sequences $\by$. 
One consequence is that---in the absence of the sequence kernel---the posterior mean prediction at $\bxref$ is given by the prior mean (cf.~\eqnref{eqn:gppredmean}), 
i.e.~is largely uninformed by other sequences.\footnote{Other sequences enter only insofar as they influence learned 
parameters in the prior mean. In this case there are two such parameters, a shift and a scale, which are used to adjust the zero-shot 
pseudo-likelihood from ESM-2.} 
Moreover posterior (function) uncertainty {\bf collapses to 0} at $\bxref$ (see \eqnref{eqn:gppredvar}).
By contrast the LOCK-GP functions as a \emph{smoother}, just like GPs with familiar kernels like the RBF kernel.\footnote{Here we assume noisy observations like in \eqnref{eqn:jointyf}.}
That is predictions for an observed sequence $\bx$ are influenced not only by the observed value $t$ at $\bx$ but also by the observed function values of nearby sequences (as determined by the kernel), and the function uncertainty is generically non-zero everywhere.
\end{itemize}

\paragraph{Implementation}

We use our own implementation of Kermut that is lightly adapted from the authors' code at \texttt{https://github.com/petergroth/kermut}.
Since Kermut is one of our strongest baseline methods, we take special care to verify that our implementation is equivalent
to the authors'. 
To do so we perform a head-to-head comparison of the two implementations, comparing Spearman R metrics obtained with our implementation
against metrics provided by the authors. 
We do so on $8$ ProteinGym datasets, using the exact same train/test splits and metric computations.
See \figref{fig:app:kermut_comp} for the results. 
We find very good agreement, with the largest difference appearing for the dataset with only $33$ data points, where we
find more variability from training run to training run.
In addition we checked our computation of ESM-2 embeddings as well as masked marginal zero-shot pseudolikelihoods
against those provided by the authors and found excellent agreement.

\begin{figure*}[t]
\centering
\includegraphics[width=\textwidth]{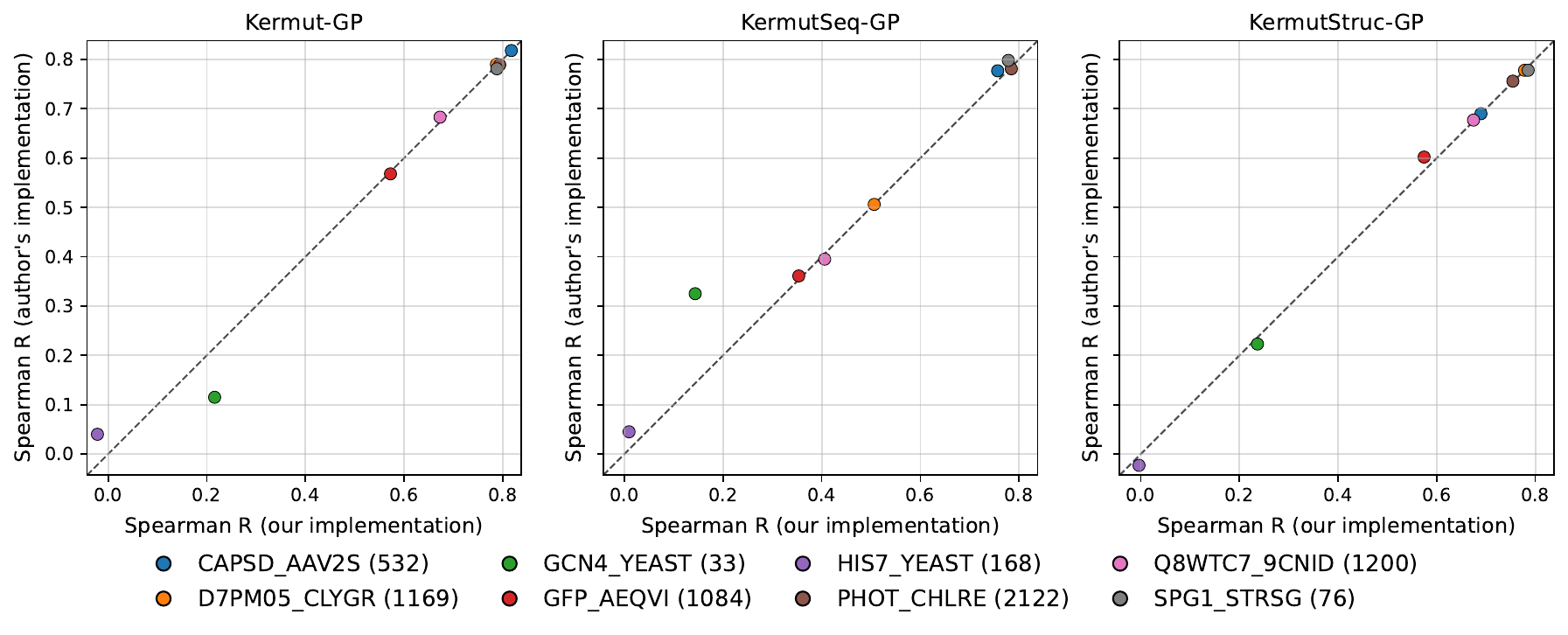}
    \caption{
    We depict Spearman R metrics obtained for $8$ ProteinGym datasets for all three Kermut variants,
    comparing our implementation to the results provided by \citet{groth2024kermut}.
    The number in parentheses after each landscape indicates the total number of data points in the given landscape.
    }
\label{fig:app:kermut_comp}
\end{figure*}

\paragraph{Discussion of Kermut Log Likelihoods} 

We briefly discuss the Kermut uncertainty results summarized in Table~\ref{tab:landscape_nll_combined} and Table~\ref{tab:metric_summary_gp}.
If we define a ``catastrophic NLL'' as one larger than $100$, we find that \texttt{LOCK-GP} and \texttt{Tanimoto-GP} do not exhibit
catastrophic NLLs for any datasets, while all three Kermut variants demonstrate catastrophic NLLs on at least one dataset.
Moreover, we find examples of this in all three evaluation regimes: cross-validation, unseen mutations, and extrapolation.
These catastrophic NLLs result from very narrow posterior predictive distributions, which appears to be driven by several distinct phenomena. 
In some cases, the lengthscale of the RBF kernel collapses; in other cases, the observation noise $\sigma_n$ becomes anomalously small. 
It is possible that this undesirable behavior could be rescued by modifying the hyperpriors and/or hyperparameter initialization used by Kermut.
However, as we can see from Table~\ref{tab:landscape_nll_combined}, the poor average NLL for the three Kermut variants is not solely due to a few outliers.
Indeed \texttt{LOCK-GP} has the lowest NLL on $17$/$21$ landscapes in the cross-validation regime with $192$ training data points.
It also has the lowest NLL on $12$/$21$ landscapes in the extrapolation regime with $128$ training data points
as well as the lowest NLL on $13$/$21$ landscapes in the extrapolation regime with $96$ training data points.

%% file: data.tex
See Table~\ref{tab:datasets} and Table~\ref{tab:datasets_numeric} for details on each dataset, including summary statistics like sequence length.

\input{dataset_table.tex}
\input{dataset_table_numeric.tex}

\paragraph{Dataset Pre-processing}

We minimally processed each ProteinGym dataset \cite{notin2023proteingym}, downloaded as DMS substitution datasets from the ProteinGym website directly in August 2025. 
For each dataset, we removed residues from either end of the associated predicted structure or reference sequence that were not present in the 
other to ensure their sequence lengths matched. 
We additionally log10-normalized DMS scores for the three Somermeyer et al.~datasets (\texttt{D7PM05\_CLYGR\_Somermeyer\_2022}, \texttt{Q6WV12\_9MAXI\_Somermeyer\_2022}, and \texttt{Q8WTC7\_9CNID\_Somermeyer\_2022}; \cite{Gonzalez_Somermeyer_2022}). 
For all other ProteinGym datasets, we performed no pre-processing on DMS scores.

We processed non-ProteinGym datasets separately. 
For the three Phillips et al.~datasets (\texttt{CH65 - MA90}, \texttt{CH65 - SI06}, and \texttt{CH65 - G189E}), 
we associated reported negative log KD scores with the reference 5UGY PDB structure \cite{Phillips_2023}. 
For the three Moulana et al.~datasets (\texttt{Moulana Regn10987}, \texttt{Moulana CB6}, and \texttt{Moulana Cov555}), 
we associated reported log KD scores with their respective reference PDB structure - 6XDG for Regen10987, 7C01 for CB6, and 7KMG for COV555 \cite{Moulana_2023}. 
To define full-length reference sequences for these datasets, we added constant residues to each set of originally-reported wildtype residues 
based on the constant residues in each respective structure's sequence. 
For the two Tsuboyama et al.~datasets \cite{Tsuboyama_2023}, we associated each dataset with its respective AlphaFold-predicted structure after 
taking the mean stability score across duplicated sequences.

Lastly, for the four Shanker et al.~datasets (\texttt{CR9114 - H1}, \texttt{CR9114 - H3}, \texttt{CR6261 - H1}, \texttt{CR6261 - H9}), 
we associated each dataset with their respective PDB structure - 3GBN for CR6261 datasets and 4FQI for CR9114 datasets \cite{Shanker_2024}. 
While these are experimental structures, it is important to note that the measured experimental values for these datasets do not precisely match the 
subtype of the complexed influenza virus measured in the structure for all experiments---please refer to \citet{Shanker_2024} for more information. 
For the reference 3GBN structure, we also removed the first residue of the sequence as it is unstructured. 
As above, we then added constant residues to each set of originally-reported wildtype residues 
based on the constant residues in each respective structure's sequence to define full-length reference sequences. 
For the CR9114 and CR6261 datasets, respectively, we also removed sequences with any missing score in the associated pair of H1/H3 or H1/H9 scores.

\paragraph{Cross-validation}

We do $7$-fold cross-validation. That is for each landscape we randomly partition the data into $7$ disjoint and approximately-equally-sized subsets.
We then define $7$ train/test sets, where each test set is one of the $7$ subsets and the training set is the remaining data points.
This procedure is done once for each landscape. Then, to modulate the number of training data points, we do i.i.d.~subsampling within
each training set, while simultaneously keeping the test set fixed. To compute metrics we aggregate test predictions across all $7$ test sets, making
sure that predictions are computed in the same target space. 
We then compute predictive metrics like Spearman R or MAE on these aggregated predictions.
In this manner we obtain a single metric for each landscape and for each given number of training data points. 
For ease of comparison of metrics across landscapes, we use the standard deviation
of all target scores in the given landscape to normalize true and predicted scores as well as the predictive distribution when available. 
This additional normalization step has no effect on metrics like Spearman or Pearson R, but it ensures that scale-dependent metrics like MAE and CRPS are
more readily comparable across landscapes. 
For $N=48$ training data points we do three subsampling replicates to mitigate against random variability that is more prevalent for small training data sets;\footnote{For the GP ablation experiments in \secref{app:ablation} we do a single replicate.} for $N > 48$ we do a single replicate.

\paragraph{Unseen Mutations}

In order to investigate the ability of different models to generalize to unseen mutations, we devise train/test splits in which all test sequences
have at least one mutation that is not present in the corresponding training set.
Given the diverse nature of our $21$ datasets, we utilize two distinct strategies to do so.

The first strategy is used for landscapes that have $20$ or fewer variable positions and for which the majority amino acid
at each position has a frequency of at least $50\%$. These happen to be the $10$ antibody landscapes: i) $3$ CH65 landscapes;
ii) $4$ CR landscapes; and iii) $3$ Moulana landscapes. 
For these landscapes in each distinct train/test split we randomly select $4$ distinct residues.
The initial training set consists of all sequences that have the majority amino acid at the selected residues.
The initial test set consists of all remaining sequences.
We then subsample each training set down to exactly $96$ data points
and remove all sequences in the test set that do not have any unseen mutations, i.e.~mutations not present in the training set. 
We define $3$ distinct train/test splits for each landscape, 
ensuring that there is minimal overlap between the `selected' residues between the three splits for each landscape.
We also require that the standard deviation of the training targets is at least a tenth of the standard deviation of the targets across the entire landscape.

For the remaining landscapes we utilize a different strategy.
First we identify ``rare'' mutations by ranking mutations by their frequency and choosing those in the bottom decile. 
We then construct a training set with $96$ data points by drawing sequences from a weighted distribution that heavily favors
sequences with one or more rare mutations. 
The test set is then defined as all remaining sequences that have at least one unseen mutation, i.e.~mutations not present in the training set.
As above we define $3$ distinct train/test splits for each landscape and require that  the standard deviation of the training targets is at least a tenth of the standard deviation of the targets across the entire landscape.

With these choices each of the $21 \times 3$ test sets has between $1684$ and $175,693$ sequences.
Metrics are computed for each train/test split and for each landscape we report averages across the three distinct splits. 

\paragraph{Extrapolation}

For each landscape we determine the reference/wild-type sequence $\bxref$.
For ProteinGym datasets this is the \texttt{target\_seq} field in the \texttt{DMS\_substitutions.csv} file,
while for the remaining datasets we take the reference sequence from the corresponding publication.
For each variant $\bx$ in a given landscape we compute its Hamming distance to the target, $d_{\rm H}(\bx, \bxref)$.
We adaptively select a Hamming distance cutoff $D$ for each landscape as the smallest integer satisfying $3 \le D \le 5$ such that 
at least $512$ sequences have $d_{\rm H}(\bx, \bxref) \le D$ and at least $384$ sequences have $d_{\rm H}(\bx, \bxref) > D$.
The initial training set is all sequences within a Hamming distance of $D$ of the 
target sequence (i.e.~we choose $\bx$ such that $d_{\rm H}(\bx, \bxref) \le D$).
The training set is then subsampled uniformly at random to the desired number of training data points (either $128$ or $512$). 
The test set is all sequences more than a Hamming distance of $D$ from the 
target sequence (i.e.~we choose $\bx$ such that $d_{\rm H}(\bx, \bxref) > D$) and is not subsampled.
By construction all test sets have at least $384$ sequences.

%% file: dataset_table.tex
\begin{table*}[t]
\centering
\scriptsize
    \caption{We summarize the datasets used in experiments in \secref{sec:ablation} and \secref{sec:local}.
    We selected datasets that are sufficiently large and have considerable sequence diversity, including an abundance
    of higher-order mutants (see Table \ref{tab:datasets_numeric}). Datasets range in size from $1,\!812$ sequences to $496,\!137$ sequences. 
    The $9$ datasets from ProteinGym \cite{notin2023proteingym} are referred to
 by their ProteinGym IDs (i.e.~from \texttt{D7PM05\_CLYGR\_Somermeyer\_2022} to \texttt{GCN4\_YEAST\_Staller\_2018}).
    Assay types are categorized by their corresponding ProteinGym category: activity, binding, organismal fitness, stability, and expression;
    note that we did not include any expression datasets.
}
\label{tab:datasets}
\begin{tabular}{lrllll}
\toprule
\textbf{Name} & \textbf{N} & \textbf{Ref. structure} & \textbf{Readout type} & \textbf{Structure in complex} & \textbf{Reference} \\
\midrule
D7PM05\_CLYGR\_Somermeyer\_2022 & 24515 & Predicted & Activity & No & \citet{Gonzalez_Somermeyer_2022} \\
Q6WV12\_9MAXI\_Somermeyer\_2022 & 31401 & Predicted & Activity & No & \citet{Gonzalez_Somermeyer_2022} \\
Q8WTC7\_9CNID\_Somermeyer\_2022 & 33510 & Predicted & Activity & No & \citet{Gonzalez_Somermeyer_2022} \\
GFP\_AEQVI\_Sarkisyan\_2016 & 51714 & Predicted & Activity & No & \citet{Sarkisyan_2016} \\
CAPSD\_AAV2S\_Sinai\_2021 & 42328 & Predicted & OrganismalFitness & No & \citet{Sinai_2021} \\
F7YBW8\_MESOW\_Ding\_2023 & 7922 & Predicted & OrganismalFitness & No & \citet{Ding_2024} \\
HIS7\_YEAST\_Pokusaeva\_2019 & 496137 & Predicted & OrganismalFitness & No & \citet{Pokusaeva_2019} \\
PHOT\_CHLRE\_Chen\_2023 & 167529 & Predicted & Activity & No & \citet{Chen_2023} \\
GCN4\_YEAST\_Staller\_2018 & 2638 & Predicted & Activity & No & \citet{Staller_2018} \\
CH65 - MA90 & 65530 & Experimental & Binding & Yes & \citet{Phillips_2023} \\
CH65 - SI06 & 64619 & Experimental & Binding & Yes & \citet{Phillips_2023} \\
CH65 - G189E & 63840 & Experimental & Binding & Yes & \citet{Phillips_2023} \\
CR9114 - H1 & 65093 & Experimental & Binding & Yes & \citet{Shanker_2024} \\
CR9114 - H3 & 65093 & Experimental & Binding & Yes & \citet{Shanker_2024} \\
CR6261 - H1 & 1812 & Experimental & Binding & Yes & \citet{Shanker_2024} \\
CR6261 - H9 & 1812 & Experimental & Binding & Yes & \citet{Shanker_2024} \\
Moulana Regn10987 & 30479 & Experimental & Binding & Yes & \citet{Moulana_2023} \\
Moulana CB6 & 32112 & Experimental & Binding & Yes & \citet{Moulana_2023} \\
Moulana Cov555 & 29892 & Experimental & Binding & Yes & \citet{Moulana_2023} \\
Tsuboyama - 2B88 & 2937 & Predicted & Stability & No & \citet{Tsuboyama_2023} \\
Tsuboyama - 2HBB & 4494 & Predicted & Stability & No & \citet{Tsuboyama_2023} \\
\bottomrule
\end{tabular}
\end{table*}

%% file: dataset_table_numeric.tex
\begin{table*}[t]
\centering
\scriptsize
    \caption{We provide detailed summary statistics for the $21$ datasets used for the experiments in 
    \secref{sec:ablation} and \secref{sec:local}.
    Each dataset has a maximum Hamming distance from reference ($d_{\rm H}$) of at least $6$,
    an average $d_{\rm H}$ of at least $2.7$, and at least $10$ variable positions.}
\label{tab:datasets_numeric}
\begin{tabular}{lrrrr}
\toprule
\textbf{Name} & \textbf{Avg.~$d_{\rm H}$} & \textbf{Max $d_{\rm H}$} & \textbf{Total sequence length} & \textbf{\# of variable positions} \\
\midrule
D7PM05\_CLYGR\_Somermeyer\_2022 & 3.0 & 23 & 235 & 234 \\
Q6WV12\_9MAXI\_Somermeyer\_2022 & 2.7 & 13 & 222 & 221 \\
Q8WTC7\_9CNID\_Somermeyer\_2022 & 3.1 & 43 & 238 & 237 \\
GFP\_AEQVI\_Sarkisyan\_2016 & 3.9 & 15 & 238 & 233 \\
CAPSD\_AAV2S\_Sinai\_2021 & 4.7 & 28 & 735 & 28 \\
F7YBW8\_MESOW\_Ding\_2023 & 5.4 & 10 & 93 & 10 \\
HIS7\_YEAST\_Pokusaeva\_2019 & 7.0 & 28 & 220 & 171 \\
PHOT\_CHLRE\_Chen\_2023 & 8.0 & 15 & 118 & 118 \\
GCN4\_YEAST\_Staller\_2018 & 17.1 & 44 & 281 & 44 \\
CH65 - MA90 & 8.0 & 16 & 760 & 16 \\
CH65 - SI06 & 8.0 & 16 & 760 & 16 \\
CH65 - G189E & 8.0 & 16 & 760 & 16 \\
CR9114 - H1 & 8.0 & 16 & 730 & 16 \\
CR9114 - H3 & 8.0 & 16 & 730 & 16 \\
CR6261 - H1 & 5.5 & 11 & 731 & 11 \\
CR6261 - H9 & 5.5 & 11 & 731 & 11 \\
Moulana Regn10987 & 7.5 & 15 & 1043 & 15 \\
Moulana CB6 & 7.5 & 15 & 1256 & 15 \\
Moulana Cov555 & 7.6 & 15 & 1255 & 15 \\
Tsuboyama - 2B88 & 3.6 & 6 & 54 & 54 \\
Tsuboyama - 2HBB & 3.3 & 7 & 48 & 48 \\
\bottomrule
\end{tabular}
\end{table*}

%% file: speed.tex
\begin{table*}[h!]
\centering
\footnotesize
\setlength{\tabcolsep}{3pt}
\input{speed_tables.tex}
    \caption{We report training and inference times as well as the corresponding peak GPU memory usage for 
    six models on the \citet{Chen_2023} dataset. 
    We vary the number of training points as $N \in \{1024,16384\}$ and keep the inference workload fixed at $1024$ predictions.
    Results are obtained using a NVIDIA A10G GPU.}
\label{tab:speed-benchmark}
\end{table*}

We benchmark the compute time\footnote{Note that to benchmark inference compute time we skip any computations 
that can be cached and subsequently reused.
For example, computing ProteinMPNN inverse folding logits for \texttt{Kermut-GP} need only be done once, since the reference
structure $\SSS$ is assumed fixed. Similarly, for training we make sure to avoid unnecessary recomputation wherever possible.} and peak memory usage of six representative models; see Table \ref{tab:speed-benchmark} for the results.
We find that \texttt{LOCK-GP} training is $\sim 3-140$x faster than the other five models, all of which make use of at least one foundation model. 
Similarly \texttt{LOCK-GP} inference is $\sim 2-80$x faster than the other models considered.
By contrast the picture for peak memory usage is more complex, and in any case this quantity can depend on the mini-batch size used
for computing foundation model embeddings---this applies to both training and inference---as well as other implementation details. 
Notably peak memory usage for GP models becomes substantial for $N \gtrsim 10^4$.
While this can be ameliorated by using inducing point methods or other approximate GP inference techniques (see \secref{sec:app:scalability}), kernel methods can be prone to significant memory usage unless care is taken to avoid realizing large kernel matrices. 
We note that we can estimate the inference time of other models in Table \ref{tab:metric_summary_long} by considering ESM-2 usage and 
comparing to Table \ref{tab:speed-benchmark}.
For example, the inference time of \texttt{Ridge-ESM2-650M} is roughly comparable to \texttt{MLP-ESM2-650M}, since computing
ESM-2 embeddings is the dominant computational cost. 

We also note that a similar inference time hierarchy exists for the multi-task models benchmarked in \secref{sec:multi}.
In particular \texttt{CLOCK-GP} is dramatically faster at inference time than \texttt{CNN-EMB-OH},
since the positional structure embeddings $\bh_{1:L}$ used by \texttt{CLOCK-GP} can be computed once and cached,
while for \texttt{CNN-EMB-OH} positional structure-and-sequence embeddings must be computed for each new test sequence.

%% file: speed_tables.tex
\begin{minipage}[t]{0.49\textwidth}
\centering
\textbf{Training time}\\[0.4ex]
\begin{tabular}{@{}lcc@{}}
\toprule
\textbf{Model} & $N = 1024$ & $N = 16384$ \\
\midrule
\texttt{LOCK-GP} & 12.8\,s & 1.8\,m \\
\texttt{Kermut-GP} & 1.8\,m & 13.9\,m \\
\texttt{MLP-ESM2-8M} & 40.7\,s & 5.8\,m \\
\texttt{MLP-ESM2-650M} & 1.7\,m & 11.3\,m \\
\texttt{MLP-ESM2-LastLayer} & 2.7\,m & 41.1\,m \\
\texttt{MLP-SaProt-LastLayer} & 21.0\,m & 249.3\,m \\
\bottomrule
\end{tabular}
\end{minipage}
\hfill
\begin{minipage}[t]{0.49\textwidth}
\centering
\textbf{Training peak memory}\\[0.4ex]
\begin{tabular}{@{}lcc@{}}
\toprule
\textbf{Model} & $N = 1024$ & $N = 16384$ \\
\midrule
\texttt{LOCK-GP} & 154\,MB & 11.0\,GB \\
\texttt{Kermut-GP} & 5.8\,GB & 16.6\,GB \\
\texttt{MLP-ESM2-8M} & 146\,MB & 358\,MB \\
\texttt{MLP-ESM2-650M} & 3.2\,GB & 3.4\,GB \\
\texttt{MLP-ESM2-LastLayer} & 215\,MB & 361\,MB \\
\texttt{MLP-SaProt-LastLayer} & 425\,MB & 562\,MB \\
\bottomrule
\end{tabular}
\end{minipage}

\vspace{0.75ex}

\begin{minipage}[t]{0.49\textwidth}
\centering
\textbf{Inference time}\\[0.4ex]
\begin{tabular}{@{}lcc@{}}
\toprule
\textbf{Model} & $N = 1024$ & $N = 16384$ \\
\midrule
\texttt{LOCK-GP} & 0.2\,s & 0.7\,s \\
\texttt{Kermut-GP} & 40.3\,s & 43.1\,s \\
\texttt{MLP-ESM2-8M} & 1.5\,s & 1.5\,s \\
\texttt{MLP-ESM2-650M} & 20.6\,s & 19.7\,s \\
\texttt{MLP-ESM2-LastLayer} & 1.1\,s & 1.1\,s \\
\texttt{MLP-SaProt-LastLayer} & 4.0\,s & 4.0\,s \\
\bottomrule
\end{tabular}
\end{minipage}
\hfill
\begin{minipage}[t]{0.49\textwidth}
\centering
\textbf{Inference peak memory}\\[0.4ex]
\begin{tabular}{@{}lcc@{}}
\toprule
\textbf{Model} & $N = 1024$ & $N = 16384$ \\
\midrule
\texttt{LOCK-GP} & 83\,MB & 5.2\,GB \\
\texttt{Kermut-GP} & 5.8\,GB & 14.3\,GB \\
\texttt{MLP-ESM2-8M} & 141\,MB & 152\,MB \\
\texttt{MLP-ESM2-650M} & 3.2\,GB & 3.3\,GB \\
\texttt{MLP-ESM2-LastLayer} & 1.6\,GB & 1.7\,GB \\
\texttt{MLP-SaProt-LastLayer} & 846\,MB & 976\,MB \\
\bottomrule
\end{tabular}
\end{minipage}
\label{tab:speed2-benchmark}

%% file: app_exp.tex
In our experiments we leverage Pytorch \cite{DBLP:conf/nips/PaszkeGMLBCKLGA19}, Gpytorch \cite{DBLP:conf/nips/GardnerPWBW18}, 
Botorch \cite{DBLP:conf/nips/BalandatKJDLWB20}, and Pytorch Lightning \cite{falcon_pytorch_lightning_2019}.
In particular, unless noted otherwise we use \texttt{fit\_gpytorch\_mll} from Botorch to fit GPs, which makes use of the quasi-Newton method L-BFGS under the hood.

\input{app_exp_local.tex}
\input{app_exp_multi.tex}

\input{app_exp_binary.tex}

%% file: app_exp_local.tex
\subsubsection{Local Learning Experiment}
\label{app:sec:localdetails}

We describe the various baselines used in the experiment in \secref{sec:local}.
We also describe some additional baselines for which we did not report results in the main text but which appear in Table \ref{tab:metric_summary_long}:
{\small
\begin{multicols}{4}
\begin{itemize}[leftmargin=*]
    \item \texttt{ESM2-650M-ZeroShot}
    \item \texttt{BLOSUM50-ZeroShot}
    \item \texttt{Ridge-ESM2-650M}
    \item \texttt{Ridge-ESM2-8M-MeanPool}
    \item \texttt{Ridge-OH-ESM2-650M-Aug}
    \item \texttt{MLP-ESM2-650M}
    \item \texttt{MLP-ESM2-8M-MeanPool}
    \item \texttt{MLP-OH-ESM2-Aug}
    \item \texttt{MLP-ESM2-8M-RandInit}
    \item \texttt{MLP-OH}
    \item \texttt{ConFit}
    \item \texttt{MLP-SaProt-LastLayer}
\end{itemize}
\end{multicols}
}

\paragraph{\texttt{Tanimoto-GP}}

We provide a brief description of the kernel in \citet{gessner2024active-mlsb}.
Let $\log^{\circ} \bS \in \mathbb{R}^{A \times A}$ be a BLOSUM substitution matrix in log-odds space.
Compute its eigendecomposition $\log^{\circ} \bS = \bU \bD \bU^\top$.
Since $\log^{\circ} \bS$ is not PSD, we cannot take the square root of $\bD$. 
In \citet{gessner2024active-mlsb} the authors replace $\bD$ with its entrywise absolute value $|\bD|$ before taking the square root. 
We find slightly better performance (see \secref{app:ablation}) by using elementwise clamping instead: $\bD_+ = \max(\bD, 0)$.
For one-hot sequence $\bx \in \mathbb{R}^{L \times A}$ we then encode $\bx$ as $\phi(\bx) = \bx \bU \sqrt{\bD_+} \in \mathbb{R}^{L \times A}$.
We then define the resulting Tanimoto kernel \cite{tripp2023tanimoto} as
\begin{equation}
k_{\rm Tani}(\bx, \bx') = \frac{\langle \phi(\bx), \phi(\bx') \rangle_F}{\|\phi(\bx)\|_F^2 + \|\phi(\bx')\|_F^2 - \langle \phi(\bx), \phi(\bx') \rangle_F}
\end{equation}
where $\langle \cdot, \cdot \rangle_F$ is the Frobenius inner product. 
\citet{gessner2024active-mlsb} does not specify which BLOSUM matrix is used; we use BLOSUM62.
As shown in \secref{app:ablation}, there is only weak dependence on the particular substitution matrix used.

\paragraph{\texttt{Kermut-GP}}

We follow the detailed model specification described in \citet{groth2024kermut}. 
The authors' released code initializes the two parameters that define the prior mean by sampling both the scalar weight and bias from $\mathcal{N}(0, 1)$.
We instead sample the weight from $\mathcal{N}(0, 0.01)$ and initialize the bias to $0$.

\paragraph{\texttt{KermutSeq-GP}}

Like \texttt{Kermut-GP} but without the structure kernel.

\paragraph{\texttt{KermutStruc-GP}}

Like \texttt{Kermut-GP} but without the RBF sequence kernel.
See \secref{app:kermut} for an extended discussion of the structure kernel.
As described in \citet{groth2024kermut}, the structure kernel includes an overall kernel scale.
The authors' released code omits this scaling for the structure-only kernel.
Following the paper, we include this learnable scale in our implementation.

\paragraph{\texttt{Ridge-OH}}

We use \texttt{RidgeCV} from sklearn on one-hot encoded sequences,
which selects from $32$ ridge parameters $\alpha_{\rm ridge}$ that are geometrically evenly spaced between $10^{-4}$ and $10^{4}$ (i.e.~{\texttt{numpy.logspace(-4, 4, 32)}).

\paragraph{\texttt{Ridge-ESM2}}

We use \texttt{RidgeCV} from sklearn on (non-mean-pooled) features from ESM-2-8M. 
We extract features by selecting the variable positions from the last hidden layer, flattening into a vector, and then elementwise centering and unit-scaling.
We use $32$ ridge parameters $\alpha_{\rm ridge}$ that are geometrically evenly spaced between $10^{-4}$ and $10^{4}$ (i.e.~{\texttt{numpy.logspace(-4, 4, 32)}).
In Table~\ref{tab:metric_summary_long} we also reports results for \texttt{Ridge-ESM2-650M}.

\paragraph{\texttt{MLP-ESM2-8M} and \texttt{MLP-ESM2-650M}}

We train a single-layer neural network with $128$ hidden units and a ReLU non-linearity on (non-mean-pooled) features from ESM-2.
We extract features by selecting the variable positions from the last hidden layer, flattening into a vector, and then elementwise centering and unit-scaling.
We use a dropout probability of $0.5$ and a L2 regularizer $\lambda=0.1$ (i.e.~relatively strong weight decay).
We train for $1000$ epochs in batch mode using Adam with default settings (i.e.~$\texttt{lr}=10^{-3}$, $\beta_1=0.9$, $\beta_2=0.999$) 
and a linear learning-rate schedule that decays to $0$ over $1000$ epochs.

\paragraph{\texttt{MLP-ESM2-LastLayer}}

We use the same feature extraction pipeline as \texttt{MLP-ESM2-8M} but fine-tune the last layer of ESM-2-8M jointly with the MLP head.
We use $32$ hidden units, a dropout probability of $0.5$, and a LeakyReLU non-linearity with negative slope $0.1$.
We train for $100$ epochs using AdamW with a batch size of $24$, gradient clipping with a maximum norm of $0.5$, 
and a learning-rate schedule with $30$ epochs of linear warmup (starting at a factor of $0.01$) followed by cosine decay to a minimum of $10^{-6}$.
We use a peak learning rate of $10^{-4}$ for training ESM-2 parameters and $10^{-3}$ for the MLP head, with weight decay $0.1$ (except for LayerNorm parameters, which have no weight decay); other AdamW settings are set to default values.

\paragraph{\texttt{SigGLM-OH}}

The model is of the form $f(\bx) = a \; \sigma(\bbeta \cdot \bx + b) + c$ where $\sigma(\cdot)$ is a logistic link function (i.e.~\texttt{torch.sigmoid}),
$a,b,c\in \RR$ are scalars, $\bbeta \in \RR^{L \times A}$ are coefficients,
and $\bx$ is a one-hot encoded sequence. 
We fit the model by minimizing a MSE loss function that also includes a L2 regularizer on $\bbeta$; we use the Adam optimizer.
The strength of the L2 regularizer is chosen via cross-validation.

\paragraph{\texttt{ConFit}}

ConFit (Contrastive Fitness Learning) fine-tunes a pre-trained PLM using a contrastive Bradley–Terry loss and
parameter-efficient LoRA adapters \cite{zhao2024contrastive}. 
Like \cite{zhao2024contrastive} we use an ensemble of five distinct ESM-1v checkpoints \cite{meier2021esm1v};
we do not use MSA context retrieval.
Our implementation is closely based on the authors' implementation at \texttt{https://github.com/luo-group/ConFit}, with a few adaptations.
We removed the KL regularization term from the loss, since we found no benefit from including it. Since zero-shot ESM-1v performance is quite
poor on most of our datasets, this is not unexpected. 
Since some of the datasets in \secref{sec:data} exhibit ties or near ties w.r.t.~the targets (i.e.~fitness scores), we modify the Bradley–Terry loss
as follows: i) we explicitly model near-ties (defined with a threshold of $10^{-5}$) using a soft target of $0.5$; and
ii) we suppress the influence of ties by downweighting their contribution to the loss by a factor of $10^{3}$.
We validated our implementation by running the author's evaluation script and comparing the resulting Spearman correlations to those obtained with our implementation,
finding that the Spearman correlations were within $0.001$ (or better) than those obtained with the authors' implementation.

Since ConFit uses a ranking-based loss instead of a regression loss, it does not yield calibrated fitness predictions;
indeed the absolute scale of the predictions is arbitrary.
Consequently, in order to compute reasonable regression metrics like MAE and RMSE, we apply an additional linear calibration step to ConFit predictions 
in which we fit two parameters---a scalar intercept and scalar coefficient---to the raw ConFit predictions.\footnote{Additionally, we normalize
the raw predictions from each ensemble member by their standard deviation before averaging predictions across the ensemble.}
We fit these parameters using \texttt{RidgeCV} from sklearn. 
This calibration step has minimal impact on Pearson or Spearman correlation metrics.

\paragraph{\texttt{ESM2-650M-ZeroShot}}

Zero-shot scores are computed using the masked marginal scoring method as described in \citet{meier2021esm1v}.

\paragraph{\texttt{BLOSUM50-ZeroShot}}

For each sequence $\bx$, zero-shot scores are computed by summing BLOSUM50 (log-odds) substitution scores\footnote{That is we compute the sum using entries 
$\log S_{a a_{\rm ref}}$ for amino acids $a$ and $a_{\rm ref}$, where we note that in our notation $\bS$ is the exponentiated substitution 
matrix and $\log^{\circ} \bS$ is the raw matrix with log-odds.} position-by-position against the reference sequence $\bxref$ \cite{altschul1991amino}.

\paragraph{\texttt{Ridge-OH-ESM2-Aug}}

Like \texttt{Ridge-OH} above except the linear model uses one-hot features concatenated with the normalized (scalar) zero-shot scores used in \texttt{ESM2-650M-ZeroShot}.
This baseline is inspired by \citet{hsu2022learning}.

\paragraph{\texttt{MLP-OH}}

We use a single-layer MLP like in \texttt{MLP-ESM2-8M} above except with one-hot features and with $32$ instead of $128$ hidden units.

\paragraph{\texttt{MLP-OH-ESM2-Aug}}

We use a single-layer MLP like in \texttt{MLP-ESM2-8M} above and with the same zero-shot-augmented features as in \texttt{Ridge-OH-ESM2-Aug}.
We use $32$ instead of $128$ hidden units.

\paragraph{\texttt{Ridge-ESM2-8M-MeanPool}}

Like \texttt{Ridge-ESM2} above, except the input features are mean-pooled across variable positions, rather than being flattened.

\paragraph{\texttt{MLP-ESM2-8M-MeanPool}}

Like \texttt{MLP-ESM2-8M} above, except the MLP uses $32$ hidden units and the input features are mean-pooled across variable positions, rather than being flattened.

\paragraph{\texttt{MLP-SaProt-LastLayer}}

Like \texttt{MLP-ESM2-LastLayer} above, except the structure-conditioned PLM SaProt-35M \cite{su2024saprot} 
is used instead of ESM-2-8M.

\paragraph{PLM Usage}

ESM-2 and ESM-1v were pre-trained on single chain sequences, while some of our datasets consist of multi-chain proteins.
Therefore, when computing ESM-2 embeddings or fine-tuning ESM-1v we proceed as follows.
For each sequence we: i) obtain the chain boundaries from the associated reference structure;
ii) compute independent forward passes for each chain; and
iii) concatenate the outputs. 
Note that we use two ESM-2 variants, namely a smaller model with $8$ million parameters and a larger model with $650$ million parameters.

SaProt is used as follows. Residues with incomplete backbone coordinates are masked using the \# token. 
If only C$\beta$ is missing, mini3di reconstructs a virtual C$\beta$ from the N, C$\alpha$, and C coordinates.
For predicted structures, pLDDT-based masking with a threshold of $70$ is applied after 3Di encoding: 
tokens are first computed for all structurally valid residues, and then positions with low pLDDT are replaced with \#. 
This preserves accurate 3Di tokens for neighboring residues.

\paragraph{Brief Discussion of Additional Baselines}

The ESM-2-based zero-shot baseline \texttt{ESM2-650M-ZeroShot} has very variable performance, with some Pearson/Spearman correlation
coefficients being markedly negative; see Tables~\ref{tab:landscape_pearsonr_cv}-\ref{tab:landscape_mae_extrapolation}.
This kind of variable performance is hardly surprising given the complex relationship between evolutionary preferences and fitness, 
especially for antibodies, which are subject to stochastic recombination and somatic hypermutation; see \cite{chungyoun2024flab} for similar observations.
By contrast BLOSUM50-based zero-shot scores are consistently positive except for the two thermostability landscapes from \citet{Tsuboyama_2023}
and some of the antibody landscapes from \citet{Moulana_2023}. 
Swapping ESM2-8M for ESM-650M in ridge regression and MLP models has mixed effects on performance.
While \texttt{Ridge-ESM2-650M} outperforms \texttt{Ridge-ESM2-8M} in the cross-validation setting, predictive performance is roughly comparable in the two OOD regimes.
Models that leverage scalar zero-shot features like \texttt{Ridge-OH-ESM2-Aug} and \texttt{MLP-OH-ESM2-Aug} generally
outperform their non-augmented counterparts, i.e.~\texttt{Ridge-OH} and \texttt{MLP-OH}.
We find that mean-pooling features like in \texttt{Ridge-ESM2-8M-MeanPool} and \texttt{MLP-ESM2-8M-MeanPool} degrades performance.
We find that \texttt{ConFit} and \texttt{MLP-ESM2-8M} perform similarly to \texttt{MLP-ESM2-LastLayer},
although \texttt{MLP-ESM2-LastLayer} performs best over all.
In particular all three of these methods perform quite well in the cross-validation setting but exhibit degraded performance in OOD settings. 
As we would expect, since \texttt{ConFit} is a ranking-based method, it exhibits poorer Pearson R and MAE metrics.
We find that using ESM2-8M as a random feature generator as in \texttt{MLP-ESM2-8M-RandInit} leads to surprisingly good results---for example,
often outperforming \texttt{MLP-ESM2-8M-MeanPool}---suggesting that a non-trivial portion of the good performance of methods 
like \texttt{MLP-ESM2-8M} should be attributed to the general representational capacity
of high-dimensional transformer-based sequence-aware n-gram-like features that is orthogonal to any particular learnings 
from global protein data obtained via pre-training.
See \cite{pmlr-v235-li24a} for related benchmarks that involve randomly initialized protein language models. 
We also find that, while the overall picture is somewhat mixed, \texttt{MLP-OH} outperforms \texttt{Ridge-OH} in most regimes w.r.t.~MAE, 
highlighting the value of making the predictive model non-linear. 
Finally in \texttt{MLP-SaProt-LastLayer} we include an additional model beyond \texttt{Kermut-GP} and \texttt{KermutStruc-GP} that
utilizes the reference structure $\SSS$. 
We find that, while \texttt{MLP-SaProt-LastLayer} outperforms \texttt{MLP-ESM2-LastLayer} on average, especially in the OOD setting, 
the gains are quite modest overall.

%% file: app_exp_multi.tex
\subsubsection{Multi-task Experiment}
\label{app:sec:multitaskdetails}

We first describe the general training and evaluation setup. 
We assume a set of training landscapes $\{\DD_k\}$ where each $\DD_k = \{\bX_k, \bt_k, \SSS_k \}$ consists of $N_k$ 
sequences $\bX_k = \{ \bx_{k,n} \}$, $N_k$ properties/targets $\bt_k = \{ t_{k,n} \}$, and a reference structure $\SSS_k$.
We are given a heldout landscape $\DD^\prime = \{\bX^\prime, \bt^\prime, \SSS^\prime \}$ with $N^\prime$ data points.
Then we use the training landscapes $\{\DD_k\}$ to fit $\bW$ (in the case of \texttt{CLOCK-GP}) or 
neural network weights (in the case of \texttt{CNN-EMB-OH}). 
The resulting model artifact is then used in conjunction with a subset of data from $\DD^\prime$ 
(say $\{\bX_{\rm train}^\prime, \bt_{\rm train}^\prime, \SSS^\prime \}$ with $N_{\rm train}^\prime < N^\prime$ data points) to
make predictions on test data from $\DD^\prime$ (in particular the complement $\{\bX_{\rm test}^\prime, \bt_{\rm test}^\prime, \SSS^\prime \}$ with 
$N_{\rm test}^\prime = N^\prime - N_{\rm train}^\prime$ data points).
It is the performance of these model predictions for $N_{\rm test}^\prime$ test data points that we report.

In more detail, we randomly split the $371$ landscapes into $280$ training landscapes, $41$ validation landscapes, and $50$ test landscapes.
The split is i.i.d.~except that we require test landscapes to have at least $800$ sequences to enable training on subsets
of size $700$.
The target label we use is \texttt{ddG\_ML}. 
We use predicted structures provided by \citet{Tsuboyama_2023}.
The thermostability data has a small number of duplicate sequences; we aggregate duplicates by computing the mean \texttt{ddG\_ML}. 
To make aggregate MAE/RMSE metrics easier to interpret, we preprocess all landscapes so that their targets are zero-centered and have unit
standard deviation.

For evaluation we define two i.i.d.~train/test splits for each of the $50$ heldout landscapes for a total of $100$ unique train/test pairs. 
For example, if a given landscape has $1000$ data points, and the training size is $N=100$, then the test set is given by 
the remaining $900$ data points.
The model is then trained on $N=100$ data points and evaluated on $900$ data points.
For this reason the test set changes as $N$ increases; this is why the zero-shot prediction metrics are not constant w.r.t.~the training
data size (though they are nearly constant).
The same $100$ train/test splits are used for all model evaluations.

For the landscape subsampling experiment we train models on i.i.d.~subsets of the $280$ training landscapes and consider
$n_{\rm landscapes} \in \{10, 20, 35, 70, 140 \}$. 
For models trained on fewer than $280$ landscapes, we train two models on distinct random 
subsets of the $280$ training landscapes, in which case metrics are averaged across both models.
This reduces variance w.r.t.~which of the $280$ training landscapes are chosen.

\paragraph{CLOCK-GP Variants}

A basic implementation of the PyTorch module that we use to compute correlation matrices is provided in \figref{code:corrmachine}.
A basic implementation of the concentrated-likelihood-based loss function is provided in \figref{code:loss}.
The signal-to-noise ratio is initialized to $40$. We use the Adam optimizer and train for $300$ epochs. 
Chroma embeddings are computed at diffusion time $t=0.01$.
We use an initial learning rate of $0.005$ and decimate the learning rate twice (once after epoch $100$ and once after epoch $200$). 
In each epoch we evaluate the loss on two mini-batches and thus take two gradient steps per epoch. 
Each mini-batch consists of a random subset of half the available landscapes. 
We randomly choose $512$ sequences from each landscape, and we center the targets in each sampled landscape. 
In contrast to the CNN, we do not use the validation data in any way, as we did not observe any evidence for overfitting.\footnote{This conclusion
may not hold in the low landscape regime with $\lesssim 20$ landscapes (where more additional regularization or early stopping may be prudent), but we have not systematically investigated overfitting in this regime.}
Indeed we found the \texttt{CLOCK-GP} very robust to various choices, for example: i) whether the signal-to-noise ratio is fixed or learned; 
ii) whether the signal-to-noise ratio is global or landscape-specific; iii) optimizer settings; 
iv) number of landscapes that enter into each gradient step; etc.

The RBF kernel used in \texttt{CLOCK-GP+CNN-ZS} is given by $k_{\rm rbf}(\bx, \by) = \exp \left(-\tfrac{1}{2}({\rm pred}(\bx) - {\rm pred}(\by))^2 / \tau^2 \right)$,
where $\tau>0$ is a learnable length scale and ${\rm pred}(\cdot)$ is the CNN-based zero-shot predictor.
This kernel is equipped with a learnable kernel scale and added to the LOCK kernel in \eqnref{eqn:lockdefn}.

\paragraph{CNN Baselines}

The CNN baseline architecture takes positional structure-sequence embeddings of dimension $128$ and projects them to $256$ channels via a $1 \times 1$ convolution. 
Next multi-scale convolutional processing is performed using three parallel branches with kernel sizes of $3$, $5$, and $7$, where each branch applies two sequential convolutional layers with batch normalization, ReLU activation functions, and dropout (with a rate of $0.05$). 
The outputs from all three branches are globally average-pooled and
concatenated, yielding a $768$ feature vector that is passed through a two-layer MLP ($768 \to 512 \to 256 \to 1$) to produce the final fitness prediction.
We note that while \texttt{CNN-EMB-OH} does not obviously \emph{need} its feature space to be augmented with one-hot sequences---since the positional
structure-sequence embeddings from Chroma (computed at diffusion time $t=0.01$) explicitly depend on sequence information---we found that explicit augmentation results in better performance.
Similarly, we also considered variants in which the CNN uses structure-only embeddings (as is done in \texttt{CLOCK-GP}) together with
one-hot sequence features, but found that this performs worse---as we would expect. 
In other words performance differences between the CNN-based methods and  \texttt{CLOCK-GP} cannot be explained by the quality of the embeddings used.

\paragraph{Ridge-OH}

We use \texttt{RidgeCV} from sklearn on one-hot encoded sequences,
which selects from $32$ ridge parameters $\alpha_{\rm ridge}$ that are geometrically evenly spaced between $10^{-4}$ and $10^{4}$ (i.e.~{\texttt{numpy.logspace(-4, 4, 32)}).

\paragraph{Ridge-EMB and Ridge-EMB-OH}
These are like \texttt{Ridge-OH} above but use different features: \texttt{Ridge-EMB} uses mean-pooled structure-and-sequence embeddings from Chroma,
while \texttt{Ridge-EMB-OH} concatenates mean-pooled structure-and-sequence embeddings with one-hot encoded sequences.

%% file: app_exp_binary.tex
\subsubsection{Binary Classification Experiment}
\label{app:sec:binarydetails}

We consider the \texttt{Q8WTC7\_9CNID\_Somermeyer\_2022} landscape from \citet{Gonzalez_Somermeyer_2022}.
We use the same cross-validation setup as in \secref{app:sec:localdetails} except we stratify to ensure class labels are approximately balanced.
We quantize the continuous fitness score from \texttt{Q8WTC7\_9CNID\_Somermeyer\_2022} using the median so that, for example, $50\%$ of binary labels are positive.
For GP inference we use the method from \citet{wenzel2019efficient}, closely following the GPyTorch implementation.\footnote{See {\scriptsize \texttt{https://docs.gpytorch.ai/en/v1.12/examples/04\_Variational\_and\_Approximate\_GPs/PolyaGamma\_Binary\_Classification.html}}.} 
For $N$ training data points we use $N_{\rm ind} = \rm{min}(N, 2048)$ inducing point locations.
Inducing points are chosen at random from the training data and kept fixed. 
To optimize the ELBO we use a maximum mini-batch size of $1024$ and train for $30$ epochs.
We decimate the learning rate after $10$ and $20$ epochs.

%% file: app_add_local.tex
\begin{figure*}[t]
\centering
    \includegraphics[width=0.49\columnwidth]{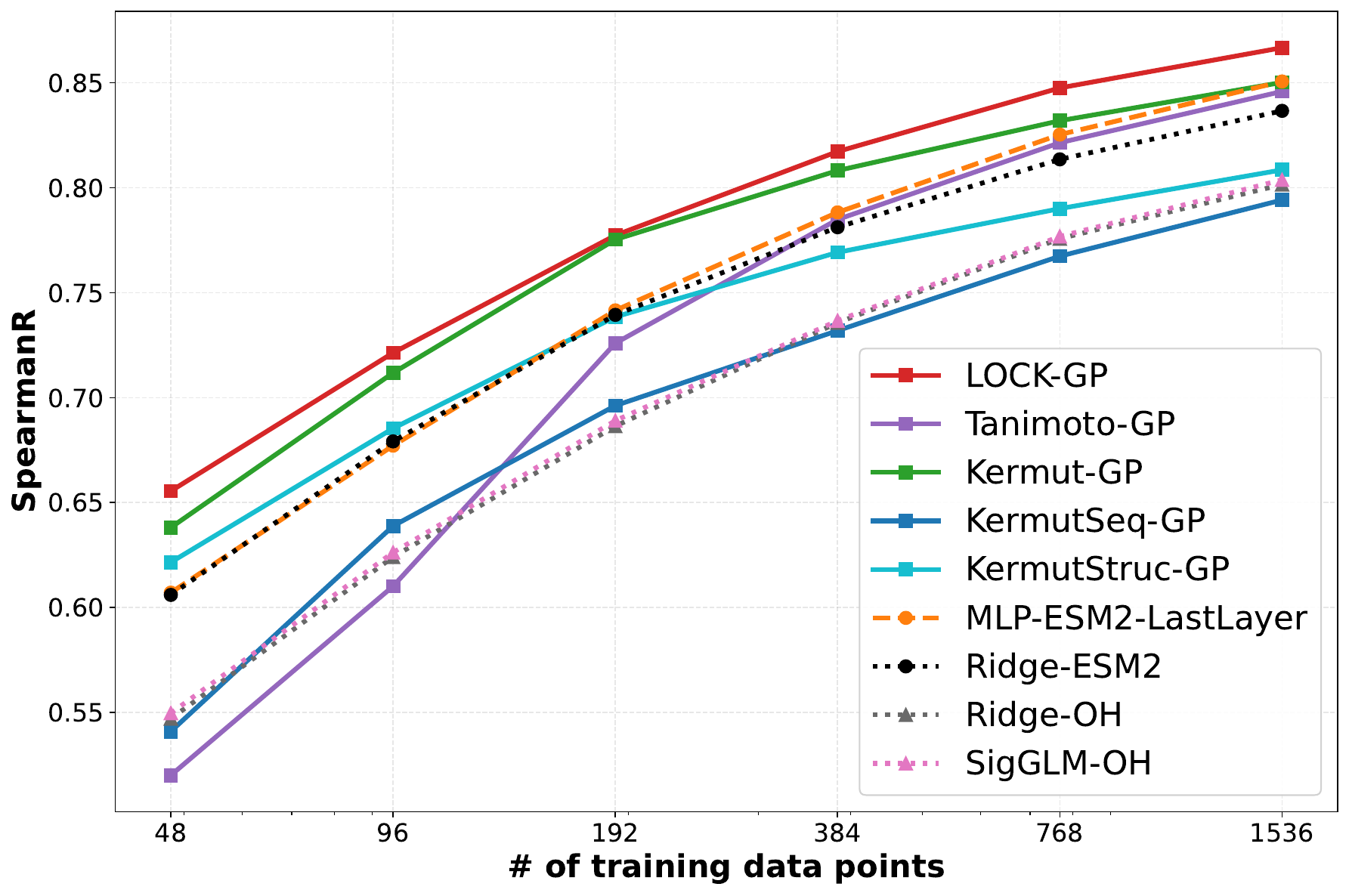}
    \includegraphics[width=0.49\columnwidth]{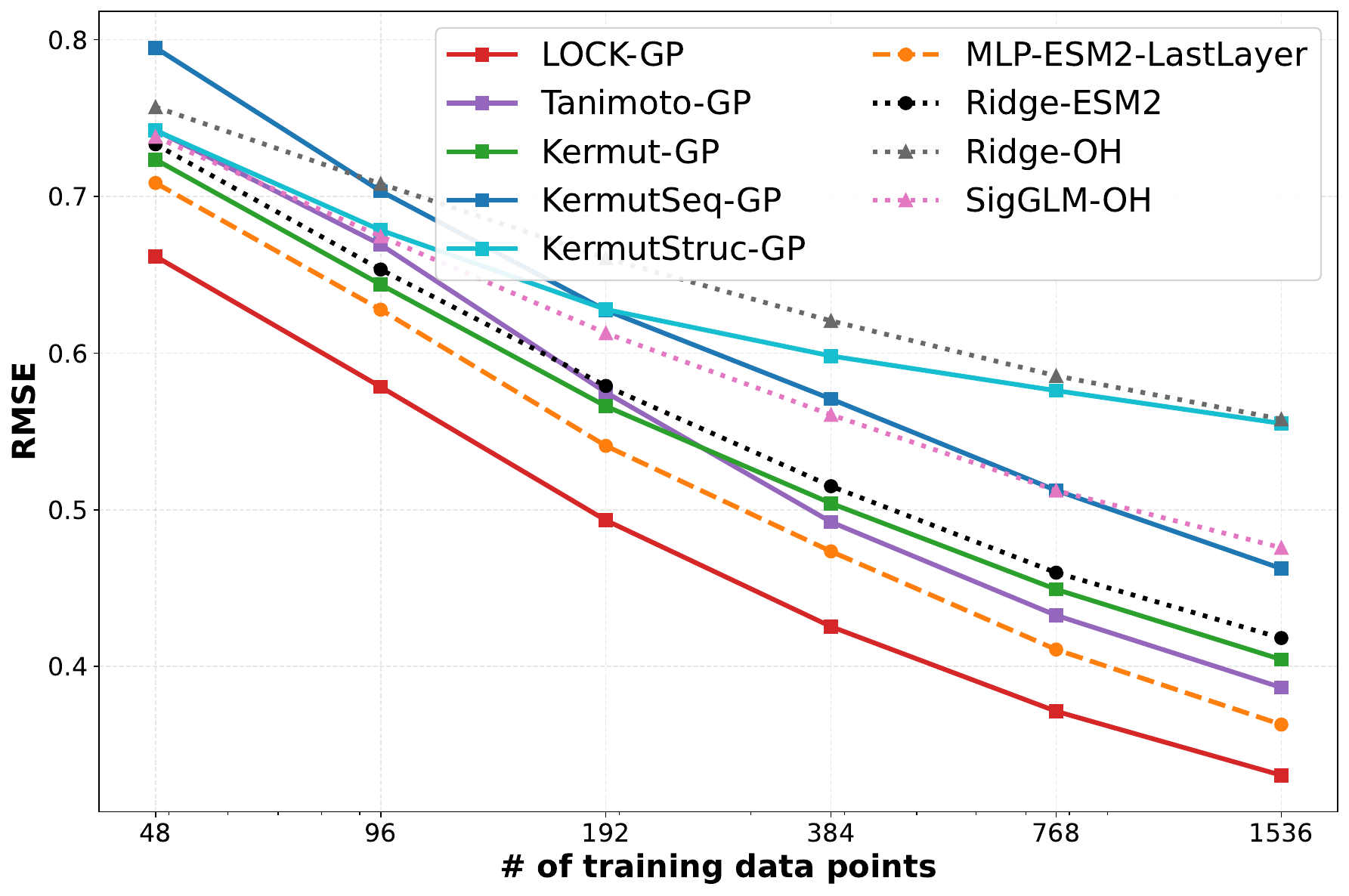}
    \caption{Predictive performance as a function of the number of training data points in the cross-validation setting.
    We depict Spearman R (left) and RMSE (right); metrics are averaged across $21$ datasets.
    See \secref{sec:local} for discussion and \secref{app:sec:localdetails} for details on each model.}
\label{fig:cvspearmanrmse}
\end{figure*}
\input{landscape_tables.tex}

We provide additional  results for the local learning experiment in \secref{sec:local}.
For a companion figure to \figref{fig:cvpearsonmae} that depicts Spearman R and RMSE metrics see \figref{fig:cvspearmanrmse}.
For a companion figure to \figref{fig:cvpearsonmae} that depicts CRPS for GP models see \figref{fig:cvcrps}.
Select results for individual landscapes are reported in Tables~\ref{tab:landscape_pearsonr_cv}-\ref{tab:landscape_nll_combined}.
For an extended version of Table \ref{tab:metric_summary} with additional models see Table \ref{tab:metric_summary_long}.
For uncertainty metrics for GP models see Table \ref{tab:metric_summary_gp}.
For average model ranks w.r.t.~three metrics see Table \ref{tab:ranks}.
\input{metric_and_rank_summaries.tex}


%% file: landscape_tables.tex
\begin{table*}
\input{landscape_table_cv-192_pearsonr.tex}
\caption{Pearson R metric for each landscape for $192$ training data points in the cross-validation regime.
    The best metric for each landscape is marked in bold.
    See \secref{sec:local} for discussion and \secref{app:sec:localdetails} for details on each model.}
\label{tab:landscape_pearsonr_cv}
\end{table*}
\begin{table*}
\input{landscape_table_cv-192_spearmanr.tex}
\caption{Spearman R metric for each landscape for $192$ training data points in the cross-validation regime.
    The best metric for each landscape is marked in bold.
    See \secref{sec:local} for discussion and \secref{app:sec:localdetails} for details on each model.}
\label{tab:landscape_spearmanr_cv}
\end{table*}
\begin{table*}
\input{landscape_table_cv-192_mae.tex}
\caption{MAE metric for each landscape for $192$ training data points in the cross-validation regime.
    The best metric for each landscape is marked in bold.
    See \secref{sec:local} for discussion and \secref{app:sec:localdetails} for details on each model.}
\label{tab:landscape_mae_cv}
\end{table*}
\begin{table*}
\input{landscape_table_unseen-96_spearmanr.tex}
\caption{Spearman R metric for each landscape for $96$ training data points in the unseen mutations regime.
    Metrics are averaged across three distinct train-test splits, and
    the best metric for each landscape is marked in bold.
    See \secref{sec:local} for discussion and \secref{app:sec:localdetails} for details on each model.}
\label{tab:landscape_spearmanr_unseen}
\end{table*}
\begin{table*}
\input{landscape_table_unseen-96_pearsonr.tex}
\caption{Pearson R metric for each landscape for $96$ training data points in the unseen mutations regime.
    Metrics are averaged across three distinct train-test splits, and
    the best metric for each landscape is marked in bold.
    See \secref{sec:local} for discussion and \secref{app:sec:localdetails} for details on each model.}
\label{tab:landscape_pearsonr_unseen}
\end{table*}

\begin{table*}
\input{landscape_table_unseen-96_mae.tex}
\caption{MAE metric for each landscape for $96$ training data points in the unseen mutations regime.
    Metrics are averaged across three distinct train-test splits, and
    the best metric for each landscape is marked in bold.
    See \secref{sec:local} for discussion and \secref{app:sec:localdetails} for details on each model.}
\label{tab:landscape_mae_unseen}
\end{table*}
\begin{table*}
\input{landscape_table_extrapolation-128_spearmanr.tex}
\caption{Spearman R metric for each landscape for $128$ training data points in the extrapolation regime.
    The best metric for each landscape is marked in bold.
    See \secref{sec:local} for discussion and \secref{app:sec:localdetails} for details on each model.}
\label{tab:landscape_spearmanr_extrapolation}
\end{table*}
\begin{table*}
\input{landscape_table_extrapolation-128_pearsonr.tex}
\caption{Pearson R metric for each landscape for $128$ training data points in the extrapolation regime.
    The best metric for each landscape is marked in bold.
    See \secref{sec:local} for discussion and \secref{app:sec:localdetails} for details on each model.}
\label{tab:landscape_pearsonr_extrapolation}
\end{table*}
\begin{table*}
\input{landscape_table_extrapolation-128_mae.tex}
\caption{MAE metric for each landscape for $128$ training data points in the extrapolation regime.
    The best metric for each landscape is marked in bold.
    See \secref{sec:local} for discussion and \secref{app:sec:localdetails} for details on each model.}
\label{tab:landscape_mae_extrapolation}
\end{table*}
\begin{table*}
    \input{landscape_table_nll_gp_combined.tex}
    \caption{We report the negative log likelihood (NLL; lower is better) for each GP variant on $21$ 
    landscapes in three different evaluation regimes.
    The best metric for each landscape is marked in bold, while NLLs larger than $10$ are marked red.}
    \label{tab:landscape_nll_combined}
\end{table*}

%% file: landscape_table_cv-192_pearsonr.tex
\centering
\resizebox{\linewidth}{!}{
\begin{tabular}{lccccccccccc}
\toprule
Landscape & \texttt{LOCK-GP} & \texttt{Tanimoto-GP} & \texttt{MLP-ESM2-LastLayer} & \texttt{Kermut-GP} & \texttt{KermutSeq-GP} & \texttt{KermutStruc-GP} & \texttt{Ridge-ESM2} & \texttt{Ridge-OH} & \texttt{SigGLM-OH} & \texttt{BLOSUM50-ZeroShot} & \texttt{ESM2-650M-ZeroShot} \\
\midrule
CAPSD\_AAV2S\_Sinai\_2021 & \textbf{0.748} & 0.699 & 0.663 & 0.695 & 0.633 & 0.610 & 0.655 & 0.581 & 0.585 & 0.297 & 0.280 \\
Q8WTC7\_9CNID\_Somermeyer\_2022 & \textbf{0.650} & 0.393 & 0.468 & 0.604 & 0.175 & 0.625 & 0.515 & 0.532 & 0.535 & 0.252 & -0.042 \\
CH65 - G189E & \textbf{0.948} & 0.936 & 0.932 & 0.886 & 0.878 & 0.841 & 0.890 & 0.842 & 0.880 & 0.543 & -0.394 \\
CH65 - MA90 & \textbf{0.969} & 0.967 & 0.963 & 0.947 & 0.943 & 0.857 & 0.950 & 0.856 & 0.872 & 0.532 & -0.443 \\
CH65 - SI06 & \textbf{0.952} & 0.926 & 0.945 & 0.881 & 0.876 & 0.854 & 0.882 & 0.854 & 0.892 & 0.478 & -0.359 \\
CR6261 - H1 & \textbf{0.975} & 0.957 & 0.960 & 0.970 & 0.968 & 0.864 & 0.948 & 0.855 & 0.932 & 0.665 & -0.738 \\
CR6261 - H9 & \textbf{0.957} & 0.926 & 0.928 & 0.953 & 0.951 & 0.885 & 0.917 & 0.884 & 0.899 & 0.683 & -0.687 \\
CR9114 - H1 & \textbf{0.949} & 0.908 & 0.933 & 0.835 & 0.853 & 0.788 & 0.898 & 0.786 & 0.884 & 0.428 & -0.338 \\
CR9114 - H3 & 0.877 & 0.710 & \textbf{0.901} & 0.533 & 0.602 & 0.494 & 0.702 & 0.499 & 0.782 & 0.399 & -0.179 \\
D7PM05\_CLYGR\_Somermeyer\_2022 & \textbf{0.695} & 0.458 & 0.472 & 0.637 & 0.360 & 0.621 & 0.496 & 0.426 & 0.431 & 0.523 & 0.054 \\
GCN4\_YEAST\_Staller\_2018 & 0.223 & 0.194 & 0.339 & 0.420 & \textbf{0.433} & 0.182 & 0.345 & 0.116 & 0.124 & 0.107 & 0.159 \\
GFP\_AEQVI\_Sarkisyan\_2016 & \textbf{0.738} & 0.617 & 0.597 & 0.725 & 0.536 & 0.692 & 0.630 & 0.502 & 0.502 & 0.619 & 0.114 \\
HIS7\_YEAST\_Pokusaeva\_2019 & \textbf{0.739} & 0.691 & 0.679 & 0.735 & 0.722 & 0.735 & 0.706 & 0.626 & 0.628 & 0.234 & 0.401 \\
PHOT\_CHLRE\_Chen\_2023 & \textbf{0.933} & 0.922 & 0.918 & 0.927 & 0.922 & 0.913 & 0.918 & 0.910 & 0.912 & 0.421 & 0.730 \\
F7YBW8\_MESOW\_Ding\_2023 & \textbf{0.814} & 0.803 & 0.789 & 0.775 & 0.777 & 0.752 & 0.768 & 0.758 & 0.755 & 0.486 & 0.364 \\
Q6WV12\_9MAXI\_Somermeyer\_2022 & \textbf{0.772} & 0.636 & 0.612 & 0.688 & 0.308 & 0.700 & 0.610 & 0.633 & 0.636 & 0.358 & 0.023 \\
Tsuboyama - 2B88 & 0.721 & 0.721 & 0.804 & 0.815 & 0.618 & 0.740 & \textbf{0.827} & 0.643 & 0.641 & 0.064 & 0.089 \\
Tsuboyama - 2HBB & 0.721 & 0.717 & 0.711 & \textbf{0.809} & 0.752 & 0.697 & 0.738 & 0.630 & 0.630 & 0.013 & 0.327 \\
Moulana CB6 & 0.989 & 0.986 & \textbf{0.991} & 0.981 & 0.957 & 0.980 & 0.982 & 0.980 & 0.987 & 0.268 & -0.206 \\
Moulana Cov555 & \textbf{0.965} & 0.921 & 0.961 & 0.861 & 0.816 & 0.854 & 0.918 & 0.854 & 0.938 & 0.401 & -0.123 \\
Moulana Regn10987 & \textbf{0.953} & 0.919 & 0.946 & 0.887 & 0.846 & 0.885 & 0.911 & 0.885 & 0.916 & 0.336 & 0.124 \\
\bottomrule
\end{tabular}
}

%% file: landscape_table_cv-192_spearmanr.tex
\centering
\resizebox{\linewidth}{!}{
\begin{tabular}{lccccccccccc}
\toprule
Landscape & \texttt{LOCK-GP} & \texttt{Tanimoto-GP} & \texttt{MLP-ESM2-LastLayer} & \texttt{Kermut-GP} & \texttt{KermutSeq-GP} & \texttt{KermutStruc-GP} & \texttt{Ridge-ESM2} & \texttt{Ridge-OH} & \texttt{SigGLM-OH} & \texttt{BLOSUM50-ZeroShot} & \texttt{ESM2-650M-ZeroShot} \\
\midrule
CAPSD\_AAV2S\_Sinai\_2021 & \textbf{0.754} & 0.718 & 0.701 & 0.719 & 0.640 & 0.644 & 0.686 & 0.610 & 0.612 & 0.178 & 0.247 \\
Q8WTC7\_9CNID\_Somermeyer\_2022 & \textbf{0.549} & 0.290 & 0.328 & 0.501 & 0.164 & 0.519 & 0.386 & 0.385 & 0.386 & 0.265 & -0.024 \\
CH65 - G189E & \textbf{0.961} & 0.952 & 0.958 & 0.917 & 0.908 & 0.877 & 0.928 & 0.878 & 0.876 & 0.558 & -0.404 \\
CH65 - MA90 & \textbf{0.966} & 0.965 & 0.965 & 0.952 & 0.947 & 0.857 & 0.959 & 0.856 & 0.856 & 0.531 & -0.438 \\
CH65 - SI06 & 0.881 & 0.877 & \textbf{0.885} & 0.874 & 0.870 & 0.855 & 0.878 & 0.854 & 0.856 & 0.498 & -0.359 \\
CR6261 - H1 & \textbf{0.967} & 0.938 & 0.949 & 0.947 & 0.941 & 0.912 & 0.921 & 0.917 & 0.916 & 0.680 & -0.724 \\
CR6261 - H9 & \textbf{0.971} & 0.958 & 0.967 & 0.965 & 0.963 & 0.924 & 0.951 & 0.927 & 0.926 & 0.710 & -0.704 \\
CR9114 - H1 & \textbf{0.913} & 0.870 & 0.896 & 0.833 & 0.833 & 0.846 & 0.854 & 0.843 & 0.846 & 0.532 & -0.458 \\
CR9114 - H3 & 0.527 & 0.512 & \textbf{0.529} & 0.488 & 0.472 & 0.491 & 0.510 & 0.493 & 0.516 & 0.396 & -0.197 \\
D7PM05\_CLYGR\_Somermeyer\_2022 & \textbf{0.699} & 0.447 & 0.456 & 0.689 & 0.364 & 0.672 & 0.496 & 0.425 & 0.428 & 0.547 & 0.052 \\
GCN4\_YEAST\_Staller\_2018 & 0.277 & 0.226 & 0.416 & 0.481 & \textbf{0.493} & 0.234 & 0.404 & 0.113 & 0.135 & 0.209 & 0.206 \\
GFP\_AEQVI\_Sarkisyan\_2016 & 0.708 & 0.603 & 0.579 & \textbf{0.727} & 0.523 & 0.695 & 0.614 & 0.477 & 0.478 & 0.604 & 0.094 \\
HIS7\_YEAST\_Pokusaeva\_2019 & 0.693 & 0.645 & 0.630 & \textbf{0.720} & 0.702 & 0.720 & 0.659 & 0.593 & 0.596 & 0.254 & 0.387 \\
PHOT\_CHLRE\_Chen\_2023 & \textbf{0.941} & 0.933 & 0.940 & 0.936 & 0.930 & 0.925 & 0.930 & 0.923 & 0.924 & 0.398 & 0.717 \\
F7YBW8\_MESOW\_Ding\_2023 & \textbf{0.728} & 0.726 & 0.707 & 0.715 & 0.715 & 0.701 & 0.717 & 0.703 & 0.700 & 0.582 & 0.423 \\
Q6WV12\_9MAXI\_Somermeyer\_2022 & \textbf{0.660} & 0.520 & 0.475 & 0.573 & 0.273 & 0.586 & 0.497 & 0.492 & 0.494 & 0.333 & 0.009 \\
Tsuboyama - 2B88 & 0.684 & 0.682 & 0.747 & 0.757 & 0.574 & 0.673 & \textbf{0.769} & 0.615 & 0.613 & 0.027 & 0.079 \\
Tsuboyama - 2HBB & 0.746 & 0.738 & 0.738 & \textbf{0.828} & 0.778 & 0.720 & 0.753 & 0.653 & 0.652 & -0.010 & 0.365 \\
Moulana CB6 & 0.918 & 0.915 & \textbf{0.920} & 0.914 & 0.881 & 0.915 & 0.912 & 0.915 & 0.917 & 0.287 & -0.210 \\
Moulana Cov555 & 0.857 & 0.832 & \textbf{0.858} & 0.830 & 0.785 & 0.826 & 0.827 & 0.828 & 0.834 & 0.420 & -0.087 \\
Moulana Regn10987 & 0.923 & 0.898 & \textbf{0.929} & 0.915 & 0.863 & 0.914 & 0.877 & 0.914 & 0.908 & 0.327 & 0.104 \\
\bottomrule
\end{tabular}
}

%% file: landscape_table_cv-192_mae.tex
\centering
\resizebox{\linewidth}{!}{
\begin{tabular}{lccccccccc}
\toprule
Landscape & \texttt{LOCK-GP} & \texttt{Tanimoto-GP} & \texttt{MLP-ESM2-LastLayer} & \texttt{Kermut-GP} & \texttt{KermutSeq-GP} & \texttt{KermutStruc-GP} & \texttt{Ridge-ESM2} & \texttt{Ridge-OH} & \texttt{SigGLM-OH} \\
\midrule
CAPSD\_AAV2S\_Sinai\_2021 & \textbf{0.528} & 0.579 & 0.611 & 0.573 & 0.618 & 0.644 & 0.612 & 0.672 & 0.661 \\
Q8WTC7\_9CNID\_Somermeyer\_2022 & \textbf{0.515} & 0.753 & 0.636 & 0.592 & 0.778 & 0.575 & 0.634 & 0.620 & 0.619 \\
CH65 - G189E & \textbf{0.220} & 0.252 & 0.254 & 0.354 & 0.363 & 0.429 & 0.359 & 0.429 & 0.352 \\
CH65 - MA90 & \textbf{0.193} & 0.202 & 0.218 & 0.256 & 0.265 & 0.432 & 0.248 & 0.432 & 0.411 \\
CH65 - SI06 & \textbf{0.206} & 0.282 & 0.240 & 0.364 & 0.383 & 0.407 & 0.373 & 0.405 & 0.332 \\
CR6261 - H1 & \textbf{0.138} & 0.207 & 0.197 & 0.167 & 0.175 & 0.410 & 0.236 & 0.424 & 0.269 \\
CR6261 - H9 & \textbf{0.196} & 0.285 & 0.260 & 0.208 & 0.220 & 0.372 & 0.315 & 0.377 & 0.335 \\
CR9114 - H1 & \textbf{0.202} & 0.305 & 0.247 & 0.407 & 0.370 & 0.481 & 0.331 & 0.484 & 0.342 \\
CR9114 - H3 & 0.220 & 0.418 & \textbf{0.197} & 0.504 & 0.414 & 0.538 & 0.463 & 0.532 & 0.387 \\
D7PM05\_CLYGR\_Somermeyer\_2022 & \textbf{0.568} & 0.911 & 0.761 & 0.643 & 0.853 & 0.653 & 0.738 & 0.794 & 0.792 \\
GCN4\_YEAST\_Staller\_2018 & 0.707 & 0.714 & 0.679 & 0.633 & \textbf{0.626} & 0.723 & 0.673 & 0.735 & 0.756 \\
GFP\_AEQVI\_Sarkisyan\_2016 & \textbf{0.529} & 0.696 & 0.662 & 0.549 & 0.723 & 0.578 & 0.624 & 0.738 & 0.735 \\
HIS7\_YEAST\_Pokusaeva\_2019 & \textbf{0.518} & 0.594 & 0.574 & 0.529 & 0.536 & 0.532 & 0.553 & 0.630 & 0.628 \\
PHOT\_CHLRE\_Chen\_2023 & \textbf{0.268} & 0.293 & 0.297 & 0.284 & 0.294 & 0.313 & 0.300 & 0.316 & 0.312 \\
F7YBW8\_MESOW\_Ding\_2023 & \textbf{0.435} & 0.453 & 0.455 & 0.472 & 0.476 & 0.498 & 0.490 & 0.495 & 0.500 \\
Q6WV12\_9MAXI\_Somermeyer\_2022 & \textbf{0.426} & 0.601 & 0.549 & 0.537 & 0.753 & 0.523 & 0.573 & 0.564 & 0.561 \\
Tsuboyama - 2B88 & 0.457 & 0.470 & 0.423 & 0.410 & 0.572 & 0.476 & \textbf{0.401} & 0.523 & 0.526 \\
Tsuboyama - 2HBB & 0.497 & 0.508 & 0.527 & \textbf{0.437} & 0.491 & 0.540 & 0.498 & 0.576 & 0.576 \\
Moulana CB6 & \textbf{0.067} & 0.102 & 0.076 & 0.140 & 0.217 & 0.145 & 0.132 & 0.144 & 0.103 \\
Moulana Cov555 & \textbf{0.174} & 0.311 & 0.214 & 0.428 & 0.471 & 0.443 & 0.316 & 0.445 & 0.289 \\
Moulana Regn10987 & \textbf{0.157} & 0.290 & 0.191 & 0.360 & 0.434 & 0.364 & 0.311 & 0.366 & 0.287 \\
\bottomrule
\end{tabular}
}

%% file: landscape_table_unseen-96_spearmanr.tex
\centering
\resizebox{\linewidth}{!}{
\begin{tabular}{lccccccccccc}
\toprule
Landscape & \texttt{LOCK-GP} & \texttt{Tanimoto-GP} & \texttt{MLP-ESM2-LastLayer} & \texttt{Kermut-GP} & \texttt{KermutSeq-GP} & \texttt{KermutStruc-GP} & \texttt{Ridge-ESM2} & \texttt{Ridge-OH} & \texttt{SigGLM-OH} & \texttt{BLOSUM50-ZeroShot} & \texttt{ESM2-650M-ZeroShot} \\
\midrule
CAPSD\_AAV2S\_Sinai\_2021 & \textbf{0.701} & 0.652 & 0.595 & 0.653 & 0.565 & 0.593 & 0.637 & 0.536 & 0.540 & 0.138 & 0.225 \\
Q8WTC7\_9CNID\_Somermeyer\_2022 & \textbf{0.404} & 0.183 & 0.187 & 0.225 & 0.037 & 0.226 & 0.230 & 0.192 & 0.203 & 0.262 & -0.033 \\
CH65 - G189E & \textbf{0.694} & 0.676 & 0.675 & 0.674 & 0.619 & 0.675 & 0.493 & 0.625 & 0.637 & 0.548 & -0.400 \\
CH65 - MA90 & 0.672 & 0.673 & 0.682 & 0.746 & \textbf{0.756} & 0.629 & 0.622 & 0.616 & 0.617 & 0.528 & -0.447 \\
CH65 - SI06 & 0.629 & 0.634 & 0.637 & \textbf{0.679} & 0.593 & 0.664 & 0.676 & 0.619 & 0.617 & 0.507 & -0.352 \\
CR6261 - H1 & 0.759 & 0.756 & 0.756 & 0.799 & \textbf{0.825} & 0.789 & 0.286 & 0.730 & 0.725 & 0.673 & -0.719 \\
CR6261 - H9 & 0.736 & 0.731 & 0.740 & 0.814 & \textbf{0.821} & 0.797 & 0.292 & 0.720 & 0.718 & 0.699 & -0.690 \\
CR9114 - H1 & 0.670 & 0.660 & 0.635 & 0.548 & 0.453 & \textbf{0.700} & 0.374 & 0.634 & 0.625 & 0.534 & -0.458 \\
CR9114 - H3 & 0.436 & 0.450 & \textbf{0.455} & 0.446 & 0.293 & 0.436 & 0.306 & 0.440 & 0.450 & 0.394 & -0.201 \\
D7PM05\_CLYGR\_Somermeyer\_2022 & \textbf{0.635} & 0.446 & 0.326 & 0.597 & 0.182 & 0.597 & 0.367 & 0.310 & 0.310 & 0.546 & 0.049 \\
GCN4\_YEAST\_Staller\_2018 & 0.207 & 0.178 & 0.297 & 0.266 & 0.261 & 0.178 & \textbf{0.316} & 0.119 & 0.119 & 0.156 & 0.166 \\
GFP\_AEQVI\_Sarkisyan\_2016 & 0.656 & 0.521 & 0.485 & 0.636 & 0.491 & \textbf{0.682} & 0.519 & 0.364 & 0.367 & 0.598 & 0.092 \\
HIS7\_YEAST\_Pokusaeva\_2019 & 0.451 & 0.252 & 0.350 & \textbf{0.633} & 0.613 & 0.607 & 0.389 & 0.227 & 0.231 & 0.374 & 0.375 \\
PHOT\_CHLRE\_Chen\_2023 & 0.413 & 0.404 & 0.311 & 0.643 & 0.606 & \textbf{0.695} & 0.296 & 0.280 & 0.260 & 0.244 & 0.673 \\
F7YBW8\_MESOW\_Ding\_2023 & 0.712 & 0.710 & \textbf{0.719} & 0.708 & 0.693 & 0.707 & 0.711 & 0.704 & 0.705 & 0.613 & 0.485 \\
Q6WV12\_9MAXI\_Somermeyer\_2022 & \textbf{0.534} & 0.397 & 0.320 & 0.468 & 0.236 & 0.476 & 0.379 & 0.390 & 0.389 & 0.335 & 0.010 \\
Tsuboyama - 2B88 & 0.502 & 0.423 & 0.637 & 0.608 & 0.477 & 0.438 & \textbf{0.655} & 0.434 & 0.436 & -0.021 & -0.078 \\
Tsuboyama - 2HBB & 0.566 & 0.492 & 0.481 & \textbf{0.706} & 0.644 & 0.610 & 0.501 & 0.432 & 0.436 & 0.008 & -0.351 \\
Moulana CB6 & 0.888 & 0.884 & 0.870 & 0.882 & 0.819 & \textbf{0.894} & 0.728 & 0.883 & 0.884 & -0.287 & 0.212 \\
Moulana Cov555 & 0.682 & 0.680 & 0.681 & 0.614 & 0.513 & 0.618 & 0.464 & 0.679 & \textbf{0.684} & -0.421 & 0.089 \\
Moulana Regn10987 & 0.854 & 0.843 & \textbf{0.887} & 0.874 & 0.682 & 0.879 & 0.756 & 0.870 & 0.867 & -0.322 & -0.107 \\
\bottomrule
\end{tabular}
}

%% file: landscape_table_unseen-96_pearsonr.tex
\centering
\resizebox{\linewidth}{!}{
\begin{tabular}{lccccccccccc}
\toprule
Landscape & \texttt{LOCK-GP} & \texttt{Tanimoto-GP} & \texttt{MLP-ESM2-LastLayer} & \texttt{Kermut-GP} & \texttt{KermutSeq-GP} & \texttt{KermutStruc-GP} & \texttt{Ridge-ESM2} & \texttt{Ridge-OH} & \texttt{SigGLM-OH} & \texttt{BLOSUM50-ZeroShot} & \texttt{ESM2-650M-ZeroShot} \\
\midrule
CAPSD\_AAV2S\_Sinai\_2021 & \textbf{0.689} & 0.633 & 0.554 & 0.617 & 0.560 & 0.546 & 0.603 & 0.505 & 0.508 & 0.267 & 0.257 \\
Q8WTC7\_9CNID\_Somermeyer\_2022 & \textbf{0.445} & 0.225 & 0.273 & 0.282 & 0.045 & 0.282 & 0.325 & 0.252 & 0.266 & 0.250 & -0.049 \\
CH65 - G189E & \textbf{0.656} & 0.655 & 0.649 & 0.648 & 0.593 & 0.651 & 0.479 & 0.606 & 0.636 & 0.535 & -0.392 \\
CH65 - MA90 & 0.681 & 0.683 & 0.689 & 0.753 & \textbf{0.758} & 0.640 & 0.639 & 0.628 & 0.636 & 0.530 & -0.451 \\
CH65 - SI06 & 0.629 & 0.632 & 0.624 & 0.665 & 0.573 & 0.656 & \textbf{0.677} & 0.611 & 0.613 & 0.492 & -0.353 \\
CR6261 - H1 & 0.687 & 0.707 & 0.709 & 0.795 & \textbf{0.842} & 0.769 & 0.297 & 0.688 & 0.717 & 0.660 & -0.738 \\
CR6261 - H9 & 0.694 & 0.684 & 0.698 & 0.786 & \textbf{0.802} & 0.770 & 0.291 & 0.689 & 0.692 & 0.674 & -0.676 \\
CR9114 - H1 & 0.637 & \textbf{0.637} & 0.600 & 0.516 & 0.430 & 0.633 & 0.352 & 0.588 & 0.635 & 0.429 & -0.339 \\
CR9114 - H3 & \textbf{0.656} & 0.550 & 0.602 & 0.426 & 0.374 & 0.412 & 0.333 & 0.420 & 0.572 & 0.400 & -0.183 \\
D7PM05\_CLYGR\_Somermeyer\_2022 & \textbf{0.636} & 0.439 & 0.343 & 0.573 & 0.192 & 0.573 & 0.388 & 0.320 & 0.321 & 0.523 & 0.052 \\
GCN4\_YEAST\_Staller\_2018 & 0.162 & 0.121 & 0.209 & 0.216 & 0.212 & 0.138 & \textbf{0.243} & 0.080 & 0.080 & 0.081 & 0.126 \\
GFP\_AEQVI\_Sarkisyan\_2016 & \textbf{0.686} & 0.534 & 0.506 & 0.639 & 0.503 & 0.675 & 0.535 & 0.394 & 0.397 & 0.609 & 0.109 \\
HIS7\_YEAST\_Pokusaeva\_2019 & 0.471 & 0.295 & 0.381 & \textbf{0.584} & 0.565 & 0.572 & 0.404 & 0.265 & 0.267 & 0.349 & 0.368 \\
PHOT\_CHLRE\_Chen\_2023 & 0.291 & 0.269 & 0.295 & 0.658 & 0.622 & 0.676 & 0.284 & 0.237 & 0.241 & 0.232 & \textbf{0.685} \\
F7YBW8\_MESOW\_Ding\_2023 & 0.814 & 0.816 & \textbf{0.817} & 0.815 & 0.781 & 0.802 & 0.789 & 0.792 & 0.793 & 0.587 & 0.458 \\
Q6WV12\_9MAXI\_Somermeyer\_2022 & \textbf{0.641} & 0.504 & 0.436 & 0.553 & 0.263 & 0.577 & 0.484 & 0.507 & 0.507 & 0.361 & 0.024 \\
Tsuboyama - 2B88 & 0.508 & 0.368 & 0.663 & 0.641 & 0.484 & 0.484 & \textbf{0.692} & 0.404 & 0.406 & -0.061 & -0.086 \\
Tsuboyama - 2HBB & 0.517 & 0.463 & 0.449 & \textbf{0.658} & 0.606 & 0.567 & 0.480 & 0.398 & 0.401 & -0.014 & -0.315 \\
Moulana CB6 & \textbf{0.984} & 0.981 & 0.979 & 0.965 & 0.840 & 0.972 & 0.792 & 0.977 & 0.980 & -0.269 & 0.206 \\
Moulana Cov555 & 0.701 & 0.699 & 0.693 & 0.614 & 0.526 & 0.618 & 0.465 & 0.677 & \textbf{0.705} & -0.401 & 0.126 \\
Moulana Regn10987 & \textbf{0.877} & 0.874 & 0.873 & 0.862 & 0.706 & 0.861 & 0.739 & 0.859 & 0.870 & -0.332 & -0.129 \\
\bottomrule
\end{tabular}
}

%% file: landscape_table_unseen-96_mae.tex
\centering
\resizebox{\linewidth}{!}{
\begin{tabular}{lccccccccc}
\toprule
Landscape & \texttt{LOCK-GP} & \texttt{Tanimoto-GP} & \texttt{MLP-ESM2-LastLayer} & \texttt{Kermut-GP} & \texttt{KermutSeq-GP} & \texttt{KermutStruc-GP} & \texttt{Ridge-ESM2} & \texttt{Ridge-OH} & \texttt{SigGLM-OH} \\
\midrule
CAPSD\_AAV2S\_Sinai\_2021 & \textbf{0.618} & 0.670 & 0.738 & 0.661 & 0.689 & 0.743 & 0.670 & 0.756 & 0.759 \\
Q8WTC7\_9CNID\_Somermeyer\_2022 & \textbf{0.647} & 0.771 & 0.704 & 0.706 & 0.775 & 0.705 & 0.704 & 0.741 & 0.733 \\
CH65 - G189E & 0.666 & 0.676 & 0.672 & 0.672 & 0.699 & \textbf{0.655} & 0.870 & 0.727 & 0.693 \\
CH65 - MA90 & 0.684 & 0.688 & 0.671 & \textbf{0.584} & 0.666 & 0.714 & 0.813 & 0.737 & 0.731 \\
CH65 - SI06 & \textbf{0.577} & 0.593 & 0.591 & 0.670 & 0.881 & 0.688 & 0.835 & 0.620 & 0.598 \\
CR6261 - H1 & 0.656 & 0.650 & 0.633 & 0.543 & \textbf{0.453} & 0.615 & 10.02 & 0.708 & 0.662 \\
CR6261 - H9 & 0.695 & 0.705 & 0.684 & 0.632 & \textbf{0.545} & 0.676 & 8.580 & 0.731 & 0.711 \\
CR9114 - H1 & \textbf{0.474} & 0.502 & 0.502 & 0.590 & 0.607 & 0.587 & 0.886 & 0.591 & 0.525 \\
CR9114 - H3 & \textbf{0.392} & 0.520 & 0.435 & 0.638 & 0.549 & 0.648 & 1.386 & 0.647 & 0.558 \\
D7PM05\_CLYGR\_Somermeyer\_2022 & \textbf{0.636} & 0.971 & 0.842 & 0.704 & 0.937 & 0.699 & 0.815 & 0.881 & 0.873 \\
GCN4\_YEAST\_Staller\_2018 & 0.729 & 0.738 & 0.777 & 0.725 & 0.729 & 0.781 & \textbf{0.710} & 0.766 & 0.762 \\
GFP\_AEQVI\_Sarkisyan\_2016 & \textbf{0.599} & 0.957 & 0.739 & 0.633 & 0.760 & 0.608 & 0.710 & 0.850 & 0.848 \\
HIS7\_YEAST\_Pokusaeva\_2019 & 0.776 & 0.901 & 0.830 & \textbf{0.687} & 0.715 & 0.702 & 0.809 & 0.900 & 0.899 \\
PHOT\_CHLRE\_Chen\_2023 & 0.918 & 0.932 & 0.943 & 0.822 & 0.774 & \textbf{0.710} & 0.911 & 0.904 & 0.874 \\
F7YBW8\_MESOW\_Ding\_2023 & 0.515 & 0.515 & \textbf{0.504} & 0.506 & 0.528 & 0.535 & 0.552 & 0.539 & 0.538 \\
Q6WV12\_9MAXI\_Somermeyer\_2022 & \textbf{0.565} & 0.820 & 0.650 & 0.632 & 0.757 & 0.613 & 0.655 & 0.639 & 0.642 \\
Tsuboyama - 2B88 & 0.599 & 0.666 & 0.517 & 0.544 & 0.621 & 0.651 & \textbf{0.508} & 0.655 & 0.658 \\
Tsuboyama - 2HBB & 0.665 & 0.735 & 0.740 & \textbf{0.600} & 0.618 & 0.644 & 0.731 & 0.823 & 0.818 \\
Moulana CB6 & \textbf{0.111} & 0.197 & 0.155 & 0.261 & 0.542 & 0.220 & 0.842 & 0.154 & 0.135 \\
Moulana Cov555 & 0.521 & 0.587 & \textbf{0.518} & 0.800 & 0.854 & 0.784 & 1.397 & 0.613 & 0.541 \\
Moulana Regn10987 & 0.374 & 0.384 & \textbf{0.295} & 0.344 & 0.603 & 0.331 & 0.770 & 0.332 & 0.321 \\
\bottomrule
\end{tabular}
}

%% file: landscape_table_extrapolation-128_spearmanr.tex
\centering
\resizebox{\linewidth}{!}{
\begin{tabular}{lccccccccccc}
\toprule
Landscape & \texttt{LOCK-GP} & \texttt{Tanimoto-GP} & \texttt{MLP-ESM2-LastLayer} & \texttt{Kermut-GP} & \texttt{KermutSeq-GP} & \texttt{KermutStruc-GP} & \texttt{Ridge-ESM2} & \texttt{Ridge-OH} & \texttt{SigGLM-OH} & \texttt{BLOSUM50-ZeroShot} & \texttt{ESM2-650M-ZeroShot} \\
\midrule
CAPSD\_AAV2S\_Sinai\_2021 & \textbf{0.710} & 0.518 & 0.399 & 0.357 & 0.600 & 0.318 & 0.499 & 0.212 & 0.218 & 0.444 & 0.342 \\
Q8WTC7\_9CNID\_Somermeyer\_2022 & \textbf{0.444} & 0.303 & 0.279 & 0.349 & 0.100 & 0.348 & 0.330 & 0.352 & 0.351 & 0.224 & -0.043 \\
CH65 - G189E & \textbf{0.857} & 0.846 & 0.849 & 0.802 & 0.689 & 0.801 & 0.836 & 0.792 & 0.777 & 0.548 & -0.391 \\
CH65 - MA90 & \textbf{0.896} & 0.879 & 0.895 & 0.806 & 0.737 & 0.797 & 0.893 & 0.795 & 0.782 & 0.520 & -0.425 \\
CH65 - SI06 & 0.823 & 0.822 & 0.791 & 0.810 & 0.759 & 0.810 & \textbf{0.849} & 0.808 & 0.804 & 0.485 & -0.345 \\
CR6261 - H1 & \textbf{0.947} & 0.919 & 0.922 & 0.928 & 0.924 & 0.896 & 0.889 & 0.877 & 0.894 & 0.596 & -0.705 \\
CR6261 - H9 & 0.918 & 0.918 & \textbf{0.923} & 0.917 & 0.910 & 0.874 & 0.902 & 0.882 & 0.880 & 0.564 & -0.594 \\
CR9114 - H1 & 0.702 & 0.627 & 0.659 & 0.730 & 0.486 & \textbf{0.730} & 0.624 & 0.664 & 0.641 & 0.524 & -0.451 \\
CR9114 - H3 & 0.517 & 0.518 & \textbf{0.524} & 0.520 & 0.424 & 0.520 & 0.512 & 0.520 & 0.523 & 0.382 & -0.184 \\
D7PM05\_CLYGR\_Somermeyer\_2022 & \textbf{0.443} & 0.314 & 0.217 & 0.422 & 0.031 & 0.422 & 0.250 & 0.271 & 0.274 & 0.332 & 0.031 \\
GCN4\_YEAST\_Staller\_2018 & 0.189 & 0.255 & 0.286 & 0.463 & \textbf{0.463} & 0.177 & 0.331 & 0.194 & 0.194 & 0.123 & 0.160 \\
GFP\_AEQVI\_Sarkisyan\_2016 & \textbf{0.516} & 0.332 & 0.346 & 0.294 & 0.139 & 0.474 & 0.384 & 0.295 & 0.295 & 0.435 & 0.047 \\
HIS7\_YEAST\_Pokusaeva\_2019 & 0.394 & 0.343 & 0.431 & \textbf{0.545} & 0.487 & 0.540 & 0.472 & 0.277 & 0.279 & 0.244 & 0.382 \\
PHOT\_CHLRE\_Chen\_2023 & 0.807 & \textbf{0.845} & 0.845 & 0.780 & 0.721 & 0.801 & 0.839 & 0.844 & 0.830 & 0.380 & 0.712 \\
F7YBW8\_MESOW\_Ding\_2023 & 0.594 & 0.564 & 0.577 & \textbf{0.629} & 0.594 & 0.561 & 0.521 & 0.523 & 0.548 & 0.517 & 0.347 \\
Q6WV12\_9MAXI\_Somermeyer\_2022 & \textbf{0.614} & 0.505 & 0.425 & 0.422 & 0.094 & 0.422 & 0.505 & 0.497 & 0.494 & 0.252 & 0.010 \\
Tsuboyama - 2B88 & 0.361 & 0.592 & \textbf{0.679} & 0.283 & 0.386 & 0.458 & 0.638 & 0.560 & 0.566 & 0.162 & -0.304 \\
Tsuboyama - 2HBB & 0.652 & 0.605 & 0.473 & \textbf{0.731} & 0.608 & 0.533 & 0.441 & 0.377 & 0.375 & -0.099 & -0.371 \\
Moulana CB6 & \textbf{0.917} & 0.915 & 0.904 & 0.913 & 0.850 & 0.913 & 0.886 & 0.914 & 0.913 & -0.269 & 0.229 \\
Moulana Cov555 & 0.826 & 0.804 & \textbf{0.831} & 0.812 & 0.349 & 0.812 & 0.826 & 0.810 & 0.812 & -0.367 & 0.137 \\
Moulana Regn10987 & 0.916 & 0.842 & \textbf{0.928} & 0.904 & 0.535 & 0.904 & 0.785 & 0.903 & 0.871 & -0.315 & -0.093 \\
\bottomrule
\end{tabular}
}

%% file: landscape_table_extrapolation-128_pearsonr.tex
\centering
\resizebox{\linewidth}{!}{
\begin{tabular}{lccccccccccc}
\toprule
Landscape & \texttt{LOCK-GP} & \texttt{Tanimoto-GP} & \texttt{MLP-ESM2-LastLayer} & \texttt{Kermut-GP} & \texttt{KermutSeq-GP} & \texttt{KermutStruc-GP} & \texttt{Ridge-ESM2} & \texttt{Ridge-OH} & \texttt{SigGLM-OH} & \texttt{BLOSUM50-ZeroShot} & \texttt{ESM2-650M-ZeroShot} \\
\midrule
CAPSD\_AAV2S\_Sinai\_2021 & \textbf{0.711} & 0.527 & 0.297 & 0.358 & 0.588 & 0.310 & 0.495 & 0.175 & 0.176 & 0.455 & 0.363 \\
Q8WTC7\_9CNID\_Somermeyer\_2022 & \textbf{0.495} & 0.341 & 0.342 & 0.409 & 0.111 & 0.415 & 0.390 & 0.405 & 0.404 & 0.210 & -0.052 \\
CH65 - G189E & \textbf{0.821} & 0.802 & 0.816 & 0.751 & 0.636 & 0.750 & 0.779 & 0.741 & 0.744 & 0.537 & -0.384 \\
CH65 - MA90 & \textbf{0.887} & 0.867 & 0.885 & 0.799 & 0.726 & 0.794 & 0.884 & 0.793 & 0.787 & 0.521 & -0.431 \\
CH65 - SI06 & 0.846 & 0.849 & 0.786 & 0.825 & 0.770 & 0.825 & \textbf{0.852} & 0.824 & 0.828 & 0.461 & -0.343 \\
CR6261 - H1 & \textbf{0.941} & 0.925 & 0.922 & 0.936 & 0.920 & 0.890 & 0.909 & 0.871 & 0.903 & 0.613 & -0.714 \\
CR6261 - H9 & 0.873 & 0.855 & 0.837 & \textbf{0.875} & 0.874 & 0.840 & 0.852 & 0.846 & 0.831 & 0.550 & -0.579 \\
CR9114 - H1 & \textbf{0.697} & 0.620 & 0.526 & 0.610 & 0.238 & 0.610 & 0.599 & 0.534 & 0.530 & 0.425 & -0.333 \\
CR9114 - H3 & \textbf{0.838} & 0.683 & 0.614 & 0.527 & 0.472 & 0.527 & 0.623 & 0.530 & 0.772 & 0.370 & -0.163 \\
D7PM05\_CLYGR\_Somermeyer\_2022 & \textbf{0.502} & 0.324 & 0.232 & 0.460 & 0.052 & 0.460 & 0.266 & 0.285 & 0.288 & 0.365 & 0.052 \\
GCN4\_YEAST\_Staller\_2018 & 0.175 & 0.230 & 0.200 & \textbf{0.424} & 0.424 & 0.162 & 0.292 & 0.183 & 0.182 & 0.068 & 0.124 \\
GFP\_AEQVI\_Sarkisyan\_2016 & \textbf{0.557} & 0.349 & 0.345 & 0.333 & 0.159 & 0.507 & 0.425 & 0.335 & 0.335 & 0.465 & 0.063 \\
HIS7\_YEAST\_Pokusaeva\_2019 & 0.449 & 0.399 & 0.405 & 0.552 & 0.487 & \textbf{0.564} & 0.526 & 0.352 & 0.356 & 0.224 & 0.397 \\
PHOT\_CHLRE\_Chen\_2023 & 0.805 & \textbf{0.848} & 0.830 & 0.786 & 0.730 & 0.807 & 0.841 & 0.847 & 0.835 & 0.390 & 0.722 \\
F7YBW8\_MESOW\_Ding\_2023 & 0.708 & 0.679 & 0.658 & \textbf{0.758} & 0.635 & 0.707 & 0.652 & 0.664 & 0.666 & 0.409 & 0.280 \\
Q6WV12\_9MAXI\_Somermeyer\_2022 & \textbf{0.624} & 0.520 & 0.464 & 0.470 & 0.119 & 0.470 & 0.539 & 0.522 & 0.521 & 0.266 & 0.011 \\
Tsuboyama - 2B88 & 0.598 & 0.667 & \textbf{0.714} & 0.482 & 0.489 & 0.639 & 0.690 & 0.603 & 0.605 & 0.080 & -0.259 \\
Tsuboyama - 2HBB & 0.630 & 0.598 & 0.499 & \textbf{0.704} & 0.569 & 0.521 & 0.479 & 0.400 & 0.398 & -0.106 & -0.353 \\
Moulana CB6 & \textbf{0.989} & 0.982 & 0.982 & 0.978 & 0.856 & 0.978 & 0.962 & 0.978 & 0.982 & -0.249 & 0.223 \\
Moulana Cov555 & \textbf{0.902} & 0.848 & 0.819 & 0.851 & 0.366 & 0.851 & 0.847 & 0.850 & 0.873 & -0.342 & 0.172 \\
Moulana Regn10987 & \textbf{0.893} & 0.824 & 0.821 & 0.842 & 0.384 & 0.842 & 0.818 & 0.838 & 0.788 & -0.326 & -0.113 \\
\bottomrule
\end{tabular}
}

%% file: landscape_table_extrapolation-128_mae.tex
\centering
\resizebox{\linewidth}{!}{
\begin{tabular}{lccccccccc}
\toprule
Landscape & \texttt{LOCK-GP} & \texttt{Tanimoto-GP} & \texttt{MLP-ESM2-LastLayer} & \texttt{Kermut-GP} & \texttt{KermutSeq-GP} & \texttt{KermutStruc-GP} & \texttt{Ridge-ESM2} & \texttt{Ridge-OH} & \texttt{SigGLM-OH} \\
\midrule
CAPSD\_AAV2S\_Sinai\_2021 & \textbf{0.582} & 0.927 & 1.338 & 1.112 & 0.680 & 1.248 & 0.970 & 1.488 & 1.541 \\
Q8WTC7\_9CNID\_Somermeyer\_2022 & 0.813 & 0.893 & 0.829 & 0.817 & 0.894 & 0.836 & 0.822 & \textbf{0.812} & 0.814 \\
CH65 - G189E & \textbf{0.627} & 0.734 & 0.738 & 0.700 & 0.763 & 0.701 & 0.704 & 0.722 & 0.776 \\
CH65 - MA90 & 0.689 & 0.748 & 0.671 & 0.594 & 0.724 & 0.683 & \textbf{0.537} & 0.707 & 0.753 \\
CH65 - SI06 & 0.559 & 0.489 & 0.724 & 0.535 & 0.500 & 0.535 & 0.454 & 0.516 & \textbf{0.448} \\
CR6261 - H1 & \textbf{0.242} & 0.364 & 0.319 & 0.330 & 0.382 & 0.593 & 0.330 & 0.627 & 0.395 \\
CR6261 - H9 & \textbf{0.364} & 0.431 & 0.429 & 0.369 & 0.387 & 0.509 & 0.414 & 0.538 & 0.447 \\
CR9114 - H1 & \textbf{0.461} & 0.567 & 0.576 & 0.598 & 0.691 & 0.598 & 0.525 & 0.669 & 0.664 \\
CR9114 - H3 & \textbf{0.506} & 1.893 & 1.853 & 3.369 & 1.606 & 3.369 & 2.996 & 4.045 & 1.464 \\
D7PM05\_CLYGR\_Somermeyer\_2022 & \textbf{0.624} & 1.283 & 1.208 & 0.775 & 1.281 & 0.775 & 1.132 & 1.165 & 1.159 \\
GCN4\_YEAST\_Staller\_2018 & 0.726 & 0.720 & 0.894 & 0.638 & \textbf{0.638} & 0.719 & 0.744 & 0.715 & 0.714 \\
GFP\_AEQVI\_Sarkisyan\_2016 & \textbf{0.761} & 1.243 & 1.110 & 1.071 & 1.152 & 0.914 & 0.991 & 1.155 & 1.153 \\
HIS7\_YEAST\_Pokusaeva\_2019 & 0.724 & 0.790 & 0.819 & \textbf{0.684} & 0.719 & 0.691 & 0.779 & 0.886 & 0.889 \\
PHOT\_CHLRE\_Chen\_2023 & 0.788 & \textbf{0.426} & 0.444 & 1.423 & 0.926 & 0.885 & 0.479 & 0.451 & 0.460 \\
F7YBW8\_MESOW\_Ding\_2023 & 0.811 & 0.905 & 0.900 & \textbf{0.600} & 1.042 & 0.825 & 0.994 & 0.970 & 0.911 \\
Q6WV12\_9MAXI\_Somermeyer\_2022 & \textbf{0.730} & 1.039 & 0.874 & 0.890 & 1.041 & 0.890 & 0.830 & 0.865 & 0.865 \\
Tsuboyama - 2B88 & 1.046 & \textbf{0.528} & 0.548 & 1.262 & 0.668 & 0.889 & 0.679 & 0.630 & 0.605 \\
Tsuboyama - 2HBB & \textbf{0.632} & 0.733 & 0.724 & 0.647 & 0.719 & 0.807 & 0.859 & 1.444 & 1.471 \\
Moulana CB6 & \textbf{0.091} & 0.211 & 0.233 & 0.199 & 0.581 & 0.199 & 0.214 & 0.195 & 0.156 \\
Moulana Cov555 & \textbf{0.349} & 0.513 & 0.492 & 0.446 & 1.328 & 0.446 & 0.457 & 0.446 & 0.403 \\
Moulana Regn10987 & \textbf{0.303} & 0.573 & 0.460 & 0.679 & 0.833 & 0.679 & 0.487 & 0.682 & 0.716 \\
\bottomrule
\end{tabular}
}

%% file: landscape_table_nll_gp_combined.tex
\centering
\resizebox{\linewidth}{!}{
\begin{tabular}{lccccccccccccccc}
\toprule
Landscape & \multicolumn{5}{c}{\textbf{Cross-validation} (192 data points)} & \multicolumn{5}{c}{\textbf{Unseen mutations} (96 data points)} & \multicolumn{5}{c}{\textbf{Extrapolation} (128 data points)} \\
 & \texttt{LOCK-GP} & \texttt{Tanimoto-GP} & \texttt{Kermut-GP} & \texttt{KermutSeq-GP} & \texttt{KermutStruc-GP} & \texttt{LOCK-GP} & \texttt{Tanimoto-GP} & \texttt{Kermut-GP} & \texttt{KermutSeq-GP} & \texttt{KermutStruc-GP} & \texttt{LOCK-GP} & \texttt{Tanimoto-GP} & \texttt{Kermut-GP} & \texttt{KermutSeq-GP} & \texttt{KermutStruc-GP} \\
\cmidrule(lr){2-6} \cmidrule(lr){7-11} \cmidrule(lr){12-16}
CAPSD\_AAV2S\_Sinai\_2021 & \textbf{1.05} & 1.12 & 1.11 & 1.20 & 1.19 & \textbf{1.19} & 1.24 & 1.23 & 1.29 & 1.37 & \textbf{1.19} & 1.52 & 1.82 & 1.27 & 2.24 \\
Q8WTC7\_9CNID\_Somermeyer\_2022 & \textbf{1.14} & 1.36 & 1.23 & 1.41 & 1.17 & \textbf{1.35} & 1.46 & 1.42 & \textcolor{red}{924} & 1.42 & \textbf{1.42} & 1.70 & 1.52 & 1.73 & 1.54 \\
CH65 - G189E & \textbf{0.298} & 0.369 & 0.666 & 0.699 & 0.814 & 3.13 & \textbf{2.50} & 3.29 & 3.39 & 2.88 & \textbf{7.83} & \textcolor{red}{10.1} & \textcolor{red}{29.4} & \textcolor{red}{17.0} & \textcolor{red}{29.5} \\
CH65 - MA90 & \textbf{0.011} & 0.059 & 0.287 & 0.322 & 0.758 & \textbf{1.39} & 1.54 & 1.41 & 2.08 & 2.07 & 3.50 & 4.29 & \textbf{2.95} & 5.45 & 8.75 \\
CH65 - SI06 & \textbf{0.265} & 0.446 & 0.676 & 0.688 & 0.770 & \textbf{2.93} & 3.11 & 6.02 & 4.58 & 6.01 & 1.52 & 1.05 & 3.70 & \textbf{1.04} & 3.70 \\
CR6261 - H1 & \textbf{-0.123} & 0.163 & -0.054 & 0.001 & 0.736 & 3.41 & 1.74 & \textbf{1.28} & 1.42 & 1.58 & \textbf{0.440} & 0.682 & 0.550 & 0.756 & 1.73 \\
CR6261 - H9 & \textbf{0.145} & 0.431 & 0.193 & 0.231 & 0.659 & 4.36 & 4.09 & 3.16 & 2.88 & \textbf{2.70} & \textbf{0.860} & 1.16 & 1.10 & 1.26 & 1.50 \\
CR9114 - H1 & \textbf{0.259} & 0.534 & 0.828 & 0.771 & 0.938 & \textbf{\textcolor{red}{19.4}} & \textcolor{red}{20.2} & \textcolor{red}{20.7} & \textcolor{red}{20.7} & \textcolor{red}{24.0} & 6.43 & \textbf{3.71} & 9.69 & 9.42 & 9.69 \\
CR9114 - H3 & \textbf{0.724} & 1.10 & 1.29 & 1.24 & 1.32 & \textbf{1.29} & 1.33 & 1.37 & 1.39 & 1.40 & \textbf{1.39} & 2.16 & 6.96 & 2.07 & 6.96 \\
D7PM05\_CLYGR\_Somermeyer\_2022 & \textbf{1.13} & 1.37 & 1.16 & 1.35 & 1.21 & \textbf{1.18} & 1.42 & 1.26 & 1.43 & 1.26 & \textbf{1.17} & 1.99 & 1.39 & \textcolor{red}{$3.3 \times 10^{3}$} & 1.39 \\
GCN4\_YEAST\_Staller\_2018 & 1.42 & 1.41 & 1.36 & \textbf{1.33} & 1.44 & 1.47 & 1.46 & \textbf{1.43} & 1.43 & 1.50 & 1.40 & 1.38 & 1.31 & \textbf{1.31} & 1.41 \\
GFP\_AEQVI\_Sarkisyan\_2016 & \textbf{1.03} & 1.19 & 1.05 & 1.24 & 1.14 & \textbf{1.11} & 1.41 & 1.19 & 1.28 & 1.18 & \textbf{1.59} & 2.67 & 1.94 & 2.42 & 1.64 \\
HIS7\_YEAST\_Pokusaeva\_2019 & \textbf{1.02} & 1.10 & 1.05 & 1.07 & 1.02 & 1.35 & 1.44 & \textbf{1.27} & 1.29 & 1.29 & 1.89 & 2.14 & \textbf{1.37} & 1.53 & 1.67 \\
PHOT\_CHLRE\_Chen\_2023 & \textbf{0.354} & 0.439 & 0.410 & 0.445 & 0.508 & 2.22 & 2.27 & 2.01 & 2.09 & \textbf{1.82} & 1.36 & \textbf{0.982} & 3.34 & 1.52 & 3.09 \\
F7YBW8\_MESOW\_Ding\_2023 & \textbf{0.923} & 0.935 & 0.971 & 0.990 & 1.01 & 1.09 & \textbf{1.07} & 1.12 & 1.15 & 1.13 & 1.42 & 1.56 & \textbf{1.18} & 1.70 & 1.51 \\
Q6WV12\_9MAXI\_Somermeyer\_2022 & \textbf{0.885} & 1.17 & 1.08 & 1.35 & 1.07 & \textbf{1.12} & 1.41 & 1.29 & 1.45 & 1.23 & \textbf{1.40} & 1.88 & 1.66 & 1.81 & 1.66 \\
Tsuboyama - 2B88 & 0.935 & 0.982 & \textbf{0.876} & 1.16 & 1.04 & \textbf{1.38} & 1.47 & 1.40 & 1.48 & 1.70 & 1.70 & \textbf{1.21} & 1.91 & 1.34 & 1.54 \\
Tsuboyama - 2HBB & 0.934 & 0.987 & \textbf{0.874} & 0.979 & 1.10 & 1.28 & 1.38 & \textbf{1.20} & 1.21 & 1.40 & \textbf{1.28} & 1.40 & 1.31 & 1.34 & 1.53 \\
Moulana CB6 & -0.204 & \textbf{-0.347} & -0.192 & 0.154 & -0.172 & \textbf{0.078} & 0.222 & 0.669 & 1.16 & 0.613 & \textbf{-0.413} & 0.069 & 2.33 & 1.17 & 2.33 \\
Moulana Cov555 & \textbf{0.062} & 0.466 & 0.746 & 0.890 & 0.766 & \textbf{2.23} & 2.54 & 3.18 & 2.81 & 3.23 & \textbf{0.691} & 1.15 & 1.50 & 2.83 & 1.50 \\
Moulana Regn10987 & \textbf{0.232} & 0.487 & 0.883 & 0.805 & 0.658 & \textbf{0.828} & 0.973 & 2.37 & 1.15 & 2.86 & 2.97 & \textbf{1.68} & 4.23 & 3.58 & 4.23 \\
\bottomrule
\end{tabular}
}

%% file: metric_and_rank_summaries.tex
\begin{table*}
\input{metric_summary_table_long.tex}
    \caption{Model performance metrics for three different evaluation regimes,
    with the number of training data points ranging from $48$ to $1536$.
    Metrics are averaged across $21$ datasets. For each column the best performing metric is marked in bold.
    This table is identical to Table \ref{tab:metric_summary}, except that it contains additional models (marked in \textcolor{purp}{purple}).
    We do not include MAE metrics for zero-shot methods, since they are wildly off-scale. 
    Note that \texttt{Ridge-ESM2-8M} is referred to as \texttt{Ridge-ESM2} in the main text. 
    We find that \texttt{ESM2-650M-ZeroShot} exhibits significant variability depending in the landscape; 
    see Tables~\ref{tab:landscape_pearsonr_cv}-\ref{tab:landscape_mae_extrapolation}.
    See \secref{sec:local} for discussion and \secref{app:sec:localdetails} for details on each model.
    }
\label{tab:metric_summary_long}
\end{table*}

\begin{table*}
\input{metric_summary_table_gp.tex}
    \caption{Uncertainty metrics for GP models for three different evaluation regimes.
    For each column the best performing metric is marked in bold.
    LOCK exhibits the best NLL (negative log likelihood) and CRPS (continuous ranked probability score; \citet{gneiting2007strictly}) across the board,
    while the picture for ECE (expected calibration error) is more mixed.
    In any case most of the GPs exhibit good (average) calibration, with most ECE values $\lesssim 0.1$, except for the extrapolation regime, in
    which some of the average ECEs for the three Kermut variants are $\sim 0.15$.
    See \secref{sec:local} for discussion and \secref{app:sec:localdetails} for details on each model.
    }
\label{tab:metric_summary_gp}
\end{table*}

\begin{table*}
\input{rank_table.tex}
    \caption{We report average performance ranks w.r.t.~three metrics for three different evaluation regimes and averaged across $21$ landscapes.
    Ranks range from $1$ (for the top performing model) to $9$ (for the worst performing model). 
    So for example a rank of $1.5$ would suggest that a model typically ranks first or second w.r.t.~the corresponding metric.
    Similarly a rank of $8.9$ would mean that the model almost always ranks last among the $9$ models w.r.t.~the corresponding metric.
    See \secref{sec:local} for discussion and \secref{app:sec:localdetails} for details on each model.
    }
\label{tab:ranks}
\end{table*}

%% file: metric_summary_table_long.tex
\centering
\resizebox{\linewidth}{!}{
\begin{tabular}{lccccccccccccccc}
\toprule
 & \multicolumn{3}{c}{\textbf{Cross-validation} (48 data points)} & \multicolumn{3}{c}{\textbf{Cross-validation} (1536 data points)} & \multicolumn{3}{c}{\textbf{Unseen mutations} (96 data points)} & \multicolumn{3}{c}{\textbf{Extrapolation} (128 data points)} & \multicolumn{3}{c}{\textbf{Extrapolation} (512 data points)} \\
 & Spearman & Pearson & MAE & Spearman & Pearson & MAE & Spearman & Pearson & MAE & Spearman & Pearson & MAE & Spearman & Pearson & MAE \\
\cmidrule(lr){2-4} \cmidrule(lr){5-7} \cmidrule(lr){8-10} \cmidrule(lr){11-13} \cmidrule(lr){14-16}
\texttt{LOCK-GP} & \textbf{0.655} & \textbf{0.682} & \textbf{0.496} & \textbf{0.867} & \textbf{0.914} & \textbf{0.210} & 0.610 & 0.622 & \textbf{0.591} & \textbf{0.669} & \textbf{0.711} & \textbf{0.592} & \textbf{0.759} & \textbf{0.807} & \textbf{0.439} \\
\texttt{Tanimoto-GP} & 0.520 & 0.517 & 0.588 & 0.846 & 0.888 & 0.272 & 0.555 & 0.560 & 0.675 & 0.632 & 0.654 & 0.762 & 0.739 & 0.769 & 0.574 \\
\texttt{Kermut-GP} & 0.638 & 0.639 & 0.547 & 0.850 & 0.888 & 0.285 & \textbf{0.629} & \textbf{0.632} & 0.617 & 0.639 & 0.654 & 0.845 & 0.750 & 0.767 & 0.670 \\
\texttt{KermutSeq-GP} & 0.541 & 0.539 & 0.615 & 0.794 & 0.838 & 0.338 & 0.532 & 0.537 & 0.681 & 0.518 & 0.505 & 0.836 & 0.605 & 0.614 & 0.714 \\
\texttt{KermutStruc-GP} & 0.622 & 0.625 & 0.564 & 0.809 & 0.805 & 0.417 & 0.614 & 0.613 & 0.634 & 0.624 & 0.641 & 0.847 & 0.706 & 0.713 & 0.847 \\
\midrule
\textcolor{purp}{\texttt{ConFit}} & 0.628 & 0.615 & 0.574 & 0.863 & 0.867 & 0.344 & 0.548 & 0.545 & 0.726 & 0.632 & 0.622 & 0.882 & 0.733 & 0.723 & 0.828 \\
\texttt{MLP-ESM2-LastLayer} & 0.607 & 0.627 & 0.529 & 0.851 & 0.900 & 0.243 & 0.558 & 0.573 & 0.626 & 0.628 & 0.619 & 0.771 & 0.738 & 0.742 & 0.618 \\
\midrule
\texttt{MLP-ESM2-8M} & 0.597 & 0.618 & 0.538 & 0.856 & 0.898 & 0.256 & 0.530 & 0.551 & 0.678 & 0.621 & 0.648 & 0.711 & 0.739 & 0.763 & 0.562 \\
\textcolor{purp}{\texttt{MLP-ESM2-650M}} & 0.552 & 0.563 & 0.652 & 0.828 & 0.873 & 0.288 & 0.496 & 0.511 & 0.760 & 0.568 & 0.576 & 0.767 & 0.677 & 0.702 & 0.640 \\
\textcolor{purp}{\texttt{MLP-ESM2-8M-MeanPool}} & 0.561 & 0.579 & 0.594 & 0.797 & 0.843 & 0.362 & 0.501 & 0.510 & 0.725 & 0.536 & 0.553 & 0.814 & 0.658 & 0.673 & 0.679 \\
\textcolor{purp}{\texttt{MLP-OH}} & 0.553 & 0.580 & 0.565 & 0.764 & 0.808 & 0.459 & 0.506 & 0.527 & 0.668 & 0.577 & 0.596 & 0.863 & 0.653 & 0.680 & 0.810 \\
\textcolor{purp}{\texttt{MLP-OH-ESM2-650M-Aug}} & 0.581 & 0.605 & 0.548 & 0.771 & 0.818 & 0.442 & 0.567 & 0.589 & 0.620 & 0.619 & 0.632 & 0.789 & 0.681 & 0.702 & 0.737 \\
\textcolor{purp}{\texttt{MLP-ESM2-8M-RandInit}} & 0.558 & 0.575 & 0.572 & 0.839 & 0.880 & 0.295 & 0.496 & 0.517 & 0.705 & 0.581 & 0.609 & 0.747 & 0.709 & 0.727 & 0.606 \\
\textcolor{purp}{\texttt{MLP-SaProt-LastLayer}} & 0.608 & 0.629 & 0.536 & 0.851 & 0.903 & 0.238 & 0.568 & 0.583 & 0.633 & 0.644 & 0.640 & 0.789 & 0.751 & 0.759 & 0.609 \\
\midrule
\texttt{Ridge-ESM2-8M} & 0.606 & 0.617 & 0.558 & 0.837 & 0.877 & 0.301 & 0.476 & 0.490 & 1.628 & 0.629 & 0.653 & 0.781 & 0.719 & 0.740 & 0.685 \\
\textcolor{purp}{\texttt{Ridge-ESM2-650M}} & 0.625 & 0.630 & 0.552 & 0.836 & 0.877 & 0.304 & 0.506 & 0.523 & 1.007 & 0.623 & 0.649 & 0.864 & 0.710 & 0.729 & 0.685 \\
\textcolor{purp}{\texttt{Ridge-OH-ESM2-650M-Aug}} & 0.568 & 0.570 & 0.582 & 0.804 & 0.805 & 0.419 & 0.549 & 0.552 & 0.716 & 0.620 & 0.624 & 0.909 & 0.703 & 0.705 & 0.864 \\
\textcolor{purp}{\texttt{Ridge-ESM2-8M-MeanPool}} & 0.552 & 0.550 & 0.603 & 0.786 & 0.829 & 0.376 & 0.361 & 0.370 & 1.617 & 0.552 & 0.564 & 0.886 & 0.654 & 0.658 & 0.762 \\
\midrule
\texttt{Ridge-OH} & 0.547 & 0.553 & 0.590 & 0.802 & 0.800 & 0.422 & 0.514 & 0.519 & 0.682 & 0.589 & 0.599 & 0.939 & 0.694 & 0.698 & 0.877 \\
\texttt{SigGLM-OH} & 0.550 & 0.569 & 0.570 & 0.804 & 0.849 & 0.342 & 0.515 & 0.535 & 0.661 & 0.587 & 0.610 & 0.800 & 0.697 & 0.724 & 0.669 \\
\midrule
\textcolor{purp}{\texttt{BLOSUM50-ZeroShot}} & 0.396 & 0.386 & -- & 0.396 & 0.386 & -- & 0.291 & 0.281 & -- & 0.270 & 0.256 & -- & 0.270 & 0.256 & -- \\
\textcolor{purp}{\texttt{ESM2-650M-ZeroShot}} & -0.044 & -0.040 & -- & -0.044 & -0.040 & -- & -0.070 & -0.062 & -- & -0.072 & -0.063 & -- & -0.072 & -0.063 & -- \\
\bottomrule
\end{tabular}
}

%% file: metric_summary_table_gp.tex
\centering
\resizebox{\linewidth}{!}{
\begin{tabular}{lccccccccccccccc}
\toprule
 & \multicolumn{3}{c}{\textbf{Cross-validation} (48 data points)} & \multicolumn{3}{c}{\textbf{Cross-validation} (1536 data points)} & \multicolumn{3}{c}{\textbf{Unseen mutations} (96 data points)} & \multicolumn{3}{c}{\textbf{Extrapolation} (128 data points)} & \multicolumn{3}{c}{\textbf{Extrapolation} (512 data points)} \\
 & NLL & ECE & CRPS & NLL & ECE & CRPS & NLL & ECE & CRPS & NLL & ECE & CRPS & NLL & ECE & CRPS \\
\cmidrule(lr){2-4} \cmidrule(lr){5-7} \cmidrule(lr){8-10} \cmidrule(lr){11-13} \cmidrule(lr){14-16}
\texttt{LOCK-GP} & \textbf{0.987} & 0.060 & \textbf{0.360} & \textbf{0.112} & 0.108 & \textbf{0.160} & \textbf{2.56} & 0.113 & \textbf{0.456} & \textbf{1.95} & 0.120 & \textbf{0.442} & \textbf{1.26} & 0.097 & \textbf{0.333} \\
\texttt{Tanimoto-GP} & 1.13 & 0.079 & 0.413 & 0.281 & 0.068 & 0.201 & 2.59 & 0.125 & 0.499 & 2.12 & \textbf{0.120} & 0.573 & 1.55 & \textbf{0.087} & 0.424 \\
\texttt{Kermut-GP} & 19.8 & 0.048 & 0.397 & 0.358 & 0.066 & 0.212 & 2.77 & \textbf{0.097} & 0.465 & 3.87 & 0.159 & 0.657 & 2.30 & 0.108 & 0.498 \\
\texttt{KermutSeq-GP} & 48.1 & 0.056 & 0.440 & 0.476 & 0.061 & 0.248 & 46.7 & 0.109 & 0.507 & 159.8 & 0.138 & 0.635 & 2.98 & 0.122 & 0.540 \\
\texttt{KermutStruc-GP} & 1.80 & \textbf{0.047} & 0.409 & 0.763 & \textbf{0.040} & 0.300 & 2.98 & 0.104 & 0.478 & 4.24 & 0.159 & 0.665 & 4.81 & 0.167 & 0.670 \\
\bottomrule
\end{tabular}
}

%% file: rank_table.tex
\centering
\resizebox{\linewidth}{!}{
\begin{tabular}{lccccccccccccccc}
\toprule
 & \multicolumn{3}{c}{\textbf{Cross-validation} (48 data points)} & \multicolumn{3}{c}{\textbf{Cross-validation} (1536 data points)} & \multicolumn{3}{c}{\textbf{Unseen mutations} (96 data points)} & \multicolumn{3}{c}{\textbf{Extrapolation} (128 data points)} & \multicolumn{3}{c}{\textbf{Extrapolation} (512 data points)} \\
 & Spearman & Pearson & MAE & Spearman & Pearson & MAE & Spearman & Pearson & MAE & Spearman & Pearson & MAE & Spearman & Pearson & MAE \\
\cmidrule(lr){2-4} \cmidrule(lr){5-7} \cmidrule(lr){8-10} \cmidrule(lr){11-13} \cmidrule(lr){14-16}
\texttt{LOCK-GP} & \textbf{2.08} & \textbf{1.95} & \textbf{1.92} & \textbf{1.71} & \textbf{1.67} & \textbf{1.33} & \textbf{3.24} & \textbf{2.95} & \textbf{2.81} & \textbf{2.86} & \textbf{2.38} & \textbf{2.62} & \textbf{2.62} & \textbf{2.43} & \textbf{2.00} \\
\texttt{Tanimoto-GP} & 5.89 & 5.75 & 5.81 & 3.76 & 3.67 & 3.48 & 4.90 & 4.71 & 5.90 & 4.48 & 3.86 & 5.48 & 4.38 & 3.71 & 3.43 \\
\texttt{Kermut-GP} & 4.05 & 4.19 & 4.33 & 4.00 & 3.90 & 4.10 & 3.29 & 3.43 & 3.48 & 4.33 & 4.29 & 4.24 & 4.29 & 3.57 & 4.43 \\
\texttt{KermutSeq-GP} & 6.65 & 6.27 & 6.06 & 6.00 & 5.86 & 6.05 & 6.05 & 6.00 & 5.38 & 6.81 & 7.29 & 6.19 & 7.24 & 6.81 & 6.43 \\
\texttt{KermutStruc-GP} & 4.73 & 5.06 & 5.67 & 6.86 & 7.52 & 7.76 & 3.57 & 4.05 & 4.38 & 5.19 & 4.86 & 5.19 & 5.43 & 6.19 & 5.76 \\
\texttt{MLP-ESM2-LastLayer} & 3.87 & 4.24 & 3.49 & 3.24 & 2.71 & 2.81 & 4.48 & 4.52 & 4.05 & 4.05 & 5.62 & 5.19 & 3.48 & 4.43 & 4.76 \\
\texttt{Ridge-ESM2} & 5.16 & 5.03 & 5.22 & 5.05 & 4.90 & 4.90 & 5.90 & 6.05 & 6.57 & 5.19 & 4.67 & 4.57 & 5.10 & 5.00 & 5.33 \\
\texttt{Ridge-OH} & 6.54 & 6.90 & 6.90 & 7.48 & 8.33 & 8.19 & 7.10 & 7.38 & 6.71 & 5.95 & 6.14 & 6.00 & 6.33 & 7.00 & 7.00 \\
\texttt{SigGLM-OH} & 6.03 & 5.60 & 5.59 & 6.90 & 6.43 & 6.38 & 6.48 & 5.90 & 5.71 & 6.14 & 5.90 & 5.52 & 6.14 & 5.86 & 5.86 \\
\bottomrule
\end{tabular}
}

%% file: app_add_multi.tex
For additional figures for the multi-task experiment in \secref{sec:multi} see 
\figref{fig:app:tsuboperf}-\ref{fig:app:landscape_kernel_scale}.
In particular see \figref{fig:app:tsuboperf} for Pearson R and RMSE metrics.
See \figref{fig:app:blosum-corr-full}-\ref{fig:app:local-clock-corr-quantiles} for visualizations of correlation matrices.
See \figref{fig:app:clock_gp_land}-\ref{fig:app:landscape_pearson_mae} for additional figures pertaining to the landscape subsampling experiment.
For additional quantitative performance metrics see Table~\ref{tab:tsuboperf}.

\input{tsubo_table.tex}

\begin{figure*}[t]
\centering
\includegraphics[width=0.525\textwidth]{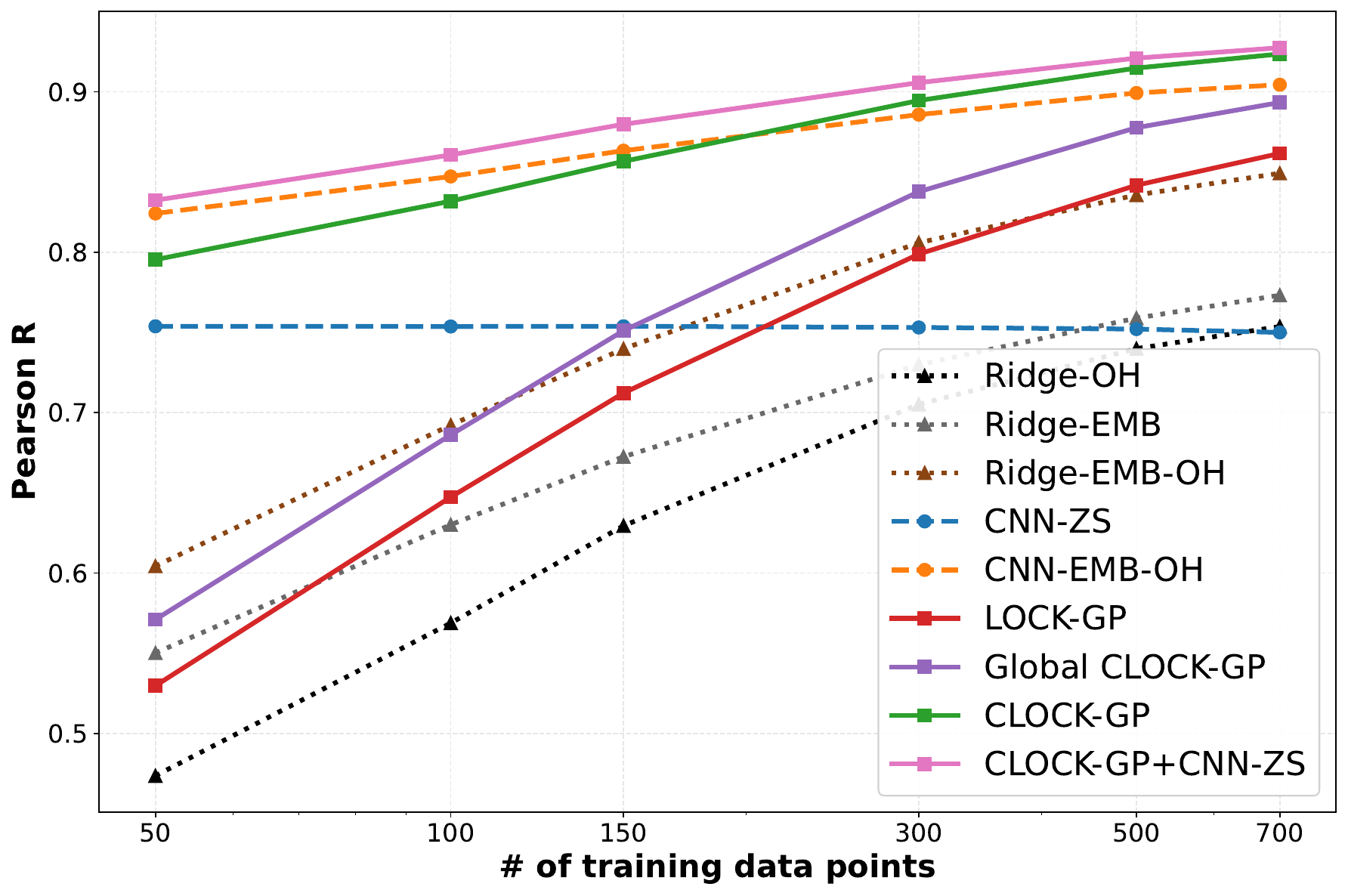}
\includegraphics[width=0.435\textwidth]{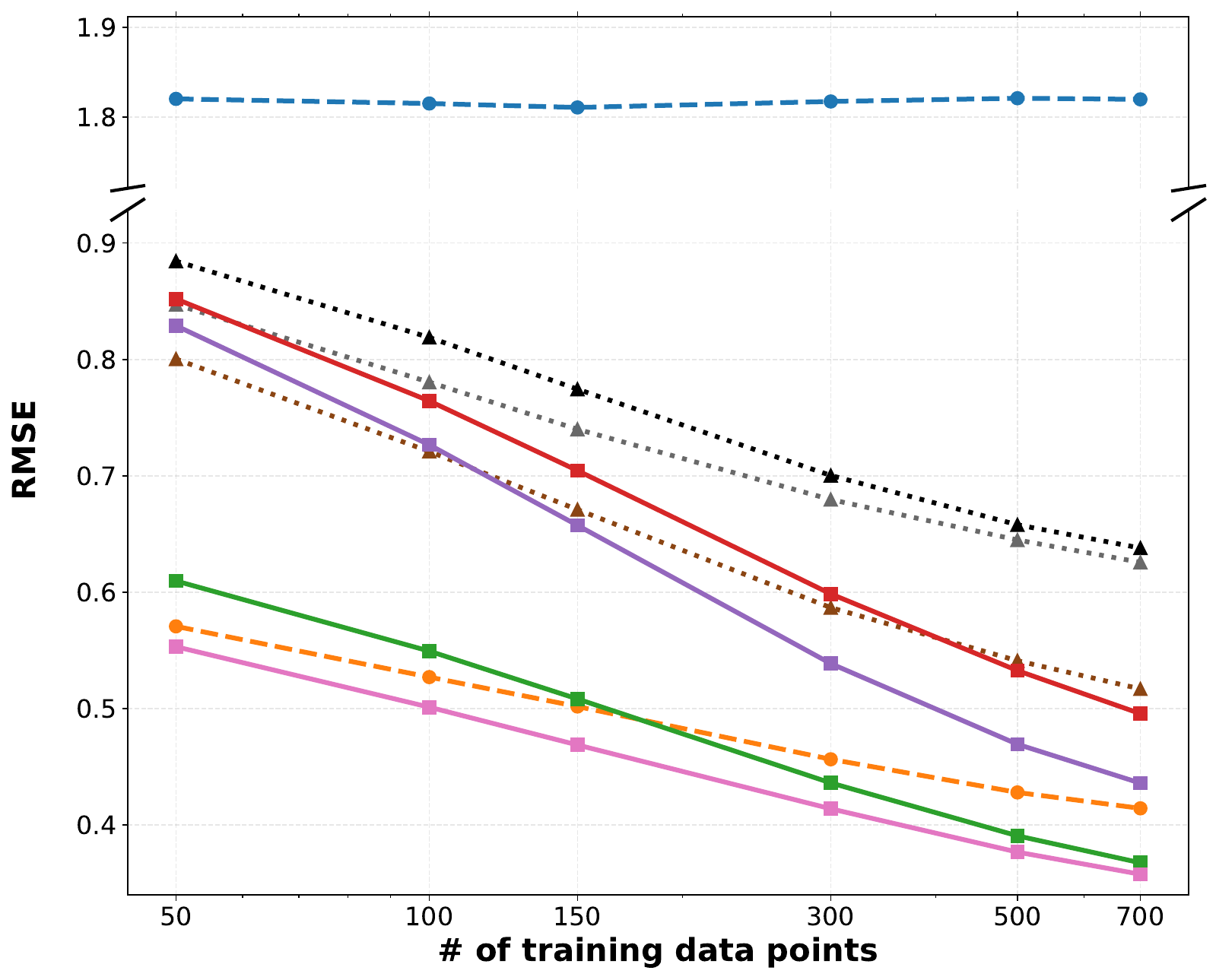}
    \caption{\emph{[Companion figure to \figref{fig:tsuboperf}]} 
    Predictive performance as a function of number of training data on the multi-landscape thermostability
    data described in \secref{sec:multi}. We depict both Pearson R (left) and root mean squared error (right);
    metrics are averaged across 100 train/test splits in 50 held-out landscapes.
    We also include \texttt{Ridge-EMB} and \texttt{Ridge-EMB-OH} models not depicted in the main text;
    see \secref{app:sec:multitaskdetails} for details. 
    The structure-conditioned CLOCK-GP augmented with an additional CNN-based kernel performs best across the board.
    Note the axis break for the RMSE figure: the absolute scale of the zero-shot predictor is very poorly calibrated.
    See Table~\ref{tab:tsuboperf} for more quantitative results.}
\label{fig:app:tsuboperf}
\end{figure*}

\begin{figure}[t]
  \centering
  \includegraphics[width=0.6 \columnwidth]{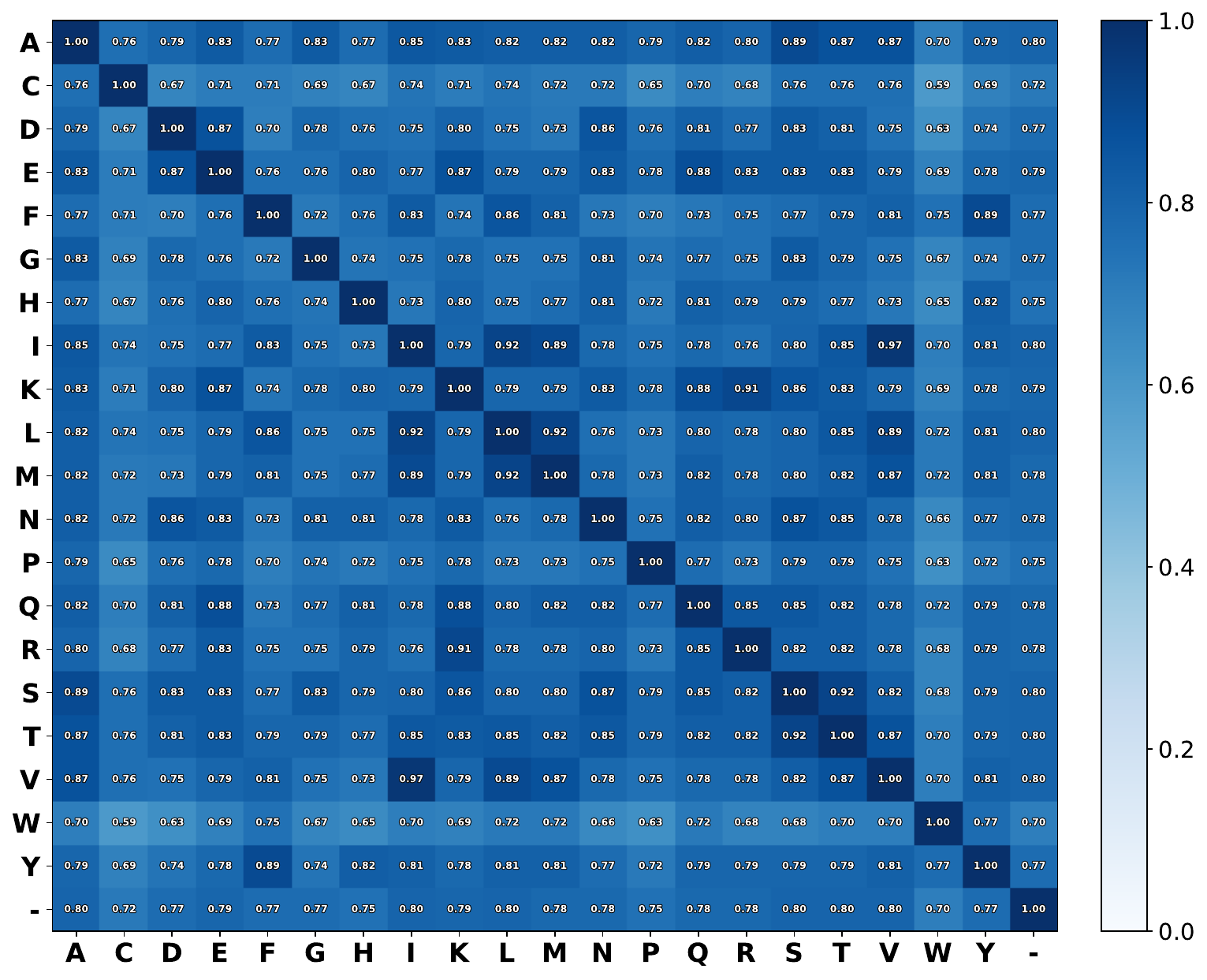}
  \caption{The BLOSUM50 substitution matrix normalized as a $21 \times 21$ correlation matrix.
    See \secref{app:norm} for details on the normalization scheme used.
    We include the gap token `-'. The correlation ranges from $0.59$ to $1.0$. 
    }
  \label{fig:app:blosum-corr-full}
\end{figure}

\begin{figure}[t]
  \centering
  \includegraphics[width=0.6 \columnwidth]{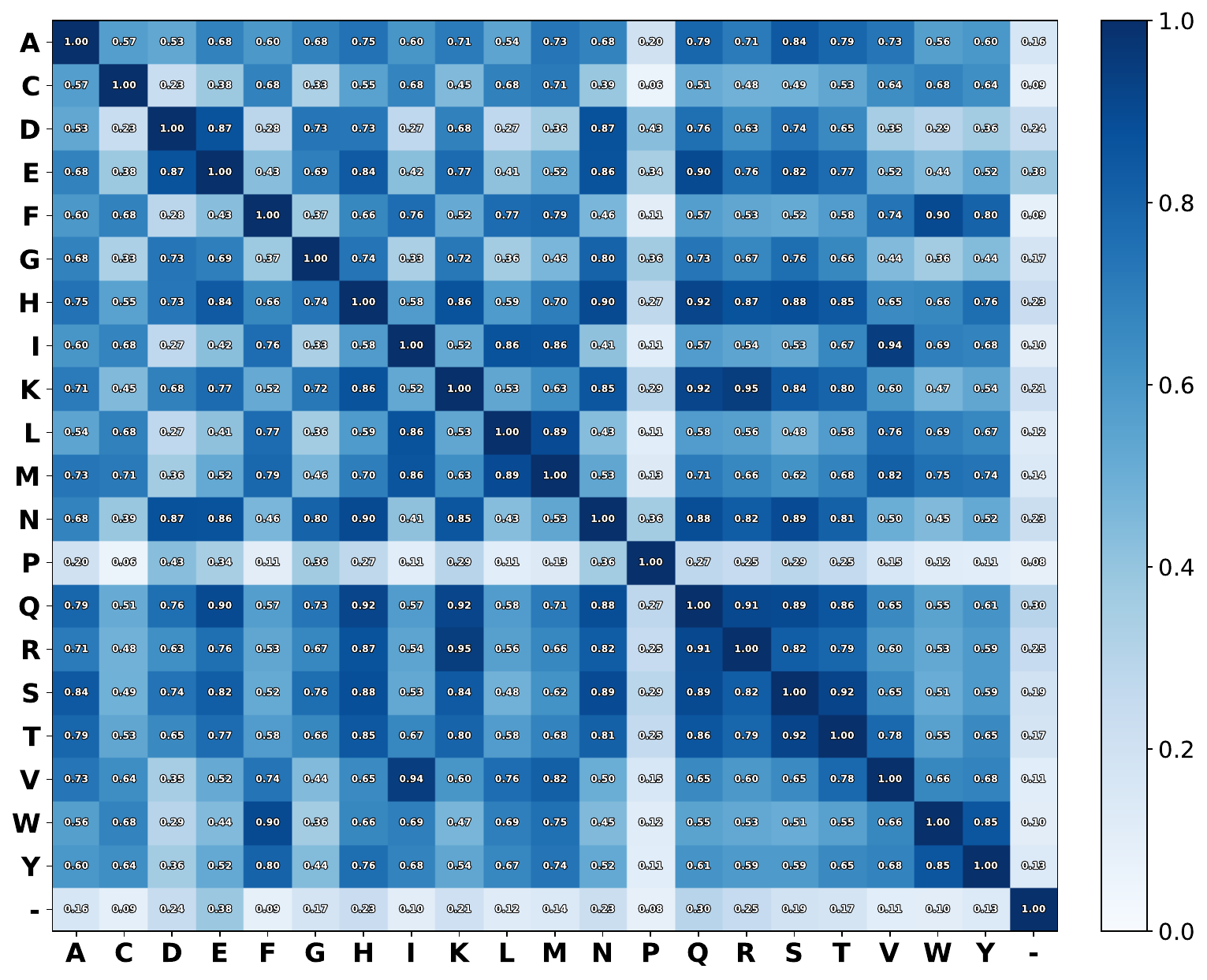}
  \caption{The global correlation matrix fit to the thermostability data described in \secref{sec:multi}.
    We include the gap token `-'. The correlation ranges from $0.06$ to $1.0$.
    }
  \label{fig:app:global-clock-corr}
\end{figure}

\begin{figure}[t]
  \centering
  \includegraphics[width=0.6 \columnwidth]{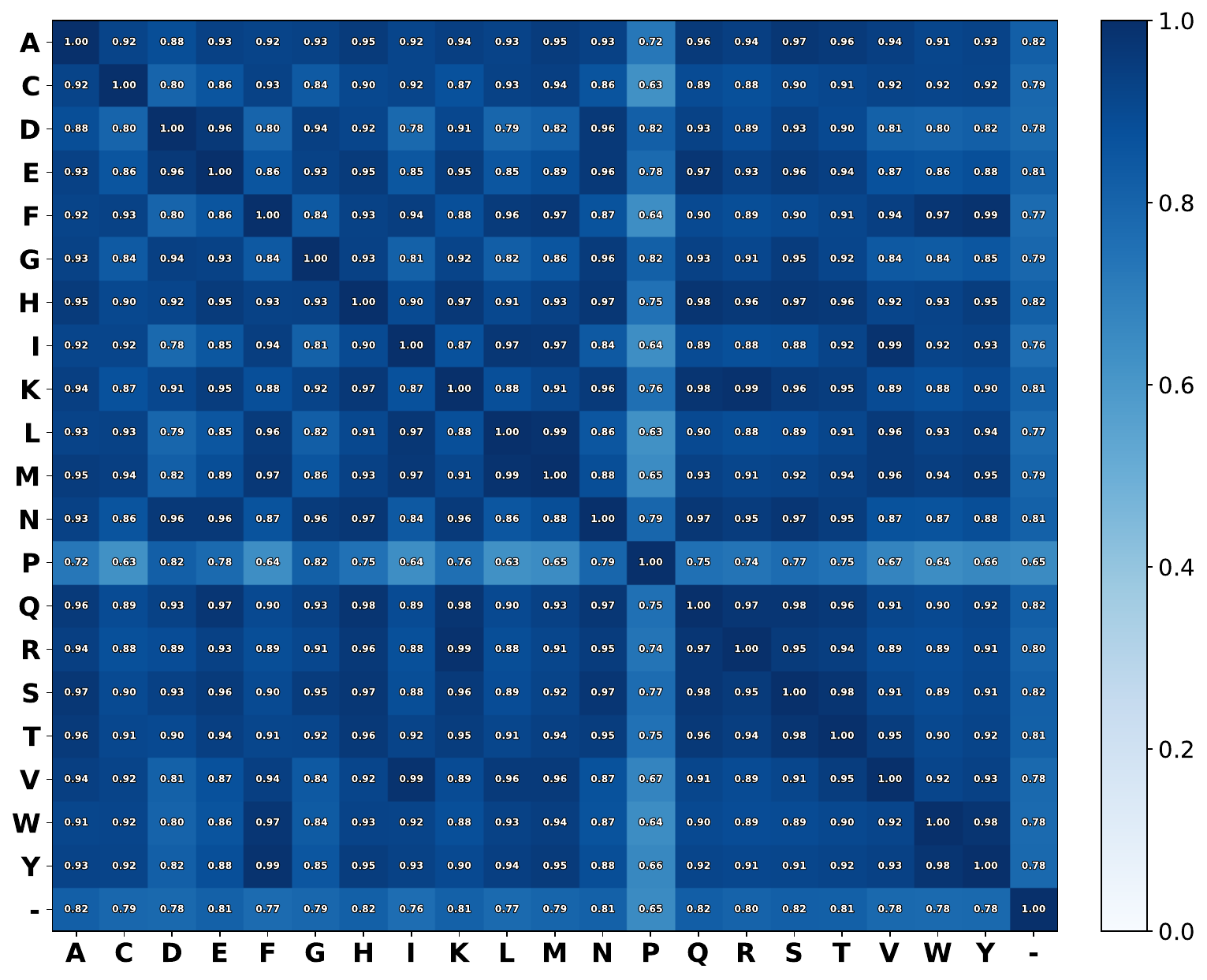}
  \caption{A visual representation of the CLOCK-GP correlation matrix fit to the thermostability data 
    described in \secref{sec:multi}. Since the correlation matrix varies from position to position, we depict
    the \emph{average} correlation matrix across all positions in the $50$ held-out test landscapes.
    We include the gap token `-'. The correlation ranges from $0.63$ to $1.0$.
    }
  \label{fig:app:local-clock-corr}
\end{figure}

\begin{figure}[t]
  \centering
  \includegraphics[width=0.6 \columnwidth]{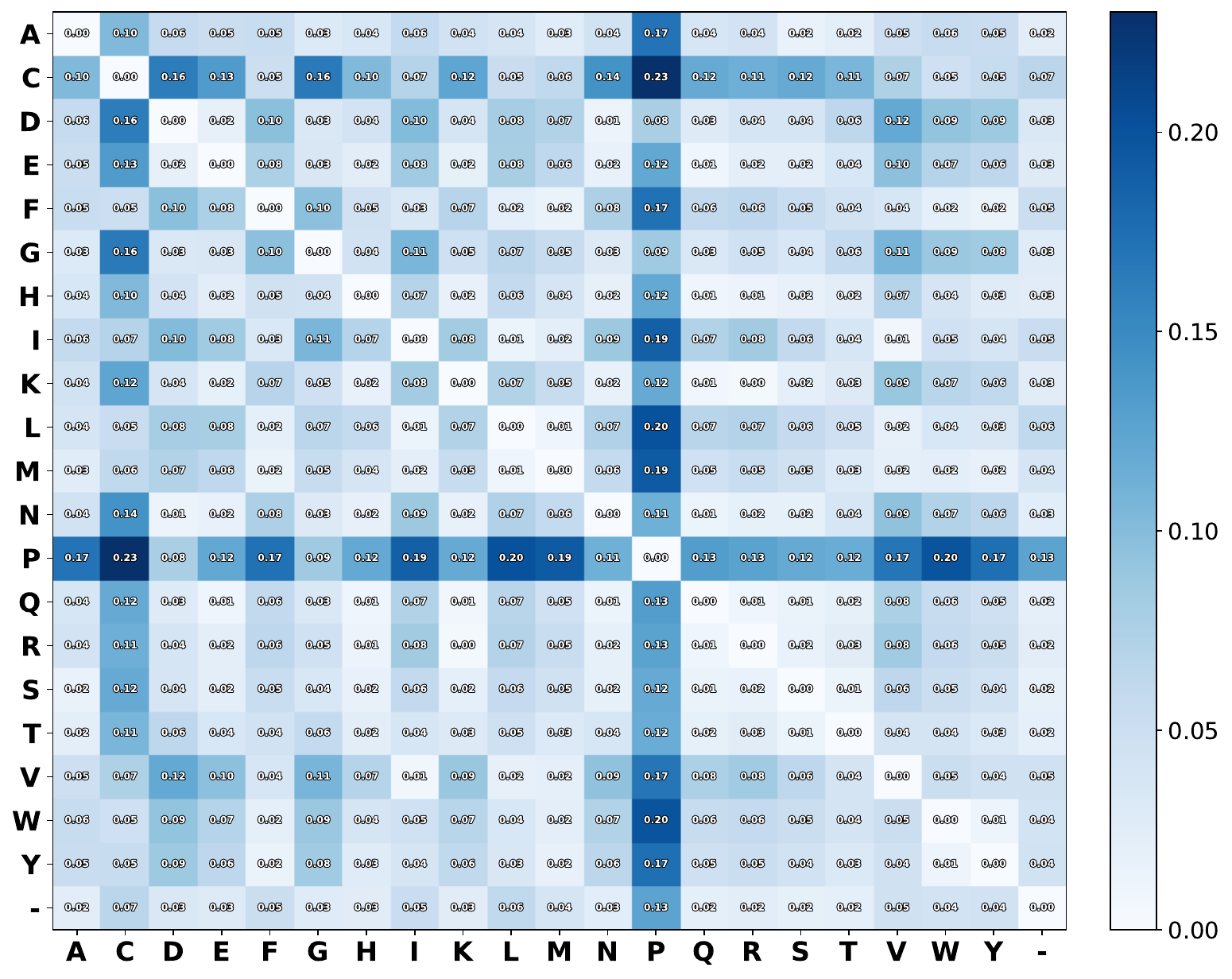}
    \caption{[Companion figure to \figref{fig:app:local-clock-corr}] 
    We depict the \emph{variability} of the CLOCK-GP correlation matrix fit to the thermostability data 
    described in \secref{sec:multi}. In particular we depict the quantile difference defined by $\delta = q_{95} - q_5$ where $q_{95}$ and $q_{5}$
    are the $0.95$ and $0.05$ quantiles, respectively. 
    This quantile difference is computed using all positions in the $50$ held-out test landscapes and ranges from $0.0$ (there is no variability
    at all in the diagonal) to $0.23$ (for proline-cysteine).
    }
  \label{fig:app:local-clock-corr-quantiles}
\end{figure}

\begin{figure*}[t]
\centering
\includegraphics[width=0.49\textwidth]{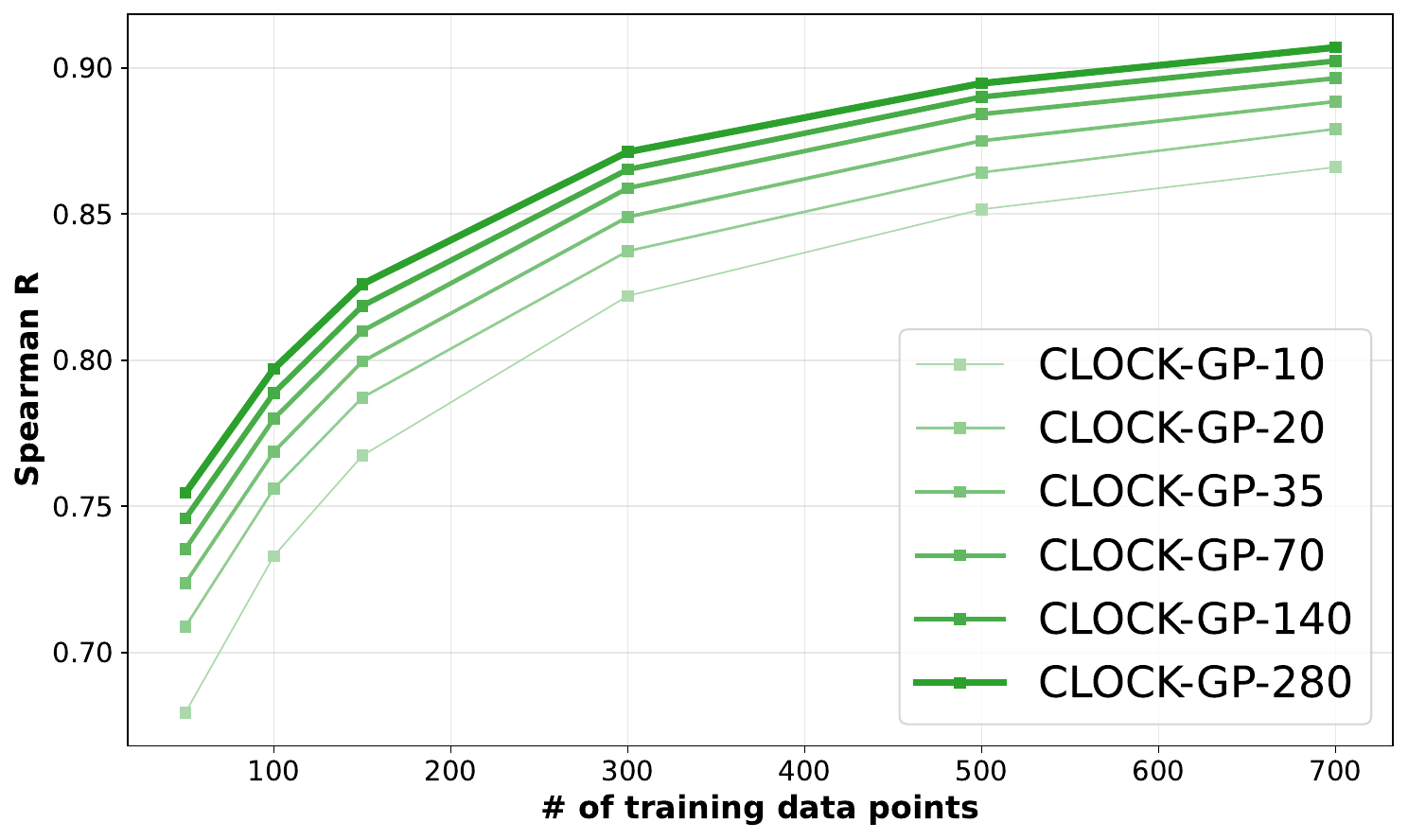}
\includegraphics[width=0.49\textwidth]{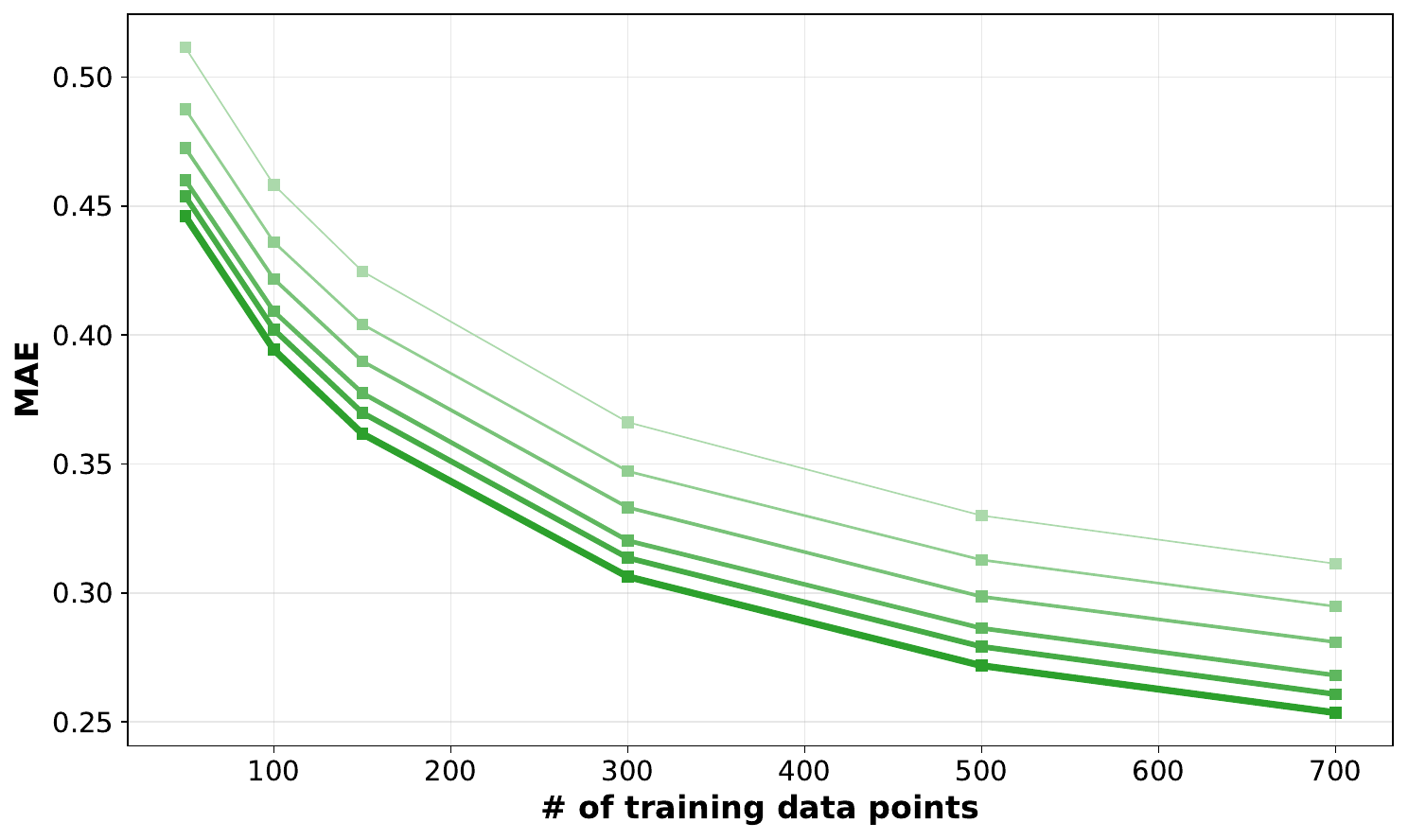}
    \caption{\emph{[Companion figure to \figref{fig:landscape_spearman}]} 
    We depict how \texttt{CLOCK-GP} performance changes as we vary the number of training landscapes from $10$ to $280$.
    Spearman R (left) and MAE (right) are averaged across 100 train/test splits in 50 held-out landscapes.
    For models trained on fewer than $280$ landscapes, we train two models on distinct random 
    subsets of the $280$ training landscapes, in which case metrics are averaged across both models.
    See \secref{sec:multi} for additional details on the setup.
    }
\label{fig:app:clock_gp_land}
\end{figure*}

\begin{figure*}[t]
\centering
\includegraphics[width=0.49\textwidth]{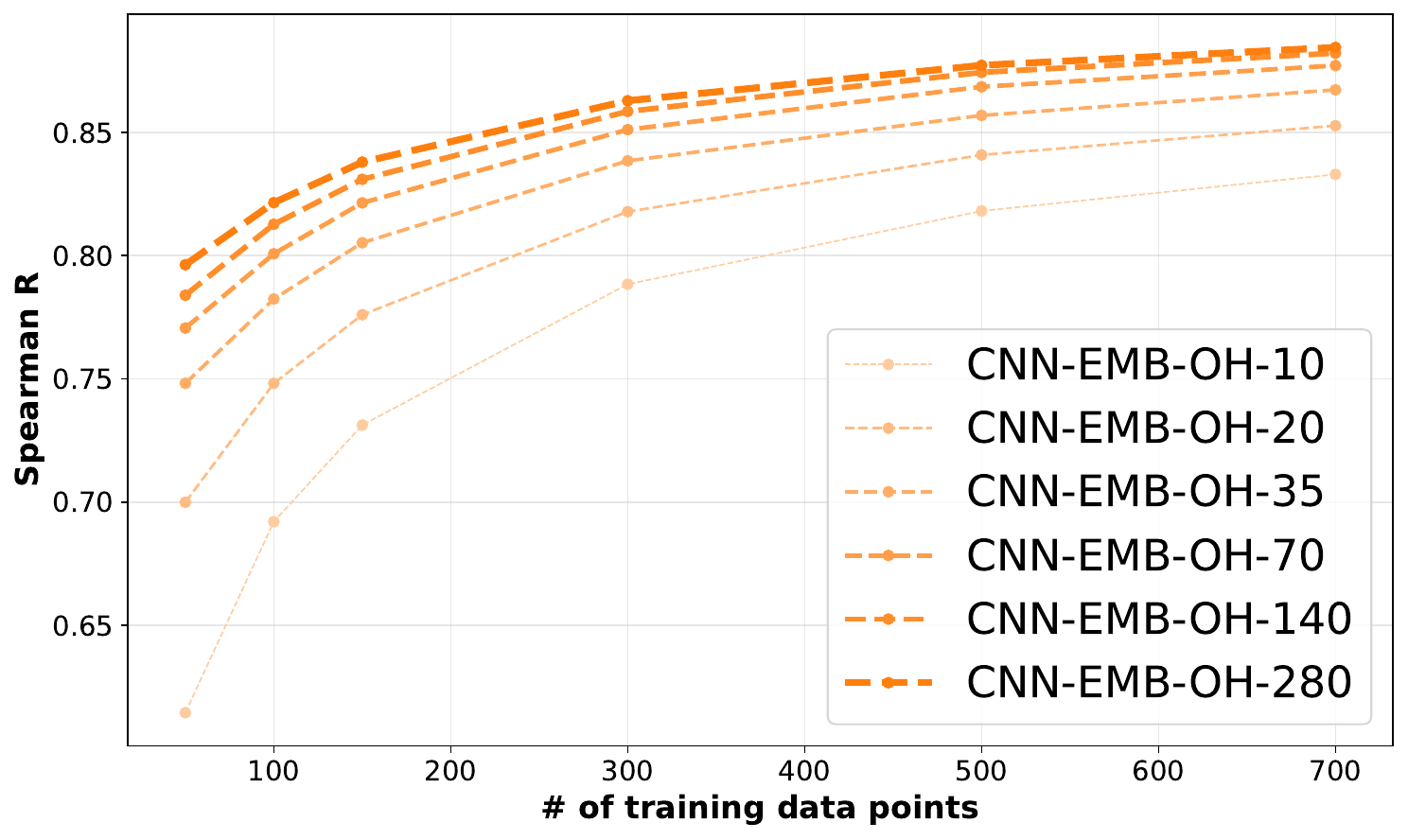}
\includegraphics[width=0.49\textwidth]{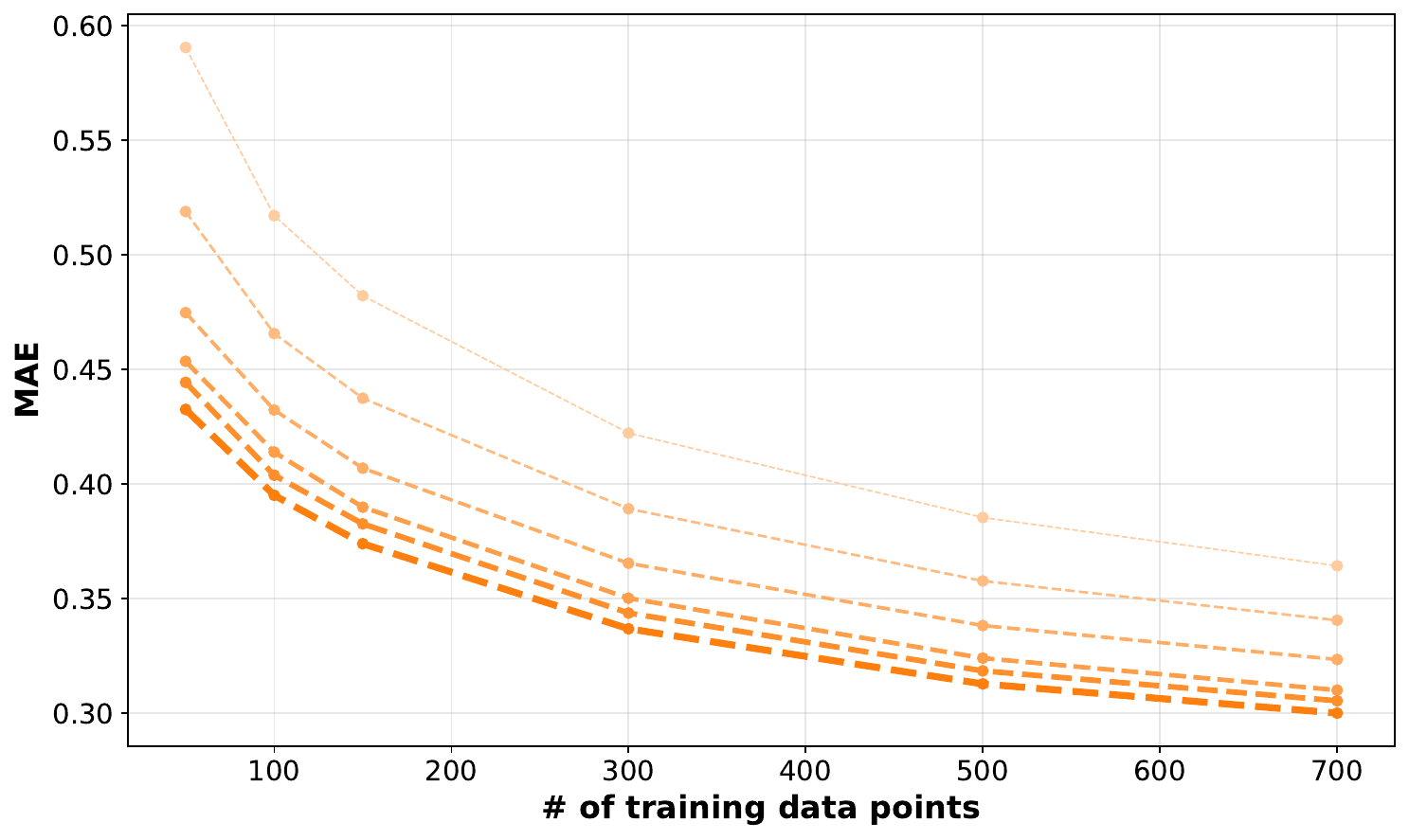}
    \caption{\emph{[Companion figure to \figref{fig:landscape_spearman}]}
    We depict how \texttt{CNN-EMB-OH} performance changes as we vary the number of training landscapes from $10$ to $280$.
    Spearman R (left) and MAE (right) are averaged across 100 train/test splits in 50 held-out landscapes.
    For models trained on fewer than $280$ landscapes, we train two models on distinct random
    subsets of the $280$ training landscapes, in which case metrics are averaged across both models.
    See \secref{sec:multi} for additional details on the setup.
    }
\label{fig:app:cnn_oh_land}
\end{figure*}

\begin{figure*}[t]
\centering
\includegraphics[width=0.49\textwidth]{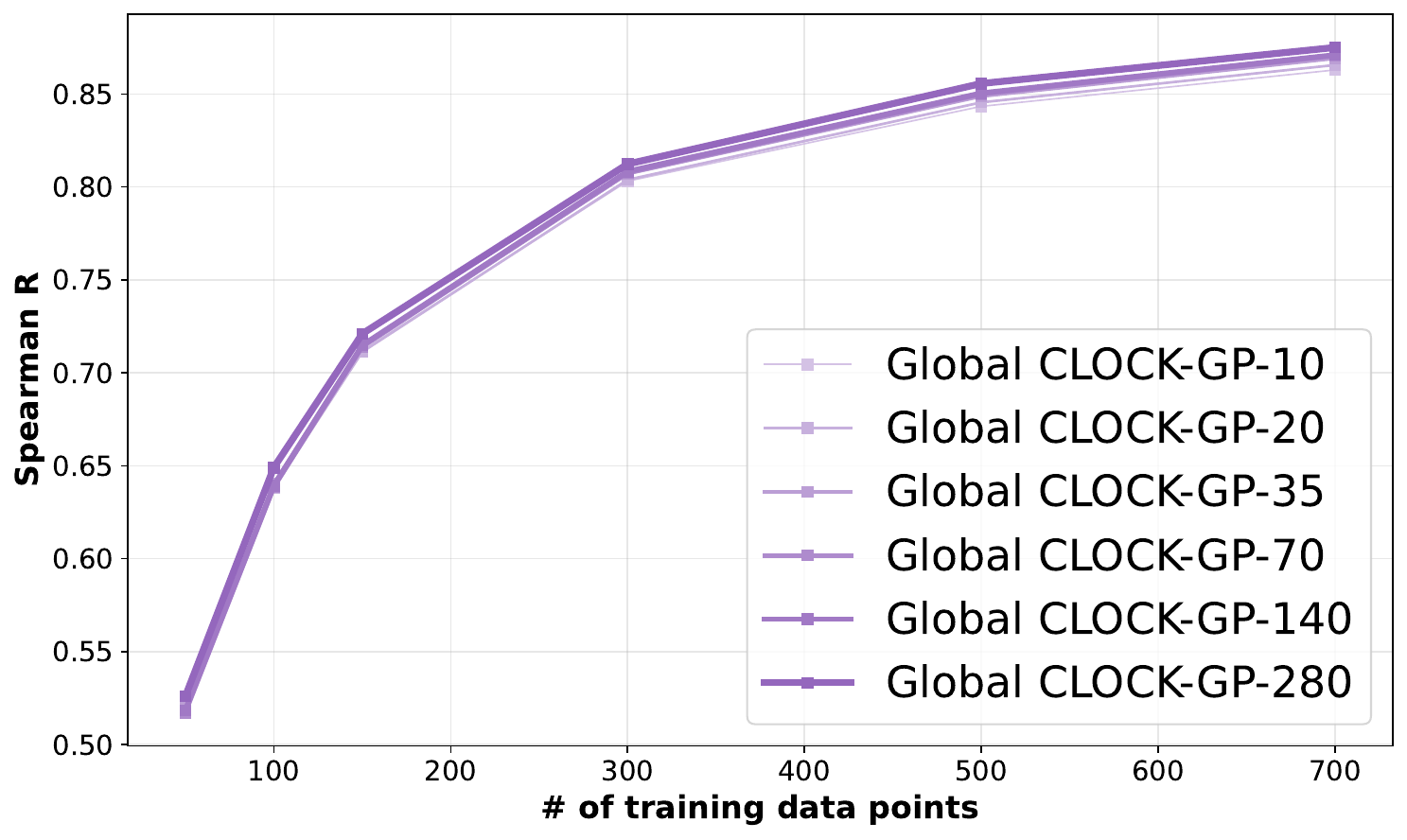}
\includegraphics[width=0.49\textwidth]{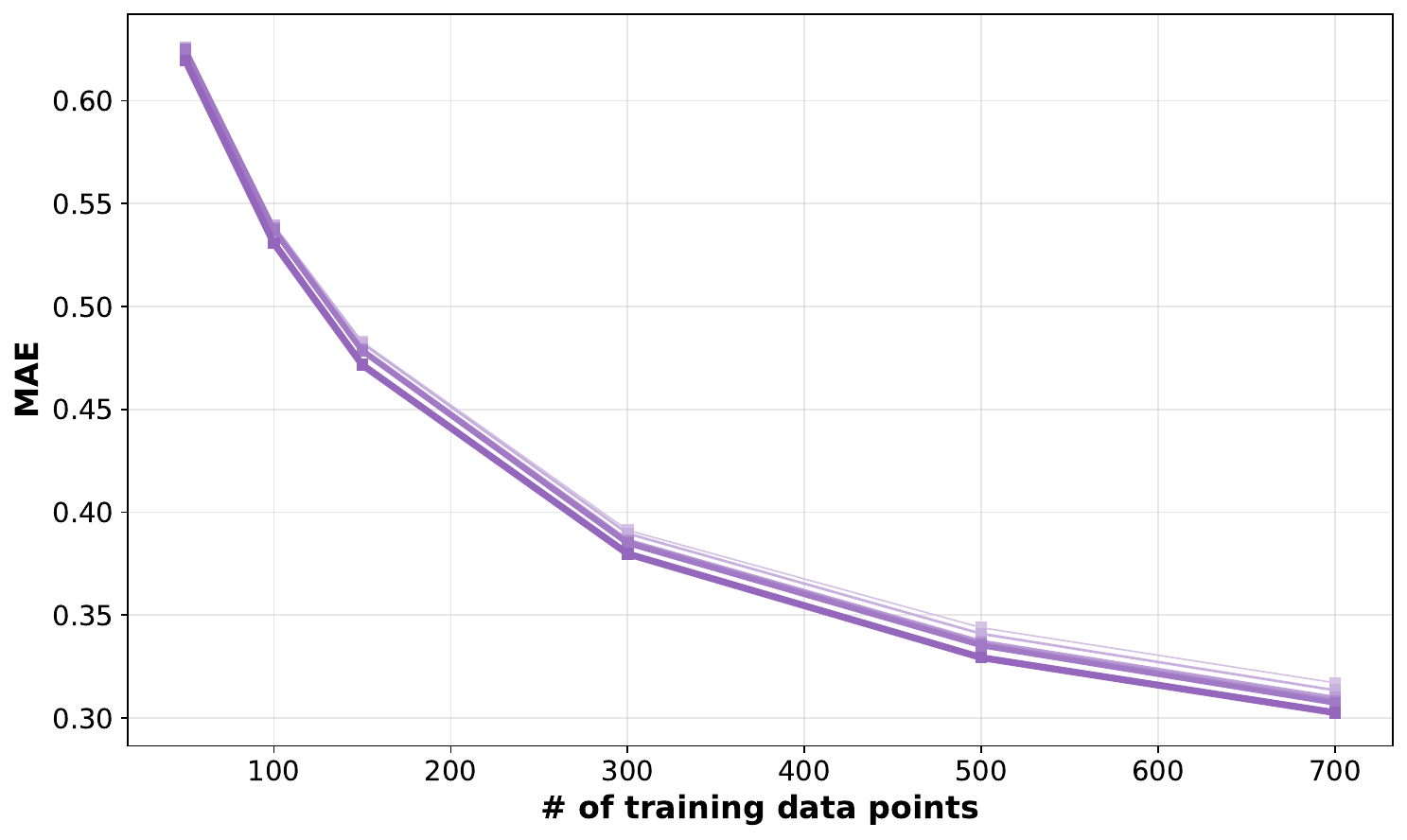}
    \caption{\emph{[Companion figure to \figref{fig:landscape_spearman}]}
    We depict how \texttt{Global CLOCK-GP} performance changes as we vary the number of training landscapes from $10$ to $280$.
    Spearman R (left) and MAE (right) are averaged across 100 train/test splits in 50 held-out landscapes.
    For models trained on fewer than $280$ landscapes, we train two models on distinct random
    subsets of the $280$ training landscapes, in which case metrics are averaged across both models.
    See \secref{sec:multi} for additional details on the setup.
    }
\label{fig:app:global_clock_gp_land}
\end{figure*}

\begin{figure*}[t]
\centering
\includegraphics[width=0.49\textwidth]{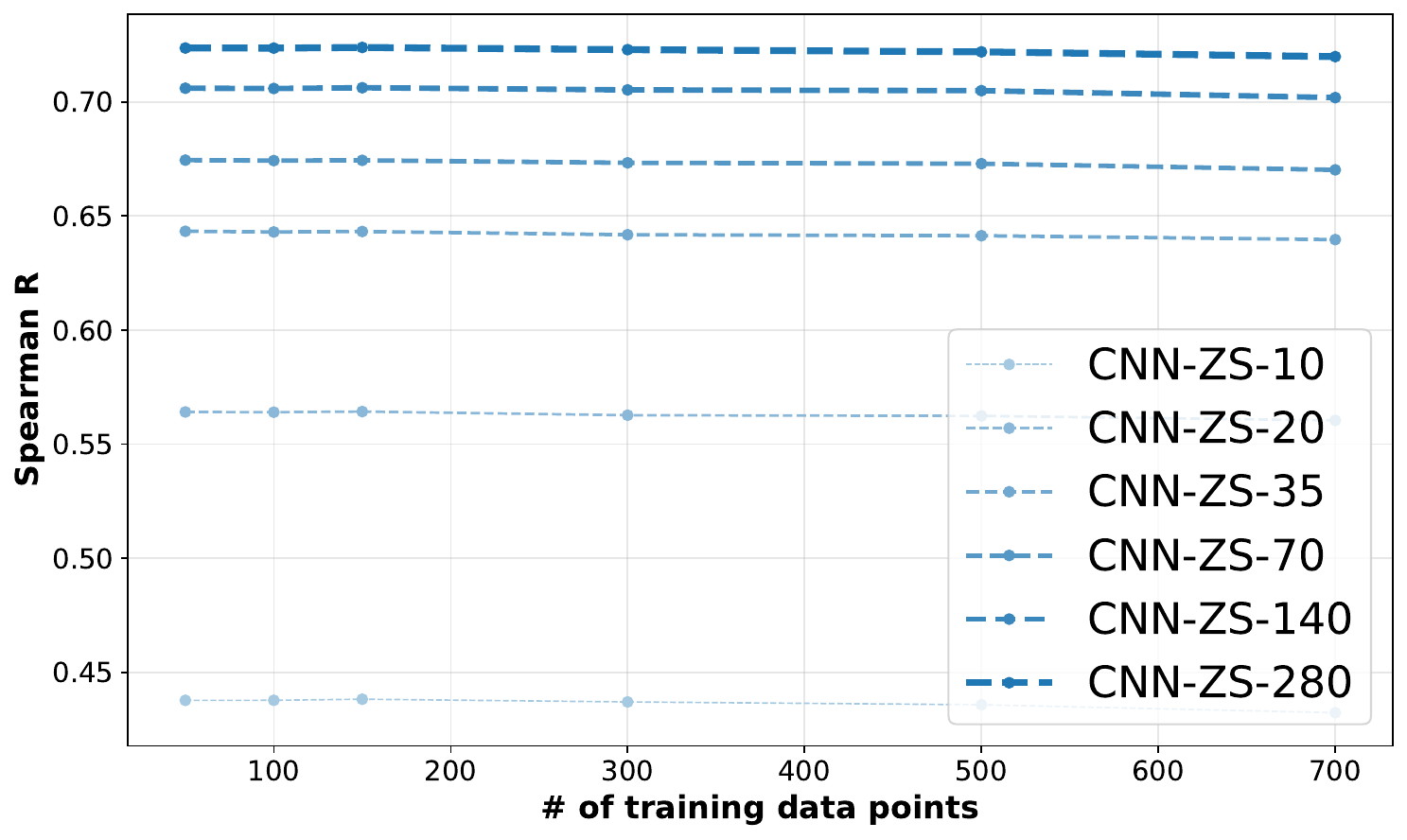}
\includegraphics[width=0.49\textwidth]{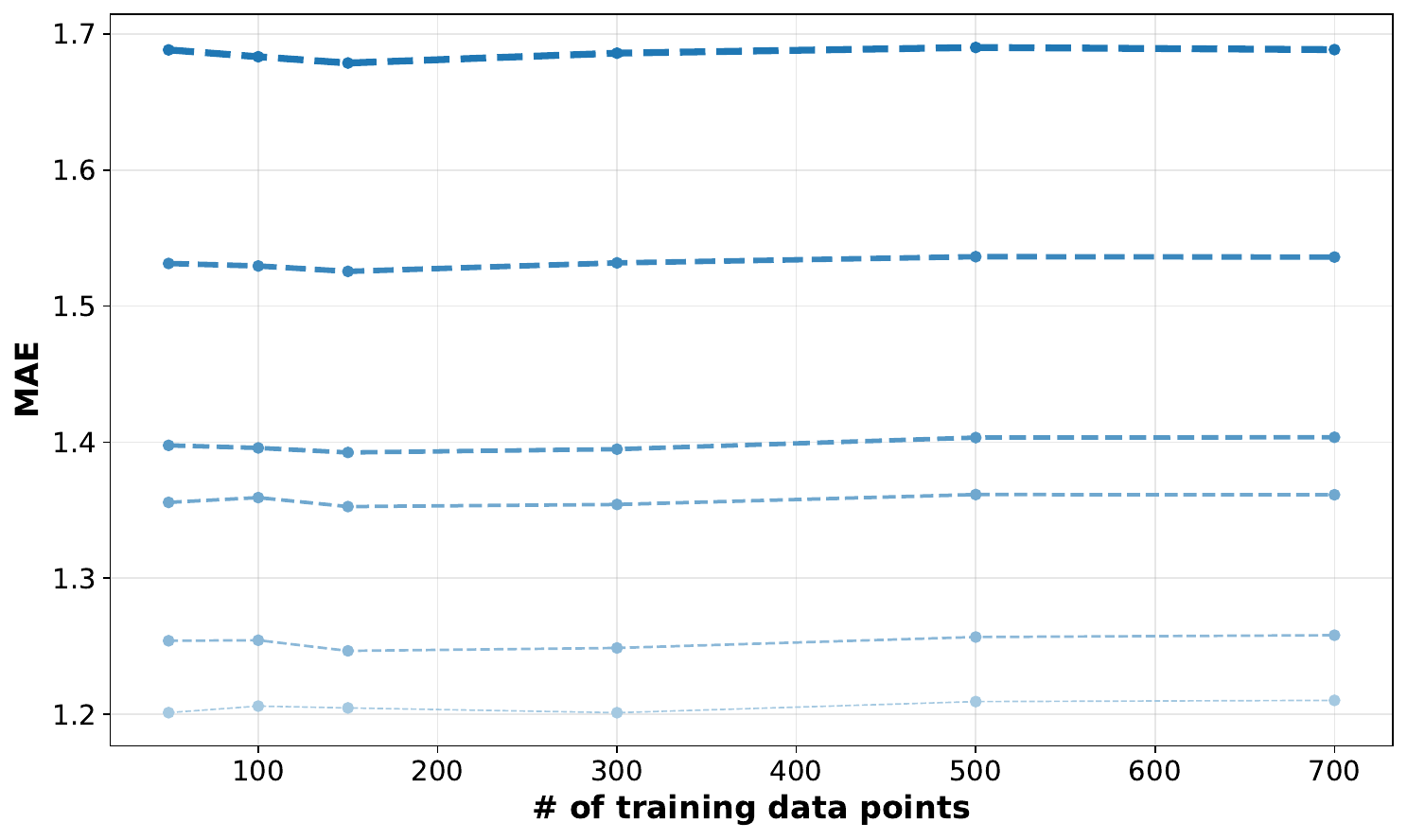}
    \caption{\emph{[Companion figure to \figref{fig:landscape_spearman}]}
    We depict how \texttt{CNN-ZS} performance changes as we vary the number of training landscapes from $10$ to $280$.
    Spearman R (left) and MAE (right) are averaged across 100 train/test splits in 50 held-out landscapes.
    Note that even though \texttt{CNN-ZS} is zero-shot, the performance varies slightly from left to right because
    the held-out test data varies as the size of the ``training set'' increases and thus the test set changes.
    For models trained on fewer than $280$ landscapes, we train two models on distinct random
    subsets of the $280$ training landscapes, in which case metrics are averaged across both models.
    See \secref{sec:multi} for additional details on the setup.
    }
\label{fig:app:cnn_zs_land}
\end{figure*}

\begin{figure*}[t]
\centering
\includegraphics[width=0.49\textwidth]{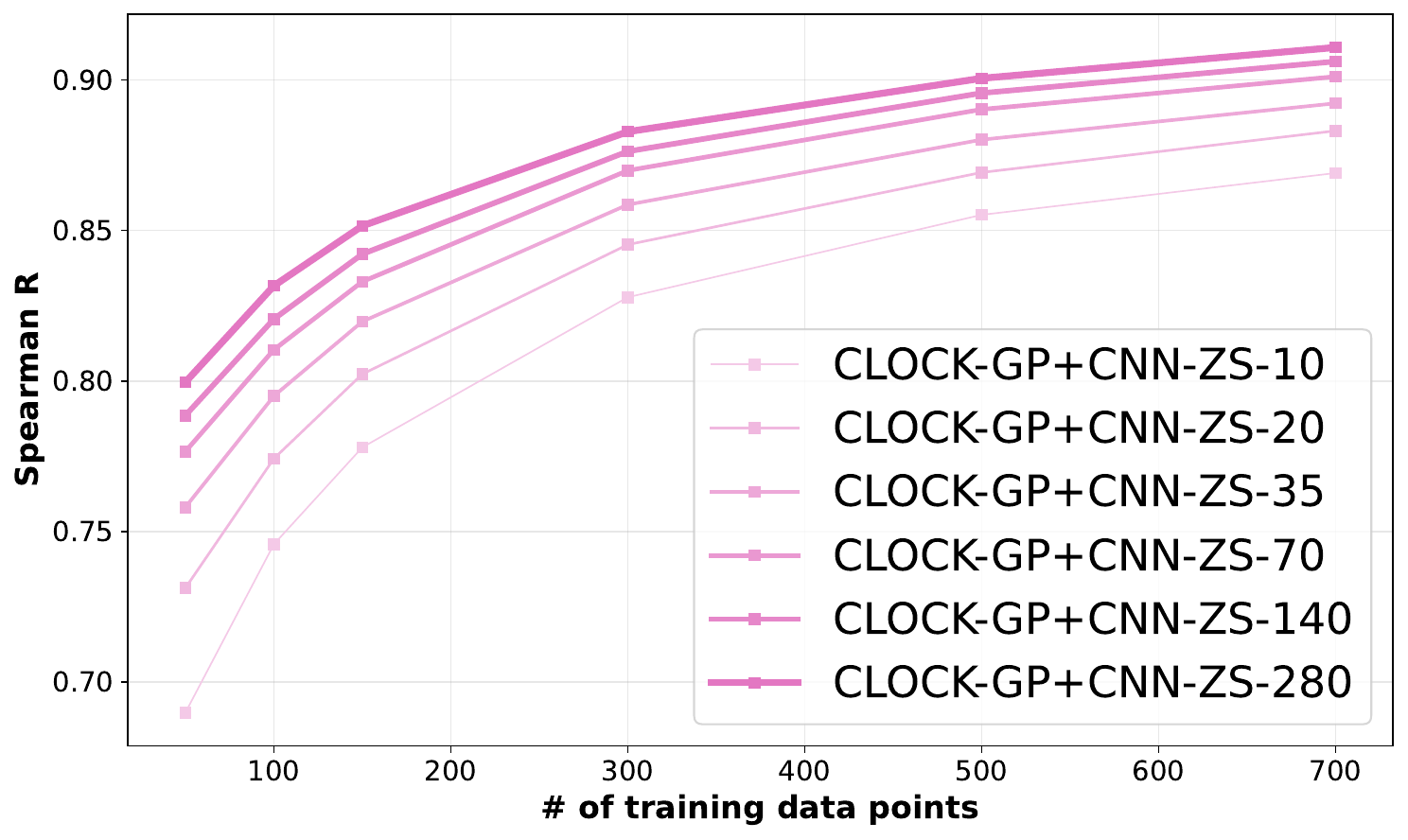}
\includegraphics[width=0.49\textwidth]{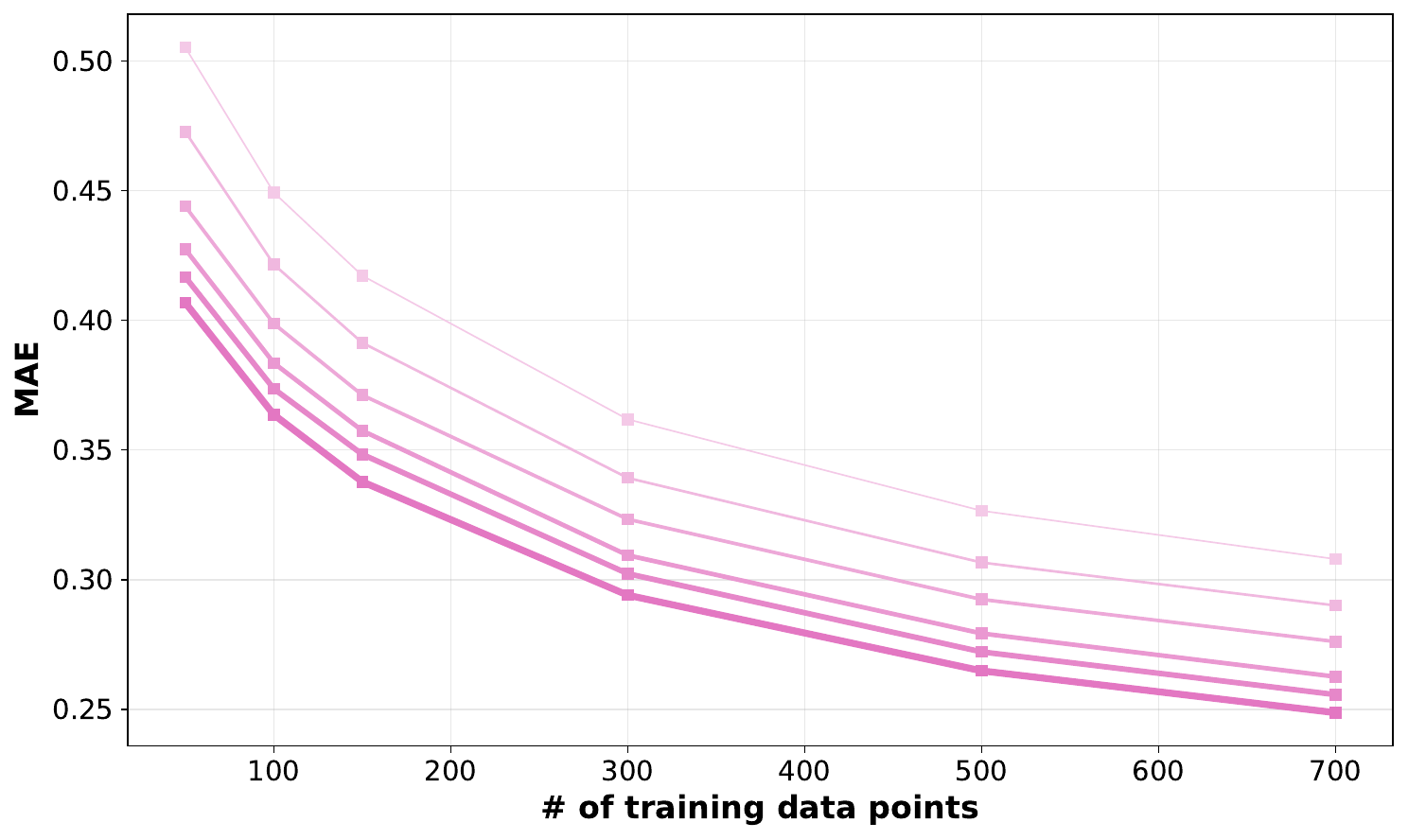}
    \caption{\emph{[Companion figure to \figref{fig:landscape_spearman}]}
    We depict how \texttt{CLOCK-GP+CNN-ZS} performance changes as we vary the number of training landscapes from $10$ to $280$.
    Spearman R (left) and MAE (right) are averaged across 100 train/test splits in 50 held-out landscapes.
    For models trained on fewer than $280$ landscapes, we train two models on distinct random
    subsets of the $280$ training landscapes, in which case metrics are averaged across both models.
    See \secref{sec:multi} for additional details on the setup.
    }
\label{fig:app:clock_gp_cnn_zs_land}
\end{figure*}

\begin{figure*}[t]
\centering
\includegraphics[width=\textwidth]{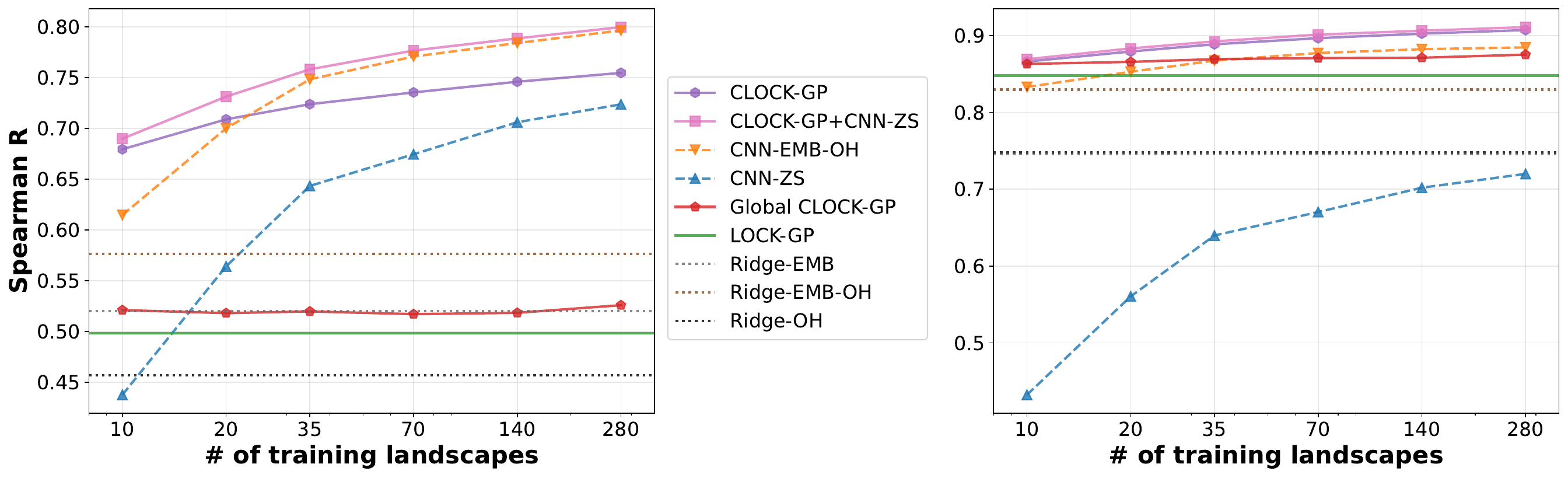}
    \caption{\emph{[Companion figure to \figref{fig:landscape_spearman}]}
    We depict how various multi-task models' performance changes as we vary the number of training landscapes from $10$ to $280$
    for a fixed number of training points: (left) $N_{\rm train}=50$; (right) $N_{\rm train}=700$.
    Spearman R is averaged across 100 train/test splits in 50 held-out landscapes.
    Note that \texttt{LOCK-GP}, \texttt{Ridge-OH}, \texttt{Ridge-EMB} and \texttt{Ridge-EMB-OH} are landscape-local models and so their performance
    does not vary with the number of training landscapes.
    For models trained on fewer than $280$ landscapes, we train two models on distinct random
    subsets of the $280$ training landscapes, in which case metrics are averaged across both models.
    See \secref{sec:multi} for additional details on the setup.
    }
\label{fig:app:landscape_50_700_spearman}
\end{figure*}

\begin{figure*}[t]
\centering
\includegraphics[width=\textwidth]{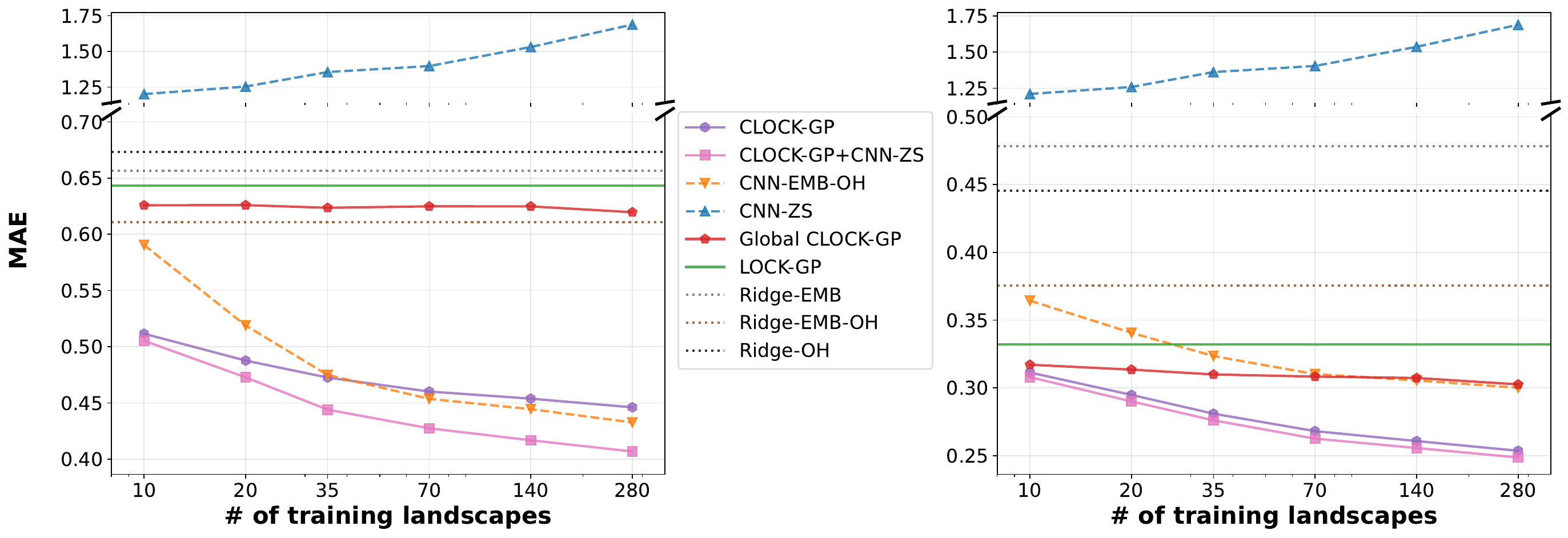}
    \caption{\emph{[Companion figure to \figref{fig:landscape_spearman}]}
    We depict how various multi-task models' performance changes as we vary the number of training landscapes from $10$ to $280$
    for a fixed number of training points: (left) $N_{\rm train}=50$; (right) $N_{\rm train}=700$.
    MAE is averaged across 100 train/test splits in 50 held-out landscapes.
    Note that \texttt{LOCK-GP}, \texttt{Ridge-OH}, \texttt{Ridge-EMB} and \texttt{Ridge-EMB-OH} are landscape-local models and so their performance
    does not vary with the number of training landscapes.
    For models trained on fewer than $280$ landscapes, we train two models on distinct random
    subsets of the $280$ training landscapes, in which case metrics are averaged across both models.
    See \secref{sec:multi} for additional details on the setup.
    }
\label{fig:app:landscape_50_700_mae}
\end{figure*}

\begin{figure}[t]
\centering
\includegraphics[width=0.48\textwidth]{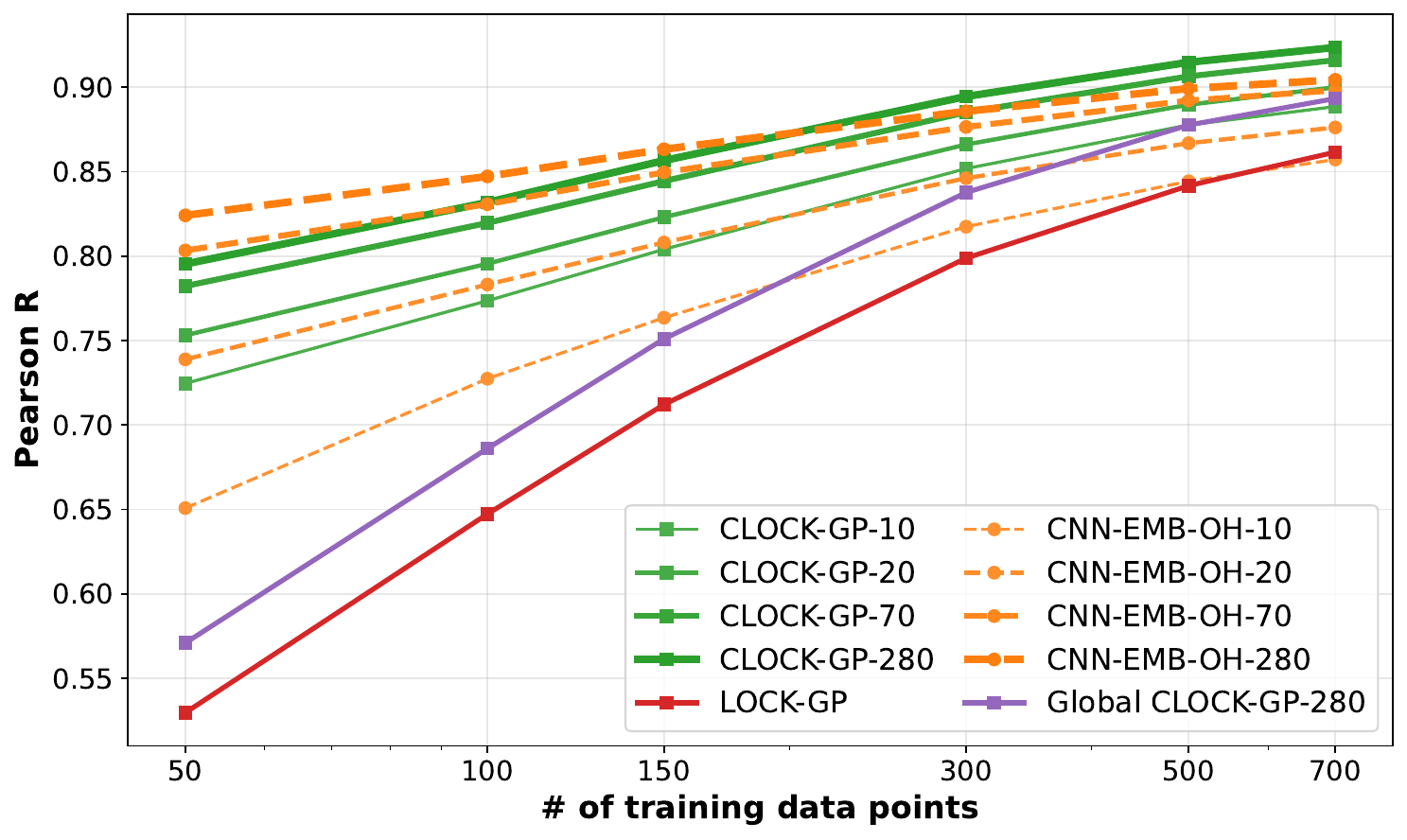}
\includegraphics[width=0.48\textwidth]{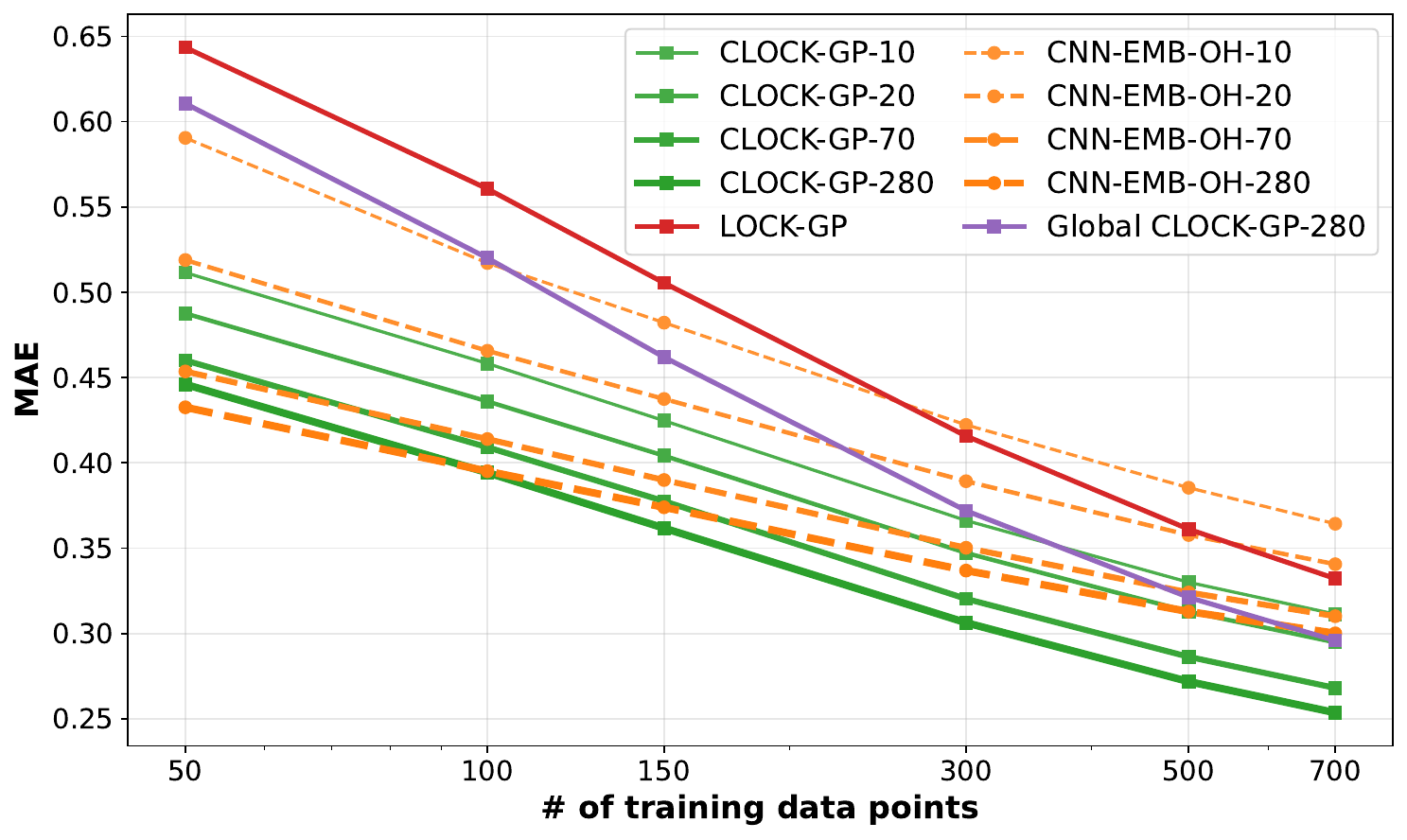}
    \caption{\emph{[Companion figure to \figref{fig:landscape_spearman}]} 
    We depict how multi-task model performance changes as we vary the number of training landscapes from $10$ to $280$.
    Pearson R (left) and MAE (right) are averaged across $100$ train/test splits in $50$ held-out landscapes.
    The number of (landscape-local) training data points is plotted on the horizontal axis.
    \texttt{CLOCK-GP} performs best in the low landscape regime: e.g.~\texttt{CLOCK-GP-10} approximately matches the
    performance of \texttt{CNN-EMB-OH-20}.
    }
    \label{fig:app:landscape_pearson_mae}
\end{figure}

\begin{figure*}[t]
\centering
\includegraphics[width=0.5\textwidth]{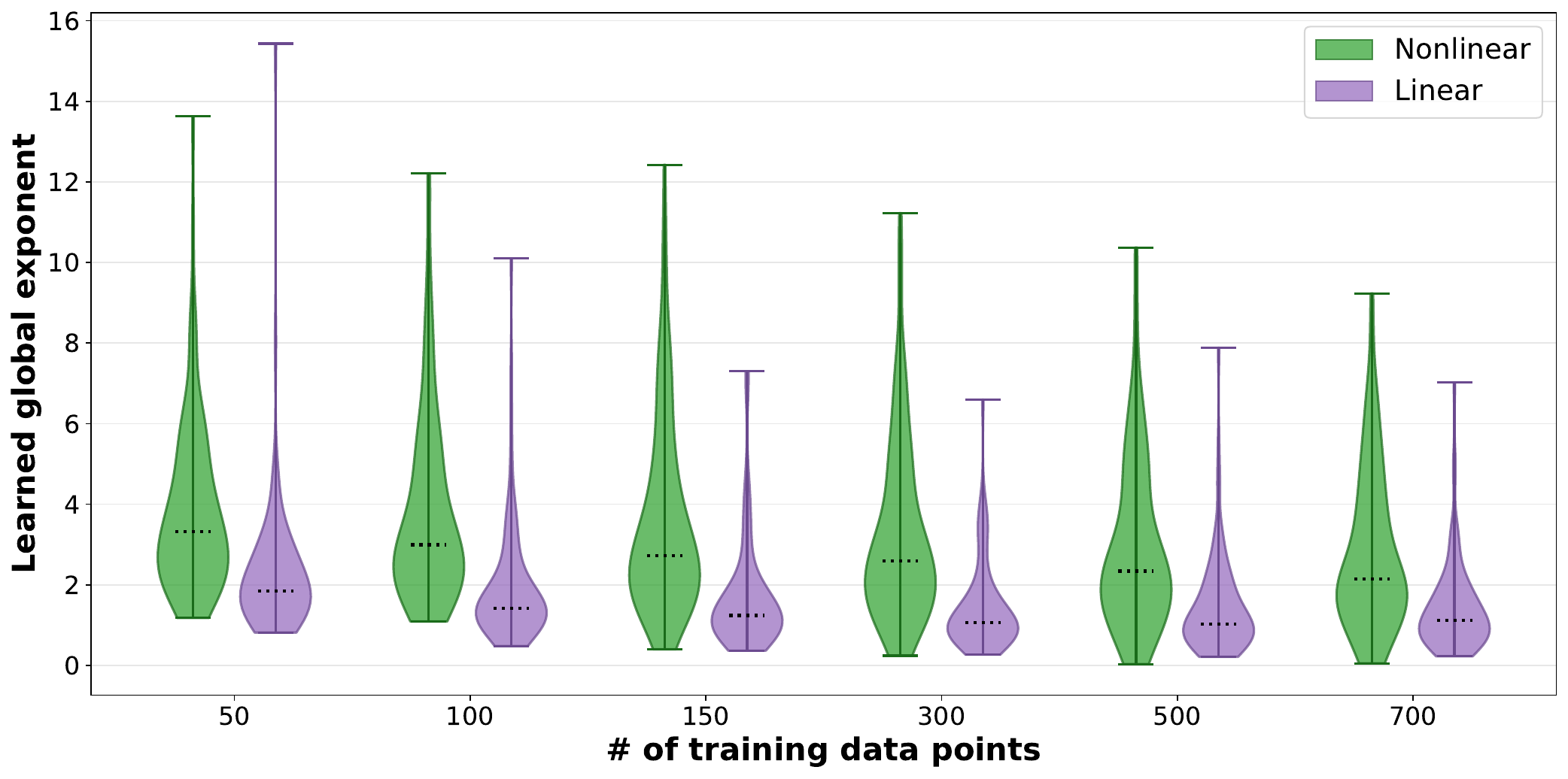}
    \caption{\emph{[Companion figure to \figref{fig:landscape_spearman}]}
    We depict the learned exponent $\alpha$ in the CLOCK kernel across 100 train/test splits in 50 held-out landscapes.
    Notably, the learned exponent is larger for the non-linear kernel $\kLnl$ than for the linear kernel $\kLlin$.
    }
\label{fig:app:landscape_exponents}
\end{figure*}

\begin{figure*}[t]
\centering
\includegraphics[width=0.5\textwidth]{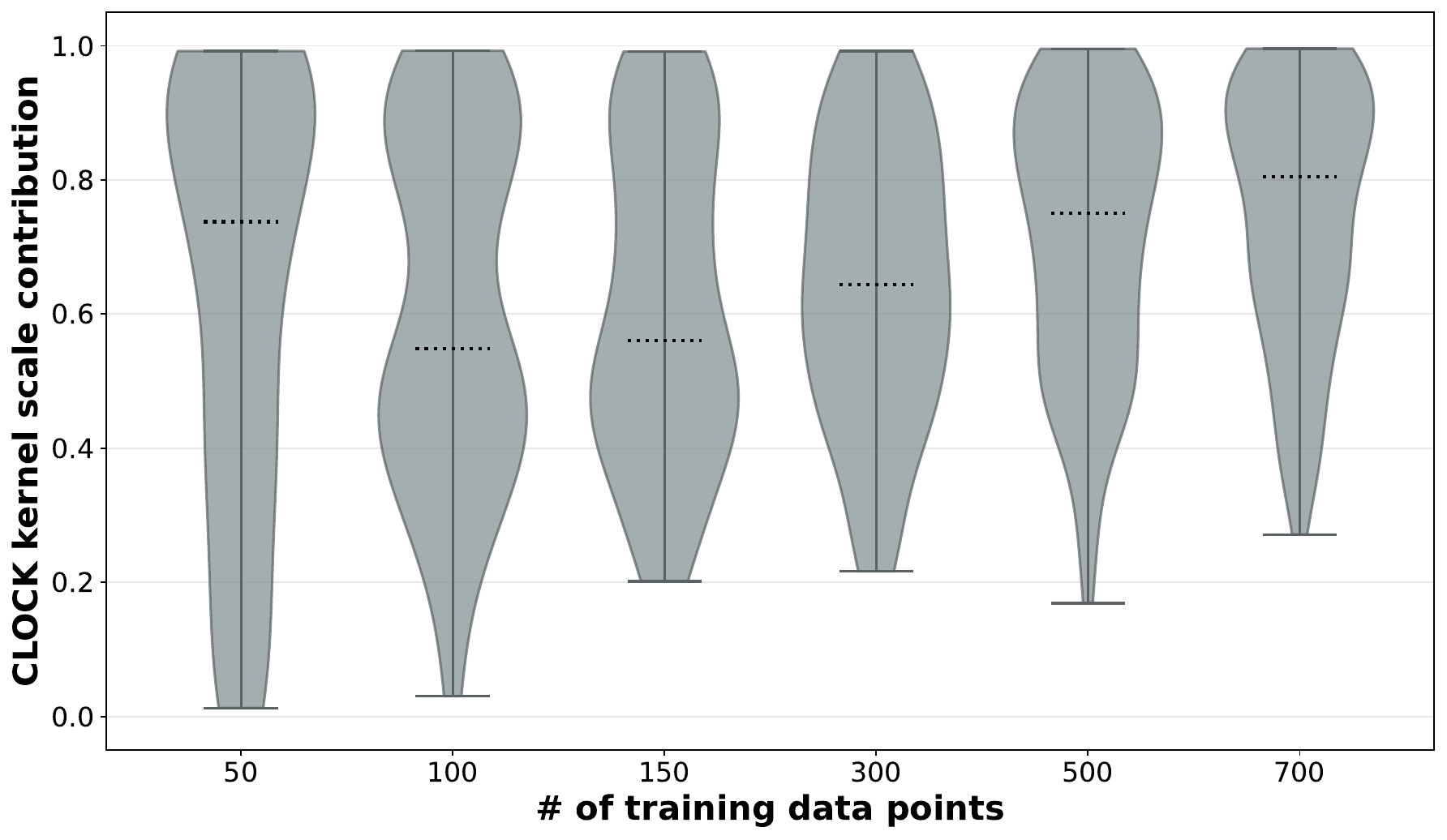}
    \caption{\emph{[Companion figure to \figref{fig:landscape_spearman}]}
    We depict the \emph{relative} contribution of the CLOCK kernel scale to \texttt{CLOCK-GP+CNN-ZS} 
    across 100 train/test splits in 50 held-out landscapes.
    That is the full kernel is of the form $\sigma_{0}^2 \kLnl + \sigma_{\rm rbf}^2 k_{\rm rbf}$
    and we depict distributions of $\frac{\sigma_{0}^2}{\sigma_{0}^2 + \sigma_{\rm rbf}^2}$.
    As we would expect---given its good performance for larger numbers of training data points---the CLOCK kernel plays an increasingly larger 
    role as the number of training data points increases.
    We note that across all training sizes the median relative contribution is greater than one half.
    }
\label{fig:app:landscape_kernel_scale}
\end{figure*}

%% file: tsubo_table.tex
\begin{table*}[t]
\centering
\caption{This table gives quantitative performance metrics for the multi-task experiment in \secref{sec:multi}.
    As such it complements \figref{fig:tsuboperf} and \figref{fig:app:tsuboperf}. 
    We show results for a few model classes that are not included in the main text.
    This includes \texttt{CNN-ZS-OH} which concatenates CNN-based zero shot predictions (i.e.~these are scalar features) with one-hot
    sequences and passes the concatenated features to a ridge regression head (as opposed to \texttt{CNN-EMB-OH}, which
    uses neural features extracted from the CNN together with one-hot features). 
    We also include two local models, \texttt{Ridge-EMB} and \texttt{Ridge-EMB-OH}, which
    combine a ridge regression head with mean-pooled sequence-and-structure features from Chroma (in the latter case the mean-pooled features
    are concatenated with one-hot features). In addition, we include $4$ GP models (marked by `\texttt{-LIN}') where the 
    kernel is \emph{linear} (i.e.~$\kLlin$), in contrast to the $4$ GPs in the main text, which use non-linear kernels (i.e.~$\kLnl$).
    \texttt{CLOCK-GP+CNN-ZS} performs best across the board w.r.t.~all training sizes and all four metrics.}
\label{tab:tsuboperf}
\small
\begin{tabular}{lcccccccccccc}
\toprule
\textbf{Model} & \multicolumn{4}{c}{\textbf{50 training points}} & \multicolumn{4}{c}{\textbf{150 training points}} & \multicolumn{4}{c}{\textbf{700 training points}} \\
 & \textbf{Spear.} & \textbf{MAE} & \textbf{Pear.} & \textbf{RMSE} & \textbf{Spear.} & \textbf{MAE} & \textbf{Pear.} & \textbf{RMSE} & \textbf{Spear.} & \textbf{MAE} & \textbf{Pear.} & \textbf{RMSE} \\
\midrule
\texttt{Ridge-OH} & 0.457 & 0.673 & 0.474 & 0.884 & 0.621 & 0.567 & 0.629 & 0.774 & 0.748 & 0.446 & 0.754 & 0.638 \\
\texttt{Ridge-EMB} & 0.520 & 0.657 & 0.550 & 0.847 & 0.643 & 0.567 & 0.672 & 0.740 & 0.746 & 0.478 & 0.773 & 0.626 \\
\texttt{Ridge-EMB-OH} & 0.576 & 0.611 & 0.604 & 0.800 & 0.717 & 0.499 & 0.740 & 0.671 & 0.829 & 0.375 & 0.849 & 0.517 \\
\midrule
\texttt{CNN-ZS} & 0.724 & 1.688 & 0.754 & 1.820 & 0.724 & 1.679 & 0.754 & 1.811 & 0.720 & 1.688 & 0.750 & 1.820 \\
\texttt{CNN-EMB-OH} & 0.796 & 0.433 & 0.824 & 0.571 & 0.838 & 0.374 & 0.863 & 0.502 & 0.884 & 0.300 & 0.904 & 0.414 \\
\texttt{CNN-ZS-OH} & 0.791 & 0.439 & 0.819 & 0.578 & 0.826 & 0.388 & 0.853 & 0.520 & 0.863 & 0.326 & 0.887 & 0.449 \\
\midrule
\texttt{LOCK-GP} & 0.498 & 0.643 & 0.530 & 0.852 & 0.689 & 0.505 & 0.712 & 0.704 & 0.848 & 0.332 & 0.862 & 0.496 \\
\texttt{Global CLOCK-GP} & 0.529 & 0.611 & 0.571 & 0.829 & 0.720 & 0.462 & 0.751 & 0.657 & 0.877 & 0.296 & 0.893 & 0.436 \\
\texttt{CLOCK-GP} & 0.755 & 0.446 & 0.795 & 0.610 & 0.826 & 0.362 & 0.857 & 0.508 & 0.907 & 0.254 & 0.924 & 0.367 \\
\texttt{CLOCK-GP+CNN-ZS} & \textbf{0.800} & \textbf{0.407} & \textbf{0.832} & \textbf{0.553} & \textbf{0.852} & \textbf{0.338} & \textbf{0.880} & \textbf{0.469} & \textbf{0.911} & \textbf{0.249} & \textbf{0.927} & \textbf{0.358} \\
\midrule
\texttt{LOCK-GP-LIN} & 0.478 & 0.641 & 0.514 & 0.860 & 0.670 & 0.513 & 0.695 & 0.720 & 0.833 & 0.350 & 0.849 & 0.518 \\
\texttt{Global CLOCK-GP-LIN} & 0.484 & 0.618 & 0.534 & 0.862 & 0.698 & 0.480 & 0.722 & 0.698 & 0.859 & 0.322 & 0.874 & 0.476 \\
\texttt{CLOCK-GP-LIN} & 0.509 & 0.581 & 0.582 & 0.853 & 0.743 & 0.424 & 0.771 & 0.641 & 0.890 & 0.279 & 0.904 & 0.412 \\
\texttt{CLOCK-GP-LIN+CNN-ZS} & 0.778 & 0.435 & 0.805 & 0.595 & 0.837 & 0.360 & 0.862 & 0.502 & 0.899 & 0.269 & 0.916 & 0.387 \\
\bottomrule
\end{tabular}
\end{table*}

%% file: correlation_machine.tex
\begin{figure}[h]
    \begin{pycode}[numbers=none,basicstyle=\ttfamily]%\scriptsize,xleftmargin=0pt,xrightmargin=-50pt]
class CorrelationMachine(nn.Module):
    """Takes structure embeddings and computes residue-wise correlation matrices."""

    def __init__(self, in_dim=128, alphabet_size=21):
        super(CorrelationMachine, self).__init__()
        self.in_dim = in_dim
        self.alphabet_size = alphabet_size
        self.embedding_dim = alphabet_size - 1
        self.layer_norm = nn.LayerNorm(in_dim)
        self.projection = nn.Linear(in_dim, alphabet_size * self.embedding_dim, bias=False)

    def forward(self, embeddings):
        """
        Args:
            embeddings: Structure embeddings (num_landscapes, max_seq_len, in_dim)

        Returns:
            correlation_matrices: (num_landscapes, max_seq_len, alphabet_size, alphabet_size)
        """
        assert embeddings.ndim == 3
        num_landscapes, max_seq_len, _ = embeddings.shape

        zs = self.projection(self.layer_norm(embeddings))
        zs = zs.reshape(num_landscapes, max_seq_len, self.alphabet_size, self.embedding_dim)
        zs = zs / math.sqrt(self.embedding_dim)
        correlation_matrices = torch.exp(-torch.cdist(zs, zs, p=2).pow(2))

        expected_shape = (num_landscapes, max_seq_len, self.alphabet_size, self.alphabet_size)
        assert correlation_matrices.shape == expected_shape 
        return correlation_matrices
\end{pycode}
    \caption{Pytorch implementation of the machinery described in \secref{sec:clock}, which takes positional structure embeddings $\bh_{1:L}(\SSS)$ as 
    input and returns correlation matrices $\bC_{1:L}$.}
\label{code:corrmachine}
\end{figure}

%% file: loss.tex
\begin{figure}[h]
    \begin{pycode}[numbers=none,basicstyle=\ttfamily]%\scriptsize]
def loss(
    seqs: torch.FloatTensor,
    targets: torch.FloatTensor,
    correlation_matrices: torch.FloatTensor,
    signal_to_noise_ratio: torch.FloatTensor, 
    mask: torch.FloatTensor = None,
):
    """
    Compute the loss for a given set of sequences, targets, and correlation matrices.

    Args:
        seqs: One-hot encoded sequences of shape (num_landscapes, batch_size, L, A)
        targets: Target values of shape (num_landscapes, batch_size)
        correlation_matrices: Correlation matrices of shape (num_landscapes, L, A, A)
        signal_to_noise_ratio: Signal-to-noise ratio of shape ()
        mask: Optional mask of shape (num_landscapes, batch_size)

    Returns:
        loss: loss tensor of shape (num_landscapes,)
    """
    num_landscapes, batch_size, L, A = seqs.shape

    assert targets.shape == (num_landscapes, batch_size)
    assert correlation_matrices.shape == (num_landscapes, L, A, A)
    assert seqs.dtype == targets.dtype and seqs.device == targets.device
    assert correlation_matrices.dtype == seqs.dtype 
    assert correlation_matrices.device == seqs.device
    assert signal_to_noise_ratio.shape == ()

    if mask is not None:
        assert mask.shape == (num_landscapes, batch_size) 
        assert mask.dtype == seqs.dtype and mask.device == seqs.device
        seqs = seqs * mask[..., None, None]

    kernels = torch.einsum("nbla,nBlA,nlaA->nbB", seqs, seqs, correlation_matrices)
    kernels = signal_to_noise_ratio * kernels + \
        torch.eye(batch_size, device=kernels.device, dtype=kernels.dtype)

    Ls = torch.linalg.cholesky(kernels, upper=False)
    L_inv_y = torch.linalg.solve_triangular(Ls, targets[..., None], upper=False).squeeze(-1)
    assert L_inv_y.shape == (num_landscapes, batch_size)

    if mask is None:
        num_seqs_per_landscape = batch_size 
        sigma_f_hat_sq = torch.einsum("nb,nb->n", L_inv_y, L_inv_y) / batch_size
    else:
        num_seqs_per_landscape = mask.sum(1)
        dot_products = torch.einsum("nb,nb,nb->n", L_inv_y, L_inv_y, mask) 
        sigma_f_hat_sq = dot_products / num_seqs_per_landscape
    assert sigma_f_hat_sq.shape == (num_landscapes,)

    loss = 0.5 * sigma_f_hat_sq.log() + \
        torch.diagonal(Ls, dim1=-1, dim2=-2).log().sum(-1) / num_seqs_per_landscape
    assert loss.shape == (num_landscapes,)
    return loss
\end{pycode}
\caption{Pytorch implementation of the concentration-likelihood-based training loss used in CLOCK. See \secref{app:clock} for details.}
\label{code:loss}
\end{figure}